\documentclass[12pt]{article}

\usepackage{newtxtext,newtxmath}

\usepackage{graphicx}

\usepackage[letterpaper,margin=1in]{geometry}

\renewenvironment{abstract}
	{\quotation}
	{\endquotation}

\date{}

\makeatletter
\renewcommand{\fnum@figure}{\textbf{Figure \thefigure}}
\renewcommand{\fnum@table}{\textbf{Table \thetable}}
\makeatother

\usepackage{scicite}

\usepackage{url}

\DeclareMathOperator*{\Ex}{\mathbb{E}}
\DeclareMathOperator*{\argmin}{\mathrm{argmin}}

\newcounter{algorithm}

\newenvironment{proof}[1][Proof]
  {\par\noindent\textit{#1.}\enspace}
  {\hfill\ensuremath{\square}\par}

\newtheorem{definition}{Definition}
\newtheorem{fact}{Fact}
\newtheorem{lemma}{Lemma}
\newtheorem{proposition}{Proposition}
\newtheorem{remark}{Remark}
\newtheorem{theorem}{Theorem}

\newcommand{\subpara}[1]{\noindent \underline{\em #1.}}

\newcommand{\BlueDown}{BlueDown}
\newcommand{\TopDown}{TopDown}

\newcommand{\R}{\mathbb{R}}
\newcommand{\Z}{\mathbb{Z}}
\newcommand{\cB}{\mathcal{B}}
\newcommand{\cL}{\mathcal{L}}
\newcommand{\cT}{\mathcal{T}}
\newcommand{\hs}{\hat{s}}
\newcommand{\hx}{\hat{x}}
\newcommand{\hz}{\hat{z}}
\newcommand{\hSigma}{\hat{\Sigma}}
\newcommand{\hOmega}{\hat{\Omega}}
\newcommand{\Ba}{\cB_a}
\newcommand{\Bs}{\cB_s}

\newcommand{\zup}{z^{\uparrow}}
\newcommand{\zUp}{z^{\Uparrow}}

\newcommand{\zDown}{z^{\Downarrow}}
\newcommand{\var}{\Omega}
\newcommand{\varup}{\Omega^{\uparrow}}
\newcommand{\varUp}{\Omega^{\Uparrow}}
\newcommand{\vardown}{\Omega^{\downarrow}}
\newcommand{\varDown}{\Omega^{\Downarrow}}
\newcommand{\hvar}{\widehat{\Sigma}}

\newcommand{\ones}{\mathbf{1}}

\newcommand{\hzup}{\hz^{\uparrow}}
\newcommand{\hzUp}{\hz^{\Uparrow}}
\newcommand{\hzdown}{\hz^{\downarrow}}
\newcommand{\hzDown}{\hz^{\Downarrow}}

\newcommand{\un}{N}

\newcommand{\um}{M}
\newcommand{\sfv}{{d}}

\newcommand{\numcons}{{q}} 
\newcommand{\EQ}{R}
\newcommand{\IEQ}{G}
\newcommand{\eq}{r}
\newcommand{\ieq}{g}

\newcommand{\GLR}{\mathsf{GLR}}
\newcommand{\ECGLR}{\mathsf{ECGLR}}
\newcommand{\FECGLR}{\mathsf{F}\text{-}\ECGLR}
\newcommand{\PECGLR}{\mathsf{P}\text{-}\ECGLR}
\newcommand{\GLS}{{\mathsf{GLS}}}

\newcommand{\ECGLS}{\mathsf{ECGLS}}
\newcommand{\OLS}{\mathsf{OLS}}
\newcommand{\ECOLS}{\mathsf{ECOLS}}

\newcommand{\ECGLRup}{\ECGLR^{\uparrow}}
\newcommand{\ECGLRUp}{\ECGLR^{\Uparrow}}
\newcommand{\ECGLRdown}{\ECGLR^{\downarrow}}
\newcommand{\ECGLRDown}{\ECGLR^{\Downarrow}}
\newcommand{\BHECGLR}[1]{#1\text{-}\ECGLR}

\newcommand{\CombineEstimators}{\mathsf{CombineEstimators}}
\newcommand{\total}{{\sc{total}}}
\newcommand{\full}{{\sc{full}}}

\newcommand{\Tup}{\cT^{\uparrow}}
\newcommand{\TUp}{\cT^{\Uparrow}}
\newcommand{\Tdown}{\cT^{\downarrow}}
\newcommand{\TDown}{\cT^{\Downarrow}}

\newcommand{\child}{\mathrm{child}}

\newcommand{\rt}{\mathsf{rt}}

\newcommand{\pinv}{\dagger}
\newcommand{\nullspace}{\mathfrak{N}}
\newcommand{\poly}{\mathsf{poly}}
\newcommand{\prn}[1]{\left ( #1 \right )}

\newcommand{\algorithmwidth}{\textwidth}

\def\scititle{
	Denoising the US Census:\\
    Succinct block hierarchical regression
}
\title{\bfseries \boldmath \scititle}

\author{
	Badih~Ghazi$^{1}$,
	Pritish~Kamath$^{1}$,
	Ravi~Kumar$^{1}$,
	Pasin~Manurangsi$^{1}$,
	Adam~Sealfon$^{1\ast}$\and
	\small$^{1}$Google Research.\and
	\small$^\ast$Corresponding author. Email: adamsealfon@google.com
}

\begin{document}

\maketitle

\newcommand\blfootnote[1]{%
  \begingroup
  \renewcommand\thefootnote{}\footnote{#1}%
  \addtocounter{footnote}{-1}%
  \endgroup
}

\blfootnote{
This is the author's version of the work. It is posted here by permission of the AAAS for personal use, not for redistribution. The definitive version was published in Science on November 5, 2026, DOI: 10.1126/science.aej1765.
}

\begin{abstract} \bfseries \boldmath
The 2020 United States Census adopted differential privacy to protect individual confidentiality, adding calibrated noise to summary demographic measurements that are then reconciled into population tables.
We developed \BlueDown{}, a postprocessing algorithm that combines techniques from linear regression and constrained optimization while exploiting measurement structure for efficiency.
On 2020 Census data, \BlueDown{} reduced the error of groupings by race at the county level by 21--36\% and of all county-level groupings by 9--58\%; 
at the tract level, it reduced error of groupings by housing type by 38--61\%. 
These gains came at no cost to confidentiality or computational efficiency, and would improve the validity of research relying on this definitive dataset and analyses used to guide the distribution of over \$1.5 trillion in annual federal spending.
\end{abstract}

\noindent
Data from the decennial U.S.\ Census determine the apportionment of
congressional seats under U.S.\ Public Law 94--171~\cite{PL94-171} and inform
the allocation of over \$1.5 trillion in annual federal
funding to state and local governments, nonprofits, businesses, and households~\cite{reamer2020counting}.  Census tabulations are also used for
municipal and infrastructure planning~\cite{NAP29028} and for research in
public health, sociology, political science, and
economics~\cite{krieger2005painting,becker2021computational,abramitzky2024law}.
A central requirement of the census, mandated by Title~13 of the U.S.\
Code~\cite{Title13Privacy}, is that individual responses must remain
confidential.  Historically, the Census Bureau relied on heuristic
data-swapping techniques to protect respondent privacy.  However, simulated
reconstruction attacks conducted after the 2010 Census revealed that those
techniques were seriously flawed, exposing the data of between 50 and 180
million individuals~\cite{abowd2021declaration,Abowd2025}.

This privacy failure prompted the Census Bureau to adopt a mathematically
rigorous disclosure avoidance framework based on differential
privacy~\cite{dwork06calibrating} for the 2020 Census~\cite{abowd2018us}.
Under this framework, calibrated discrete Gaussian noise is added to aggregations of the confidential census returns, producing a set of noisy
measurements that may be released publicly.  Differential privacy provides a
mathematical guarantee that limits the influence of any single individual's data on
the published output, but the requisite noise addition inherently degrades
statistical utility.

To navigate this privacy-utility tradeoff, the Census Bureau developed the \TopDown{} algorithm~\cite{AbowdAC+22} as the core of its 2020
Disclosure Avoidance System (DAS).  \TopDown{} takes over 10~billion
noisy measurements spanning six geographic levels from nation to census
block, and processes them through a massive series of constrained optimization
problems to restore required invariants and produce a consistent Microdata
Detail File.  This file is subsequently aggregated to generate the
Redistricting Data Summary File used for legislative redistricting, as well as the
Demographics and Housing Characteristics File used for community planning and funding. Although \TopDown{} is a
sophisticated procedure, it processes geographic levels
sequentially and does not combine all available measurements optimally when
producing estimates for each geographic unit.

We developed \BlueDown, an alternative processing algorithm that
augments the quadratic optimization core of \TopDown{} with a statistically principled hierarchical regression.
A key component of \BlueDown{}
computes the best linear unbiased estimator (BLUE) for all geographic units by
integrating all measurements, not just those corresponding to a single geographic
unit and the larger geographies that contain it, in time that scales linearly with the number of geographic regions. 
The BLUE estimates are then passed through an additional optimization routine based on \TopDown{} that ensures
that the solution satisfies mandatory structural constraints, such as
non-negative integer counts; in satisfying inequality constraints, our full estimator necessarily cannot remain unbiased.
Algebraic symmetries in the census workload and covariances further compress the
computation, reducing the number of parameters required to represent certain matrices
by nearly a factor of $2{,}000$.
On 2020 Census data, \BlueDown{} reduced estimation error by 
9\% to 44\%
across a suite of histogram queries at the county and tract levels, 
with even larger gains for race and group-quarters breakdowns, 
while satisfying the same privacy guarantee and all structural constraints imposed by the Census Bureau.
Our strong improvements for race queries and for counties and tracts with low population are especially notable, 
because the error levels of these aggregates have been highlighted as especially concerning for downstream studies \cite{kenny2024census,santos2020dp}.

\subsection*{The Census privacy pipeline}

The Census Disclosure Avoidance System (DAS) transforms confidential census
returns into publishable data products through the pipeline illustrated in
Figure~\ref{fig:bluedown-pipeline}.  Raw returns are compiled, deduplicated, imputed,
and tabulated to produce the Census Edited File (CEF), the most authoritative
set of census aggregates.  The CEF is confidential and cannot be released
directly.  To satisfy differential privacy, independent discrete Gaussian
noise~\cite{canonne20discrete} is added to a collection of linear queries on
the CEF, producing a Noisy Measurement File (NMF) that is publicly
releasable but must be post-processed for usability.  A denoising algorithm
(either the currently deployed \TopDown{} or our proposed \BlueDown) ingests the NMF
together with policy-mandated invariants and structural constraints, and
outputs a Microdata Detail File (MDF) containing one record per person.
The MDF is then aggregated to produce the Redistricting Data Summary File.

The geographic hierarchy, known as the geocode tree, has six levels: country,
state, county, tract, block group, and block,
with census blocks as leaves (see Figure~\ref{fig:us-hierarchy}).
The 2020 hierarchy contained 5,892,698 non-empty blocks.  The 50~states and
Washington, D.C.\ each appear as a state-level node; in addition, the Census
Bureau maintains a parallel American Indian and Alaska Native (AIAN) accounting
spine for each of the 37~states that contain tribal areas, yielding 88
state-level nodes in total~\cite{AbowdAC+22}.

The Redistricting Data characterizes each person by four demographic attributes: household or
group quarters type (8~categories, including correctional facilities, nursing
facilities, and student housing), Hispanic/Latino origin (2~categories),
voting age (2~categories, split at age~18), and race (63~categories,
representing all non-empty subsets of six race groups).  These attributes
combine to form $8 \times 2 \times 2 \times 63 = 2{,}016$ demographic buckets.

For every node in the geocode tree, 2,603 linear queries are measured, each
with independently sampled discrete Gaussian noise whose scale is determined by
the allocated privacy budget~\cite{AbowdAC+22}.  These queries include the
total population, counts by individual demographic attributes, various
marginals and cross-tabulations, as well as a detailed query that 
returns the count in every one of the 2,016 buckets (see Tables~\ref{tab:per-node-workload-single-feature} and~\ref{tab:per-node-workload-multi-feature}, Supplementary Materials).

Several classes of constraints must be enforced on the output.  State and
national total populations are released exactly, bypassing the noise
mechanism, to ensure the accuracy of electoral
apportionment. Structural zeros encode
impossible demographic combinations (for example, no minors in nursing
facilities).  All output counts must be non-negative integers, and occupancy
bounds on housing units and group quarters facilities must be
respected~\cite{USCensusBureau2023TopDownHow}.

\subsection*{Optimal hierarchical estimation}

If we temporarily ignore the inequality and integrality constraints, the
post-processing task becomes a generalized linear regression problem.  For each
node~$v$ in the geocode tree, we observe
\[
  y_v = W_v\, x^{\star}_v + \eta_v,
\]
where $x^{\star}_v \in \mathbb{R}^{2016}$ is the vector of true population
counts, $W_v \in \mathbb{R}^{2603 \times 2016}$ is the known {\em workload matrix},
and $\eta_v$ is zero-mean discrete Gaussian noise with known diagonal
covariance~$\Sigma_v$.  The parent-child structure of the geocode tree imposes
additional equality constraints: for every internal node~$v$, the true counts
must equal the sum of the true counts at its children,
$x^{\star}_v = \sum_{u \in \mathrm{children}(v)} x^{\star}_u$.
These estimates and equality constraints describe a least-squares optimization problem
\begin{align*}
\argmin_{x \in (\R^{2016})^{|V|}}
\sum_{v\in V}
(W_v x_v - y_v)^\top \Sigma_v^{-1}(W_v x_v - y_v) \ \ 
\text{subject to}\ \ \EQ x = \eq\ \  \text{(equality constraints)}
\end{align*}
where we are given workload $W = (W_v)$, estimates $y=(y_v)$, and equality constraints $(R,r)$ including the state population totals, structural zeros, consistency constraints, and additional equality constraints implied by housing occupancy bounds.

The statistically optimal solution $x$ to this constrained linear regression problem is the
constrained generalized least-squares estimator, which is the best linear unbiased
estimator (BLUE): among all unbiased estimators that are linear functions of the observations, it minimizes the variance, where the true solution is assumed to satisfy the same constraints~\cite{aitken1935leastsquares}.
Computing the BLUE na\"ively
would require forming and inverting covariance matrices whose dimensions span
the full product of all geographic nodes and demographic buckets, 
which is infeasible at Census scale.

\paragraph{A hierarchical BLUE algorithm.}
Observe that the workload, covariance, and equality constraints decompose into
per-node problems linked only by the parent-child consistency constraints.  A key
insight is that this block hierarchical structure admits an algorithm that computes
the globally optimal BLUE one node at a time using just two passes over the
geocode tree, in time that scales linearly with the number of
nodes~(Figure~\ref{fig:bluedown-pipeline}(B)--(C)). 
This structure was earlier leveraged by Cumings-Menon and Zhuravlev at the U.S.\ Census Bureau~\cite{cumingsmenon2025squaresestimationhierarchicaldata}, who developed a hierarchical BLUE algorithm that is very similar to this step of our full algorithm. The main advantage of our formulation of this step is that it also optimally incorporates equality constraints; see the supplementary text for a more detailed comparison.

In the \textit{bottom-up pass}, we start at the leaves, where each node's
local estimate is simply the solution to the per-node generalized least-squares 
regression problem.  At each internal node, we combine two independent sources of
information: the node's own local estimate and an estimate derived by
aggregating the estimates at its children.  Because the child estimates are
themselves BLUEs and are independent of the node's own measurements, the two
estimates can be merged using an explicit
formula~(Algorithm~\ref{alg:combine-estimates}, Supplementary Materials)
that produces the BLUE of the combined problem, including the linking equality
constraint. This property allows us to compute the BLUE of every geographic node
with respect to all measurements in its subtree.

In the \textit{top-down pass}, information flows in the opposite direction.
For each node~$v$, we compute what is implied about $x^{\star}_v$ by the measurements in the rest of the tree (including the parent, siblings,
and all more distant nodes). We then use the same
merging formula~(Algorithm~\ref{alg:combine-estimates}, Supplementary Materials) to combine this
``outside'' estimate with the subtree estimate from the bottom-up pass.  Concretely, the outside estimate for $v$ is obtained
by combining the outside and local estimates at $v$'s parent and subtracting
the subtree estimates of $v$'s siblings.  After this pass, every node's estimate
incorporates all measurements in the hierarchy.  We proved that the 
resulting estimator is exactly the BLUE of the full constrained regression
(Theorem~1, Supplementary Materials).
We emphasize that this optimality statement refers to the simpler problem excluding inequality constraints, which are incompatible with unbiased estimation and will be addressed below along with integrality.

\paragraph{Succinct matrix algebra.}
Even with linear-time scaling in the size of the geocode tree, a computational
bottleneck remains:
each node carries a $2{,}016 \times 2{,}016$ covariance matrix, and the
algorithm uses products and inversions of these matrices. A second key insight is 
that the measurement structure contains symmetries that can be exploited for large 
efficiency gains. We observed that the 63 census race categories are measured 
symmetrically: every query in the workload matrix either enumerates each race
individually at the same noise scale or aggregates over all races.  Formally, 
the workload decomposes into Kronecker products of a $32 \times 32$ asymmetric
component with the identity matrix and all-ones vector over the 63 symmetric race
categories (see Figure~\ref{fig:bluedown-pipeline}(D)).  This structure is preserved under all operations of the tree 
algorithm, so every intermediate covariance matrix can be represented by a pair of
$32 \times 32$ matrices rather than a single $2{,}016 \times 2{,}016$ matrix.
The resulting compression reduces the representation size by almost a factor of
$2{,}000$, and yields an even larger speedup (${\sim}10^4$-fold) 
for matrix multiplication and inversion, which dominate the running time of the
block hierarchical BLUE algorithm.

\subsection*{Enforcing census constraints}

Although the BLUE computed by the two-pass algorithm satisfies all equality
constraints, 
it does not enforce the inequality and integrality constraints required by the
Census Bureau: counts must be non-negative integers, and housing type totals for each census block must fall between specified upper and lower bounds.
To produce a fully conformant output, the
\BlueDown{} algorithm retains the per-node processing and bottom-up pass
we have just described but replaces the top-down pass with a sequence of constrained
optimization problems closely related to the Census Bureau's TopDown algorithm~\cite{AbowdAC+22}.

At each level of the geocode tree, \BlueDown{} solves a quadratic
program that minimizes the Mahalanobis distance from the bottom-up BLUE
estimates subject to all equality, inequality, and non-negativity constraints,
followed by an integer program that rounds the continuous solution to integer
counts while respecting the same constraints.  These programs are solved using
a standard mixed-integer solver~\cite{gurobi}.  In the absence of inequality
and integrality constraints, this procedure reduces to the linear top-down pass
of the BLUE algorithm. When those constraints are active, it can be viewed as a
generalization of the \TopDown{} optimization~\cite{AbowdAC+22} to
non-diagonal covariance matrices computed in the bottom-up pass. These non-diagonal
covariances capture the correlations between different demographic query estimates 
that result from making more efficient statistical use of the measurements at 
each node in the geocode tree. 

The succinct matrix representation remains essential in this phase.
A direct formulation of the quadratic objectives would involve roughly
$4 \times 2{,}016^2\approx 1.6\times 10^{7}$ monomial terms per child node.  By introducing auxiliary
variables and exploiting the Kronecker structure, we reduced the objective to
$2{,}016$ terms with the addition of $2{,}080$ additional variables and equality constraints, 
yielding an optimization problem of comparable
size and runtime to that of \TopDown{}
(see Supplementary Materials for details).

For both \TopDown{} and \BlueDown, incorporating inequality and
integrality constraints made the resulting estimator biased, in contrast to
the exact BLUE.  However, these constraints encode genuine structural
properties of the census data, and enforcing them empirically improved accuracy
on a wide range of queries.

\subsection*{Accuracy gains on 2020 Census data}

\paragraph{Evaluation methodology.}
The ground-truth CEF is confidential, so we followed a
methodology proposed by the Census
Bureau~\cite{ashmead2025approximatemontecarlosimulation} to evaluate accuracy.
Using the publicly released 2020 MDF as a surrogate
for the CEF, we independently sampled 10 NMFs using the same
noise distribution and privacy budgets as the production 2020 Census.  We ran
both \BlueDown{} and a \TopDown{}
baseline on each NMF and measured accuracy using the $\ell_1$~distance between the
algorithm's output and the MDF for a range of aggregations at each
geographic level.  Error statistics released by the Census Bureau show that the MDF closely approximates the CEF on these metrics~\cite{AbowdAC+22,
USCensusBureau2023detailed}, and the Census Bureau used this same protocol to
construct confidence intervals for \TopDown~\cite{USCensus2020AMCMDF}.
To facilitate comparison, errors
were normalized by the median \TopDown{} error for each query, so that
values below 1.0 indicate improvement.  

\paragraph{Broad accuracy improvements.}
\BlueDown{} significantly improved accuracy on the US detailed query and all comparisons at the state, county, tract, block group, and block levels,
with the largest gains at the county and tract levels (Figures~\ref{fig:general-query-relative-errors} and \ref{fig:more-general-query-relative-errors}, Tables~\ref{tab:first-four-queries}--\ref{tab:last-four-queries}).
Differences for the other queries at the US level were not significant.
At the county level, error reductions for marginal queries ranged from 9\% to 30\%,
and at the tract level from 12\% to 44\%. 
At the state
and block group levels, gains were modest (3\% to 10\% and 0.5\% to 7\%, respectively). 
Figures~\ref{fig:county-and-tract-bias}(A)--(B) and \ref{fig:more-levels-ridgeline} provide ridgeline plots of the full error distributions, showing that the accuracy improvements were consistent. 
The magnitude of improvement varied across queries and levels, due to the Census Bureau's specific privacy budget allocation and enforced constraints. We typically observed larger improvements on queries for which the privacy budget allocation assigned substantial budget to similar queries at lower levels of the tree.
The variance across the 10 replicates was small below the state
level, reflecting the statistical concentration from averaging over
thousands of geographic units.  County and tract geographies are scales that
are highly relevant to legislative redistricting and federal funding allocation,
so these improvements have direct policy significance.

\paragraph{Race and group-quarters queries.}
Figure~\ref{fig:race-and-gq-relative-errors} disaggregates the results for individual race
categories and group-quarters types.
\BlueDown{} reduced error by 21\% to 36\% at the county level and 12\% to 16\% at the tract level for each single-race category and for two or more races.
Such race queries are relevant to
enforcement of the Voting Rights Act and to the allocation of
minority-serving programs, and have been highlighted as groupings on which the Census Bureau's \TopDown{} algorithm sometimes incurred particularly large errors, 
resulting in degradation of public health statistics such as mortality rates by race \cite{santos2020dp,kenny2024census}. \BlueDown's improvements to race queries should translate into improved understanding of public health disparities and other race-related statistics.

Group-quarters queries showed even larger gains: error reductions of 32\% to 61\%
at the county, tract, and block group levels for most facility types, 
including correctional facilities, nursing facilities, student housing, and 
military quarters.
At the block level we observed similar improvement for institutional facility types
and a smaller but significant improvement of around 20\% for non-institutional
facility types. For the household query we saw substantial improvements of 38\% and
12\% at the tract and county levels, respectively, but very small improvements at
the block and block group levels.
These gains stemmed in part from the incorporation into the \BlueDown{} bottom-up pass of
structural zero constraints for blocks with no group quarters facilities, effectively augmenting corresponding privacy budgets for the
parent node and for other housing types.

\paragraph{Bias reduction and small-population geographies.}
Both algorithms exhibited systematic bias as a consequence of the non-negativity
constraint: total population tended to be overestimated in small geographies and
underestimated in large ones. 
Figure~\ref{fig:county-and-tract-bias}(C)--(D) shows the mean signed
error (bias) for total population at the county and tract levels, binned by the true population of each
geographic unit. 
At the county and tract levels, \BlueDown{}
substantially reduced this bias in both the smallest and largest population
bins. The bias reduction
was smaller at the block group and block levels
(Figure~\ref{fig:more-levels-bias}, Supplementary Materials). 
The bias reduction for the smallest population bins at the county and tract levels also corresponds to particularly large improvements in absolute error (Figure~\ref{fig:errors-by-population-bin}, Supplementary Materials). 
Bias and error for small-population geographies have been flagged as significant data quality concerns 
for \TopDown{}, particularly at the county level \cite{santos2020dp,kenny2024census}.

\subsection*{Outlook}

\BlueDown{} demonstrated that substantial accuracy improvements in census
data products could be achievable at no additional privacy cost through algorithmic 
advances.
Optimizing the privacy budget allocation strategy in
conjunction with \BlueDown, rather than for \TopDown{}
alone~\cite{AbowdAC+22}, may yield additional gains in future data releases. Potentially relevant to the 2030 Census, in June 2026 the U.S.\ Department of Commerce announced that noise infusion was no longer consistent with department policies~\cite{commerce2026dao,mervis2026scientists}. Because \BlueDown{} is purely a post-processing
algorithm that operates on the same noisy measurements as \TopDown, 
it can also be applied retrospectively to the 2020 Census to improve the 
accuracy of existing data products.

In tasks that use census estimates directly as numerators or denominators, accuracy improvements should translate directly into practical impact. Errors become particularly consequential when a small number of people carry disproportionate weight. Group quarters queries showed large accuracy improvements with BlueDown, including at small-area levels (e.g.\ correctional facility error reduced by ${>}50\%$). A single dormitory, barracks, or prison may constitute a large share of a rural area's population, and discrepancies impact local service planning. Accuracy also matters disproportionately when statutory thresholds are sharp and a marginal change in a count flips a resource-allocation classification. Federal transit funds, for instance, are apportioned in separate pools above and below 200,000 residents, as determined by census tallies~\cite{uscode_49_5336}. 
Further investigation of the roles of demographic composition, urbanicity, and other factors on estimation accuracy across communities may illuminate downstream policy implications and effects of upstream parameter choices such as privacy budget allocation and geographic spine.

The hierarchical regression framework underlying \BlueDown{} is not
specific to the U.S.\ Census.  It could be applied wherever noisy measurements are
organized in a tree and linked by parent-child consistency constraints.
Potential applications include statistical systems of other national census
agencies, hierarchical survey designs in the social and health sciences, and
attribution systems in which measurements are aggregated across organizational
or geographic hierarchies.

Recent analyses by Su et al.~\cite{su2025usdecennial} have shown that the
2020 Census provides stronger privacy guarantees than originally reported,
using improved privacy accounting.  Because \BlueDown{} operates on the
same noisy data, these tighter guarantees apply to it as well, giving a 
complementary improvement to the overall privacy-accuracy tradeoff.

These results suggest that adopting \BlueDown{} could meaningfully improve the downstream policy decisions and scientific analyses that depend on census data.

\newpage

\begin{figure}
\includegraphics[width=\textwidth]{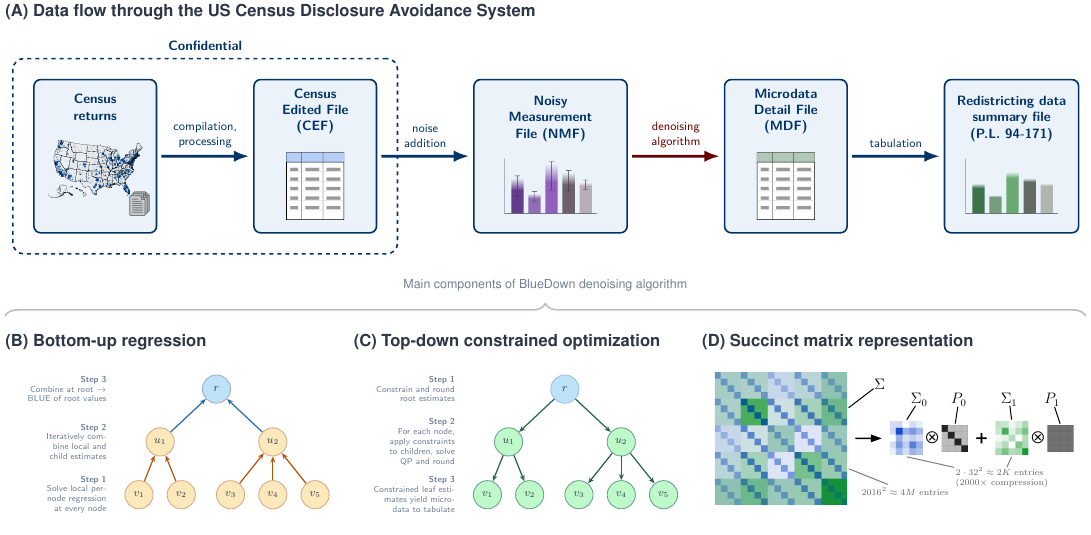}
\caption{\textbf{The BlueDown algorithm and its role in the US Census data processing pipeline.} 
(\textbf{A}) The flow of data through the Census DAS, with the denoising step highlighted. 
(\textbf{B}--\textbf{C}) The main steps of BlueDown consist of a bottom-up regression pass and a top-down 
constrained optimization pass, both of which are linear time in the size of the geocode tree. 
(\textbf{D}) Our lossless compression scheme exploits the structure of the Census measurement system to reduce the computational resources 
needed to work with Census covariance matrices, enabling deployment of BlueDown at Census scale. 
}
\label{fig:bluedown-pipeline}
\end{figure}

\begin{figure}
\begin{center}
\includegraphics[width=0.8\textwidth]{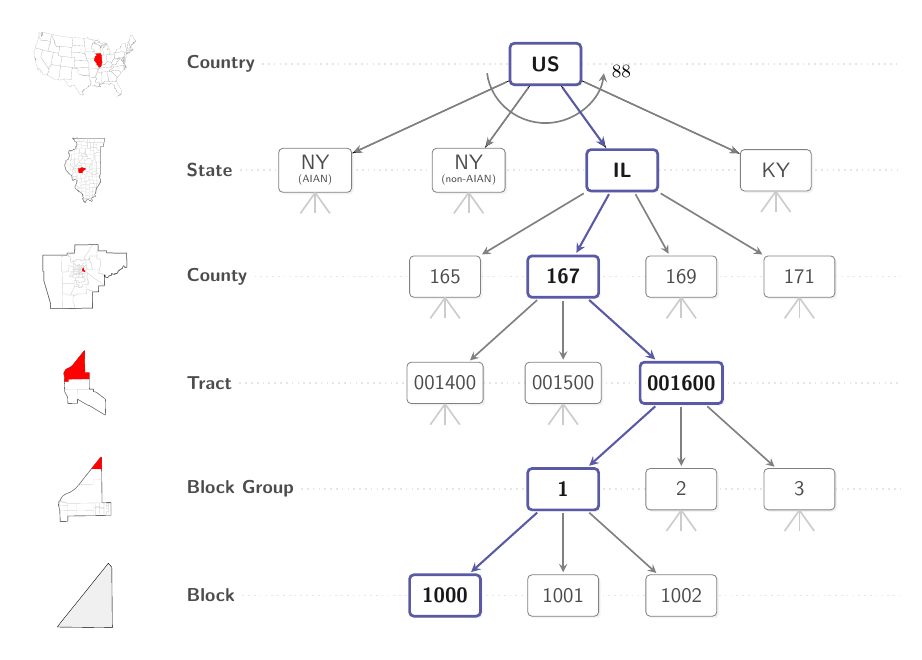}
\end{center}
\caption{\textbf{The U.S.\ Census geocode tree.} A particular path through the tree corresponding to a census block in Springfield, IL is highlighted in blue, and the corresponding geographic regions are highlighted in red. The tree consists of six levels from Country to Block. At the State level, 37 states are split into AIAN (American Indian and Alaska Native) and non-AIAN subtrees, which with the inclusion of Washington, D.C.\ results in a total of 88 State-level geos.}
\label{fig:us-hierarchy}
\end{figure}

\begin{figure}
\begin{center}
\includegraphics[width=\textwidth]{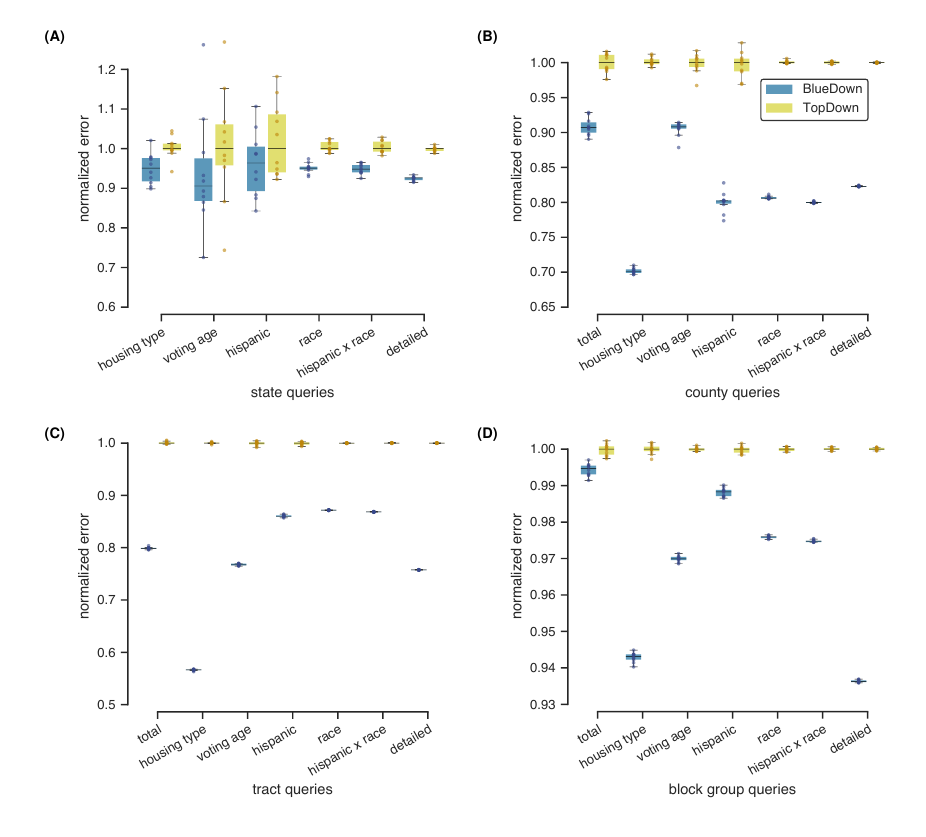}
\end{center}
\vspace{-3em}
\caption{
\textbf{Accuracy improvement for marginal queries.} We compare the accuracy of the estimates produced by \BlueDown{} and \TopDown{} over ten noisy measurement files (NMF) independently sampled based on U.S.\ Census data. Each panel corresponds to one of the geographic levels (\textbf{A}) State, (\textbf{B}) County, (\textbf{C}) Tract, and (\textbf{D}) Block Group. For each NMF we report the error of each algorithm for each of seven marginal aggregations, averaged over all nodes of that level and normalized by the median error of \TopDown{} for that query and level. Distributions are visualized as box plots showing the median and interquartile range (IQR), where error bars (whiskers) extend up to 1.5$\times$ IQR and individual NMF replicates are overlaid as points. Error means and standard deviations are reported without normalization in Tables~\ref{tab:first-four-queries}--\ref{tab:last-four-queries}. At the State level the total query is constrained to be exact under both algorithms, and is omitted from the plot.}
\label{fig:general-query-relative-errors}
\end{figure}

\begin{figure}
\centering
\includegraphics[width=\textwidth]{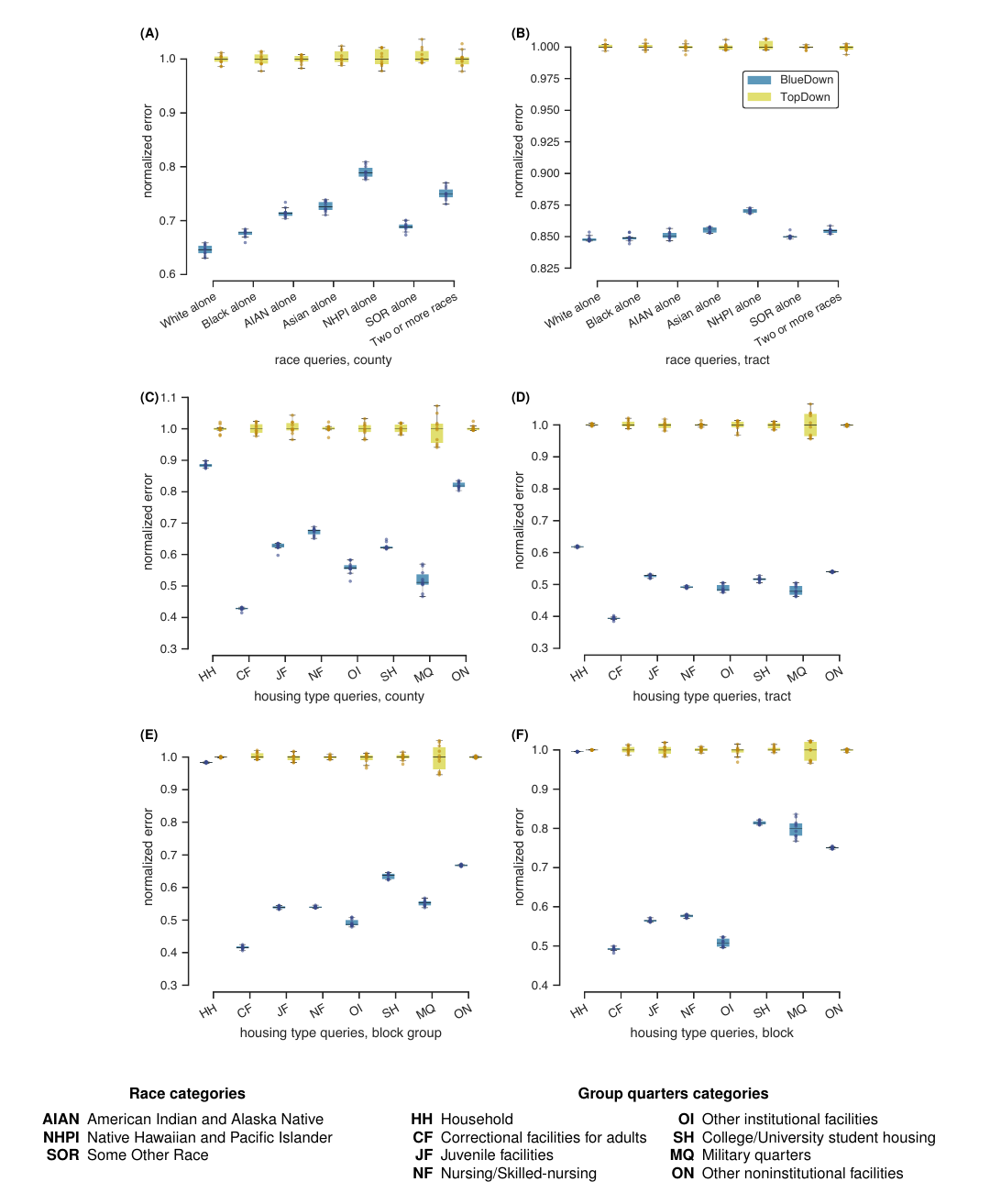}
\vspace{-3em}
\caption{
\textbf{Accuracy improvement for race and group quarters queries}. Using the same methodology as Figure~\ref{fig:general-query-relative-errors}, 
we report average normalized accuracy improvements for each of the basic census race categories at the (\textbf{A}) County and (\textbf{B}) Tract levels, and for each housing and group quarters type at the (\textbf{C}) County, (\textbf{D}) Tract, (\textbf{E}) Block Group and (\textbf{F}) Block levels.}
\label{fig:race-and-gq-relative-errors}
\end{figure}

\begin{figure}
\begin{center}
\includegraphics[width=\textwidth]{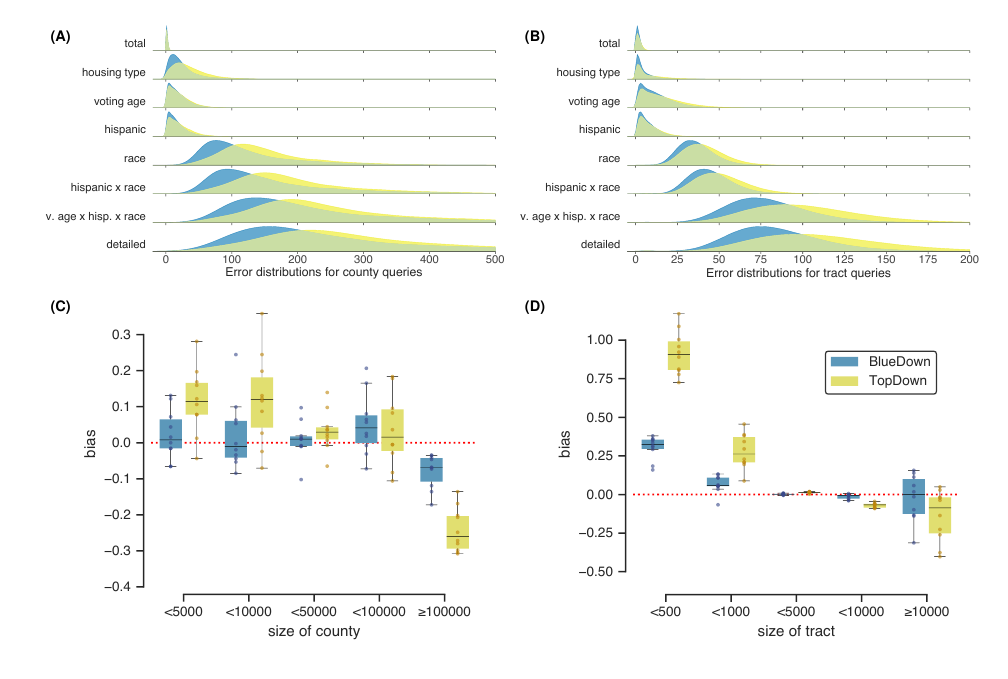}
\end{center}  
\caption{
\textbf{Error distribution and bias at the county and tract levels.}
(\textbf{A}--\textbf{B}) Error distribution of both \BlueDown{} and \TopDown{} for each marginal query across all geos at that level is visualized using Gaussian kernel density estimation. The green region indicates the overlap in the two distributions.
(\textbf{C}--\textbf{D}) Mean error (bias) for the total population count of each population bin at each level, for each of the ten independent noisy measurement files, presented as box plots defined as in Figure~\ref{fig:general-query-relative-errors}. Population bins are mutually exclusive. The dotted red line corresponds to zero bias.}
\label{fig:county-and-tract-bias}
\end{figure}

\clearpage 

\bibliography{main.bbl}
\bibliographystyle{sciencemag}

\section*{Acknowledgments}
We are grateful to U.S.\ Census Bureau statisticians Ryan Cumings-Menon for extensive feedback and guidance, Pavel Zhuravlev for clarifications on the DAS data schemas and implementation details, and Philip Leclerc for several helpful pointers. Gemini 2.5 Pro and internal Google models were used in generation of manually verified Python scripts used to produce plots.

\paragraph*{Funding:}
All authors are employed by Google LLC and have no additional support to declare.

\paragraph*{Author contributions:}
Conceptualization: BG, PK, RK, PM, and AS;
Data curation: AS;
Formal analysis: AS and PK;
Methodology: AS, PK, BG, RK, and PM;
Software: AS;
Project administration: RK, BG, and PK;
Validation: PK, PM, RK, and AS;
Visualization: AS and PK;
Writing: AS, PK, RK and PM

\paragraph*{Competing interests:}
None declared.

\makeatletter
\g@addto@macro\UrlBreaks{%
  \do\/%
  \do\-%
  \do\_%
  \do\.%
  \do\?%
  \do\&%
  \do\=%
}
\makeatother

\paragraph*{Data and materials availability:}
All data and code needed to evaluate the conclusions in the paper are provided in the Harvard Dataverse \cite{DVN/JCZBBV_2026}. The data we used in the experiments were obtained from the download locations in the official US Census document available at:
\begin{itemize}
    \item 2020 Census Noisy Measurement Files for Redistricting Data: \url{https://www2.census.gov/programs-surveys/decennial/2020/data/01-Redistricting_File--PL_94-171/00-2020-Redistricting-Noisy-Measurement-File/Noisy-Measurement-File-2020-Census-Redistricting-Data.pdf},
    \item 2020 Census Privacy-Protected Microdata File (PPMF): \url{https://www.census.gov/data/tables/2024/dec/2020-census-ppmf.html}.
\end{itemize}

Our code is also available on GitHub at \url{https://github.com/google-research/census-bluedown/}.

\subsection*{Supplementary materials}
Materials and Methods\\
Supplementary Text\\
Figs. S1 to S6\\
Algorithms 1 to 7\\
Tables S1 to S4\\
References \textit{(28-\arabic{enumiv})} \\

\newpage

\renewcommand{\thefigure}{S\arabic{figure}}
\renewcommand{\thetable}{S\arabic{table}}
\renewcommand{\theequation}{S\arabic{equation}}
\renewcommand{\thepage}{S\arabic{page}}
\setcounter{figure}{0}
\setcounter{table}{0}
\setcounter{equation}{0}
\setcounter{page}{1} 

\begin{center}
\thispagestyle{empty}

\vspace{1cm}

{\LARGE Supplementary Materials for}\\[-1mm]

\subsection*{Denoising the US Census: Succinct block hierarchical regression}

Badih~Ghazi,
Pritish~Kamath,
Ravi~Kumar,
Pasin~Manurangsi,
Adam~Sealfon$^\ast$\\
\small$^\ast$Corresponding author. Email: adamsealfon@google.com
\end{center}

\noindent Materials and Methods\\
Supplementary Text\\
Figures S1 to S6\\
Algorithms 1 to 7\\
Tables S1 to S4\\

\newpage

\subsection*{Materials and Methods}
We first describe a formalization of the Census DAS. For a more in-depth description of the Census DAS and the \TopDown{} algorithm, see \cite{AbowdAC+22} and \ref{sec:topdown-comparison} in the supplementary text below. This algorithm has been the state of the art for postprocessing of census redistricting data.

\paragraph{Data pipeline.}

A simplified overview of the processing pipeline of the Census Disclosure Avoidance System (DAS) is outlined in Figure~\ref{fig:bluedown-pipeline}(A). The raw census returns are compiled, de-duplicated, imputed, and tabulated to generate the \emph{Census Edited File (CEF)}, which is the most authoritative set of census aggregates but is sensitive and confidential.  To achieve differential privacy (DP), discrete Gaussian noise (Definition \ref{def:discrete-gaussian}) is added to the CEF data to generate the \emph{Noisy Measurement File (NMF)}, which is publicly released but must be processed further for usability. A post-processing \emph{denoising algorithm}  (either \TopDown{} or our proposed \BlueDown) ingests the NMF, final tabulation geographic definitions, policy-mandated invariants, and expert-defined edit constraints and structural zeros as its inputs.  The algorithm's primary output is a \emph{Microdata Detail File (MDF)}, containing one record for each person and housing unit, which is subsequently aggregated to produce the 2020 Census Redistricting Data (P.L. 94-171) Summary File. 

\paragraph{Geocodes and demographics.}
Each person in the raw census data is associated with a demographic bucket and to a leaf in the tree of geographies known as the geocode tree. We describe these next.

\begin{itemize}
\item The \emph{geocode tree} $\cT$ has six levels: {\em Country}, {\em State}, {\em County}, {\em Tract}, {\em Block Group}, and {\em Block}, where the leaves $\cL$ correspond to the different {\em Blocks}.
See Figure~\ref{fig:us-hierarchy} for an illustration of one path in the tree. 37 of the 50 states in the US are accounted for in two ways, one {\em spine} for American Indian and Alaska Native (AIAN) population, and one for non-AIAN population. With the inclusion of Washington, D.C., this leads to a total of 88 state-level regions at depth one in the tree.
The number of non-empty blocks in the 2020 hierarchy was 5,892,698.
(In some cases, particularly in the AIAN spine, some block groups or even tracts or counties may contain only a single block, so the tree can be simplified by omitting single child nodes that correspond to an identical geography as the parent. Such block groups should then be treated as leaves in the tree. For notational simplicity our discussion describes all leaves as at the block level, but our algorithm works equally well if some leaves are at higher levels.)

\item The \emph{demographic buckets} $\cB$ are specified by the following:
	\begin{itemize}
	\item {\em Household or Group Quarters Type} with 8 types,
    (1) household;
    (2) correctional facilities for adults;
    (3) juvenile facilities;
    (4) nursing facilities/skilled-nursing facilities;
    (5) other institutional facilities;
    (6) college/university student housing;
    (7) military quarters;
    (8) other noninstitutional facilities.
    These 8 types are additionally categorized into three groups: {\em individual quarters} $\{1\}$; {\em institutional group quarters} $\{2, 3, 4, 5\}$; and {\em non-institutional group quarters} $\{6, 7, 8\}$.
	\item {\em Hispanic/Latino origin} with 2 types: Hispanic/Latino; not Hispanic/Latino,
    \item {\em Voting Age}, with 2 types: age 17 or younger on April 1, 2020; age 18 or older on April 1, 2020,
	\item {\em Census Race} with 63 types, consisting of every combination (except ``none of the above'') of
    (1) Black/African American,
    (2) American Indian/Native Alaskan,
    (3) Asian,
    (4) Native Hawaiian/Pacific Islander,
    (5) White, and
    (6) some other race.
	\end{itemize}
In total, $|\cB| = 8 \cdot 2 \cdot 2 \cdot 63 = 2016$, and we denote a bucket $b \in \cB$ as a tuple $(b_1, b_2, b_3, b_4) \in [8] \times [2] \times [2] \times [63]$.
\end{itemize}
For each node $u \in \cT$ and each bucket $b \in \cB$, we use $x^\star_{u, b} \ge 0$ to denote the true count of the people living in the location represented by the node $u$, and belonging to the demographic bucket $b$, according to the CEF. We use $x^\star_u \in \Z_{\ge 0}^{\cB}$ to denote the $|\cB|$-dimensional vector of counts associated to all the demographic buckets for node $u$, and $x^\star \in \Z_{\ge 0}^{\cT \times \cB}$ to denote the combined vector across all nodes and demographic buckets.

\paragraph{Input measurements.} For each node $u$ in the geocode tree, 
we observe noisy measurements from various aggregate {\em queries} over combinations of demographic types. Each query result is observed with the addition of discrete Gaussian noise~\cite{canonne20discrete}; the scale of the noise is different for different queries and different levels, as per the privacy budget allocated across the queries according to the specification described in~\cite{AbowdAC+22} and in the NMF dataset itself. We denote such noisy measurements as $y_u \gets W_u x^\star_u + \xi_u$ where $W_u$ is the \emph{workload matrix} for node $u$, and $\xi_u$ is the added noise with covariance $\Sigma_u$.  

\begin{itemize}
\item {\em Workload matrix $W_u$.} Each per-node workload $W_u$ is a collection of different {\em types} of linear queries applied to $x^\star_u$ for the node $u \in \cT$. Each workload matrix $W_u \in \R^{2603 \times 2016}$ has 2603 queries in total with the different query types described in Tables~\ref{tab:per-node-workload-single-feature} and~\ref{tab:per-node-workload-multi-feature}. Due to the presence of the DETAILED query, which corresponds to an identity submatrix, it is immediate that the matrix $W_u$ has full (column) rank.

\item {\em Covariance matrix $\Sigma_u$.} Each per-node noise covariance matrix $\Sigma_u \in \R^{2603 \times 2603}$ is diagonal. The specific values in $\Sigma_u$ depend on the node $u$ and the query type, i.e., for each node $u$, all queries described in a single row of Tables~\ref{tab:per-node-workload-single-feature} and~\ref{tab:per-node-workload-multi-feature} are measured with the same noise scale; see Tables 4 and 6 in~\cite{AbowdAC+22} for details.

(For the most part, the noise scale for a particular type of query depends only on the level of $u$ in $\cT$, e.g., State, County, etc. There are a few exceptions where a parent node has exactly one child, in which case the two are collapsed into a single node with smaller measurement variance. Our algorithm does not rely on any assumptions about the structure of the privacy budgeting across measurements and levels, except that the workload matrix must be full rank, i.e., the DETAILED query must have positive privacy budget.)

\end{itemize}

A notable property is that for each query type, the noise scale is symmetric in the value of $b_4$. In particular, in Tables~\ref{tab:per-node-workload-single-feature} and~\ref{tab:per-node-workload-multi-feature}, either there are queries that aggregate over all values of $b_4$, or there are individual queries for each $b_4$, all of which are measured with the same noise level. By contrast, such a symmetry does not hold for $b_1$ due to the presence of the HHINSTLEVELS query type. We critically leverage this symmetry for $b_4$ by working with a \emph{succinct representation} of covariance matrices~(Definition~\ref{def:succinct}), represented using two $32 \times 32$ dimensional matrices, making matrix operations considerably more efficient. We discuss this in more detail in the section ``Exploiting symmetries in per-node workload matrices'' in the Supplementary Text. 

\paragraph{Invariants and constraints.} 
To ensure the complete accuracy of electoral apportionment, the statewide total populations are reported exactly, bypassing the DP mechanism. Additionally, the total number of housing units and occupied group quarters facilities of each type is released for every block, as well as certain combinations of demographics that are deemed impossible.
The implied equality and inequality constraints are then assumed to hold on the collected data and enforced on the output, as described below.

\begin{itemize}
\item {\em Equality Constraints.} For each geographic region with zero households or occupied group quarters of a particular type, the number of individuals in that type of housing is constrained to be zero. The number of individuals below the age of 18 in the group quarters type ``Nursing facilities / skilled nursing facilities'' is always constrained to be zero. The total number of individuals  residing in each state (across both the AIAN and the non-AIAN spines) is constrained to be equal to the released enumerated count for that state, and similarly for the US as a whole \cite{USCensusBureau2023TopDownHow}.
\item {\em Inequality Constraints.} Every occupied group quarters facility must contain at least 1 individual, and every housing unit and occupied group quarters facility must contain at most 99,999 individuals. The number of households and occupied group quarters of each type in each geographic region is publicly released, implying upper and lower bounds on the total number of individuals in each household or group quarters type for each region \cite{USCensusBureau2023TopDownHow}. Finally, all counts are required to be non-negative.
\end{itemize}

\paragraph{Accuracy metrics.} We evaluate the accuracy of the \BlueDown{} and \TopDown{} algorithms using the following family of metrics, following \cite{AbowdAC+22}. Let $\mathcal{G}$ be a set of geographic areas of interest, e.g.\ all regions at one of the six geographic levels. A query over a set of geographic areas $\mathcal{G}$ is described by a partition of subsets $\mathcal{B}_1, \ldots, \mathcal{B}_k$ of the demographic buckets $\mathcal{B}$, where e.g.\ $k=1$ for the total population query, $k=63$ for the race query, and $k=2016$ for the detailed query. Given such a partition, for each $g\in\mathcal{G}$, let 
\[x_{g, q_1}^\star, \ldots, x_{g, q_k}^\star\]
be the true count in area $g$ for each demographic bucket subset, and let 
\[x_{g, q_1, t}, \ldots, x_{g, q_k, t}\] 
be the count in that area for each demographic bucket subset for the $t$th replicate run of one of the algorithms (\BlueDown{} or \TopDown{}). Then the error for the $t$th replicate run is given by
\[
\text{err}_t(\mathcal{G}, \mathcal{B}_1, \ldots, \mathcal{B}_k)
= \frac{1}{|\mathcal{G}|}\sum_{g \in\mathcal{G}} \sum_{j=1}^k |x_{g, q_j, t} - x_{g, q_j}^\star|.
\]
The mean and standard deviation of the errors across the $T=10$ runs as reported in Tables~\ref{tab:first-four-queries}--\ref{tab:last-four-queries} are simply the mean and standard deviation of the values 
$\text{err}_1(\mathcal{G}, \mathcal{B}_1, \ldots, \mathcal{B}_k), \ldots, \text{err}_T(\mathcal{G}, \mathcal{B}_1, \ldots, \mathcal{B}_k)$. 
The bias as reported in Figures~\ref{fig:county-and-tract-bias}(\textbf{C})--(\textbf{D}) and \ref{fig:more-levels-bias} is given by the same expression, but omitting the absolute value around the expression $x_{g, q_j, t} - x_{g, q_j}^\star$. 
Figure~\ref{fig:errors-by-population-bin} plots the individual values $\text{err}_1(\mathcal{G}, \mathcal{B}_1, \ldots, \mathcal{B}_k), \ldots, \text{err}_T(\mathcal{G}, \mathcal{B}_1, \ldots, \mathcal{B}_k)$.
The normalized errors reported in Figures~\ref{fig:general-query-relative-errors}--\ref{fig:race-and-gq-relative-errors} and \ref{fig:more-general-query-relative-errors} are given by taking the $T$ error values $\text{err}_1(\mathcal{G}, \mathcal{B}_1, \ldots, \mathcal{B}_k), \ldots, \text{err}_T(\mathcal{G}, \mathcal{B}_1, \ldots, \mathcal{B}_k)$ for either \BlueDown{} or \TopDown{}, and dividing each by the median of the corresponding $T$ values for \TopDown{}; this normalization by the typical \TopDown{} error allows visualization of \BlueDown{}'s improvement for different queries on the same axis, as well as the run to run variability in error of both algorithms. Finally, the ridgeline plots in Figures~\ref{fig:county-and-tract-bias}(\textbf{A})--(\textbf{B}) and \ref{fig:more-levels-ridgeline} display distributions of 
$\sum_{j=1}^k |x_{g, q_j, t} - x_{g, q_j}^\star|$ over all $g\in\mathcal{G}, t\in[T]$.

Where percent improvement ranges are given in the text, they refer to the median across replicate runs of the errors for each run, with conservative directed rounding applied to the reported range of medians for different partitions.

\paragraph{Statistical tests.} In Tables \ref{tab:first-four-queries} and \ref{tab:last-four-queries}, for each of the two algorithms \BlueDown{} and \TopDown{} we report the mean and standard deviation across the ten replicate runs of the average error for a particular query and geographic level. We evaluated six geographic levels (US, state, county, tract, block group, and block) and eight query types sliced by different combinations of zero or more of the demographic buckets (total population, housing type, voting age, Hispanic origin, race, Hispanic origin $\times$ race, voting age $\times$ Hispanic origin $\times$ race, and the detailed query, which is housing type $\times$ voting age $\times$ Hispanic origin $\times$ race). For each combination, we additionally report nominal p-values for a two-tailed paired t-test over the average error values for each algorithm across the ten replicate runs.

\paragraph{Generation of illustrative diagrams.} Gemini 3.0 Pro and internal Google models were used to generate Python and tikz code for drawing manually verified diagrams and illustrative maps in Figures~\ref{fig:bluedown-pipeline} and \ref{fig:us-hierarchy}, using Census Bureau geometry files \cite{uscensus_tiger,uscensus_popest2019}.

\subsubsection*{Overview of \BlueDown{} algorithm}

Given the noisy measurements over various aggregate queries at each node $u$, the goal is the following:
\begin{center}
\textsl{
Obtain consistent estimates $\hx_u \in \Z_{\ge 0}^{\cB}$ satisfying all invariants and constraints, with low error.
}
\end{center}
Here, ``consistent'' means that the sum of estimated counts in all blocks below $u$ agrees with $\hx_u$ for each query.
At a high level, if we ignore the inequality, equality, and integrality constraints, then the problem is as follows: given $Wx^\star + \eta$ where $W$ is the combined workload matrix across all nodes and $\eta$ is the discrete Gaussian noise added for privacy, find an estimate $\hat{x}$ of $x^{\star}$.

This description corresponds to a generalized linear regression problem, but with additional considerations. The first complication is that, as mentioned above, we have linear equality constraints in the form of per-node equality constraints, the statewide total constraints, and consistency with respect to the geocode tree.  However, this ``equality-constrained'' linear regression problem can still be solved using standard techniques, as we describe, yielding a so-called \emph{best linear unbiased estimator (BLUE)}.  The resulting method is not computationally efficient for the census problem, since even explicitly materializing the matrix $W$ is already infeasible at Census scale: it consists of roughly 16 billion measurements over 12 billion variables. Fortunately, the workload matrix $W$ and constraints have a block-hierarchical structure formalized below that we exploit to achieve an efficient algorithm whose running time scales linearly with the size of the tree. To do so, we develop efficient primitives described below that can optimally combine two independent BLUEs related by equality constraints; the resulting estimator is the BLUE of the combined constrained linear regression. Using this primitive, we develop an efficient algorithm~(Algorithm \ref{alg:tree-post-processing}) for obtaining the BLUE under the block-hierarchical structure of $W$, in two passes over the geocode tree.
The bottom-up pass recursively combines the estimates for a node with the estimators of its children, thereby computing the BLUE for each node with respect to the set of measurements corresponding to itself and its tree descendants.
In the top-down pass, we combine the estimator for each node with estimators for its parents and siblings. This gives us the final BLUE for the equality-constrained generalized linear regression problem.

This formulation of the algorithm is still not efficient enough at the scale of the US Census. This is due to the fact that the covariance matrices at each node are quite large ($2016 \times 2016$ dimensional). We exploit symmetries of these matrices and show how to represent and perform operations over these matrices succinctly using a pair of $32 \times 32$ dimensional matrices. This yields a nearly 2000-fold reduction in the representation size of the covariance matrices, and an even larger improvement in the time needed to compute matrix inverses and products, which are the operations that dominate the performance of Algorithm~\ref{alg:tree-post-processing}.

Using this succinct matrix representation, the modified Algorithm~\ref{alg:tree-post-processing} is highly performant but does not yet satisfy the full suite of constraints required in the Census setting. 
Our final \BlueDown{} algorithm~(Algorithm~\ref{alg:bluedown}) applies a heuristic modification to the top-down pass of Algorithm~\ref{alg:tree-post-processing} to accommodate the inequality and integrality constraints. 
This modified top-down pass consists of solving a sequence of quadratic programming and integer programming problems, and can be viewed as a generalization of the US Census \TopDown{} algorithm to non-diagonal covariance matrices. 
Our succinct matrix representation is again crucial here for efficiency of computation.
The incorporation of these constraints yields a dataset that fully conforms to all Census data requirements.
Moreover, while satisfying inequality constraints inherently results in a biased estimator, the incorporation of these constraints substantially improved error on many queries since the constraints were derived from known structural properties of the Census data.

\subsection*{Supplementary Text}

\setcounter{section}{0}
\renewcommand{\thesection}{\S \arabic{section}}
\renewcommand{\thesubsection}{\S \arabic{section}.\arabic{subsection}}
\renewcommand{\thesubsubsection}{\S \arabic{section}.\arabic{subsection}.\arabic{subsubsection}}

\makeatletter
\renewcommand{\section}{\@ifstar\secstar\secnum}
\def\secstar#1{%
  \par\addvspace{3.5ex \@plus 1ex \@minus .2ex}%
  \noindent\underline{#1}%
  \par\nobreak\vspace{1.5ex \@plus .2ex}%
  \@afterindentfalse\@afterheading
}
\def\secnum#1{%
  \par\addvspace{3.5ex \@plus 1ex \@minus .2ex}%
  \refstepcounter{section}%
  \noindent\underline{\thesection.\quad #1}%
  \par\nobreak\vspace{1.5ex \@plus .2ex}%
  \@afterindentfalse\@afterheading
}

\renewcommand{\subsection}{\@ifstar\subsecstar\subsecnum}
\def\subsecstar#1{%
  \par\addvspace{3.25ex \@plus 1ex \@minus .2ex}%
  \noindent\underline{#1}%
  \par\nobreak\vspace{1.2ex \@plus .2ex}%
  \@afterindentfalse\@afterheading
}
\def\subsecnum#1{%
  \par\addvspace{3.25ex \@plus 1ex \@minus .2ex}%
  \refstepcounter{subsection}%
  \noindent\underline{\thesubsection.\quad #1}%
  \par\nobreak\vspace{1.2ex \@plus .2ex}%
  \@afterindentfalse\@afterheading
}

\renewcommand{\subsubsection}{\@ifstar\subsubsecstar\subsubsecnum}
\def\subsubsecstar#1{%
  \par\addvspace{3.25ex \@plus 1ex \@minus .2ex}%
  \noindent\underline{#1}%
  \par\nobreak\vspace{1.2ex \@plus .2ex}%
  \@afterindentfalse\@afterheading
}
\def\subsubsecnum#1{%
  \par\addvspace{3.25ex \@plus 1ex \@minus .2ex}%
  \refstepcounter{subsubsection}%
  \noindent\underline{\thesubsubsection.\quad #1}%
  \par\nobreak\vspace{1.2ex \@plus .2ex}%
  \@afterindentfalse\@afterheading
}
\makeatother

The supplementary text is organized as follows. Sections~\ref{sec:prelims}--\ref{sec:abstract-tree-structure} provide technical background and notation.
In~\ref{sec:prelims} we provide necessary background concerning Kronecker algebra, generalized linear regression, and constrained regression.
In~\ref{sec:formal-problem} we formally state and introduce notation for the Census optimization problem.
In~\ref{sec:abstract-tree-structure} we formalize the block hierarchical structure leveraged by our algorithms.

Sections~\ref{sec:equality-constraints}--\ref{sec:symmetries} present our block hierarchical constrained least-squares regression algorithm, which is one of the key components of our full algorithm. 
In~\ref{sec:equality-constraints} we introduce two special structured types of constrained regression problems that we use as building blocks of our algorithm.
In~\ref{sec:optimal-linear} we describe our least-squares regression algorithm and prove that it computes the BLUE. 
In~\ref{sec:symmetries} we show how to exploit symmetries in workload and covariance matrices to achieve significant improvements in computational efficiency.

In~\ref{sec:top-down-pass} we present the BlueDown algorithm, which combines the linear regression framework of~\ref{sec:optimal-linear} with a constrained optimization approach that can incorporate the integrality and inequality constraints. We also show how to leverage the succinct representation framework of~\ref{sec:symmetries} to achieve similar computational improvements on the optimization problem as well.

Sections~\ref{sec:topdown-comparison}--\ref{sec:additional-related-work} present additional discussion and comparisons.
In~\ref{sec:topdown-comparison} we describe in detail the Census Bureau's TopDown algorithm and compare it to BlueDown.
In~\ref{sec:efficiency} we discuss the efficiency of our algorithm.
Finally, in~\ref{sec:additional-related-work} we 
give an overview of additional related work.

\section{Technical background}
\label{sec:prelims}

Let $\ones_\sfv$ denote the all-ones (column) vector of dimension $\sfv$; when the dimension is clear from context, we will sometimes omit $\sfv$.

\begin{definition}[Matrix Pseudoinverse]\label{def:pseudo-inverse}
For a matrix $A \in \R^{m \times n}$, its \emph{(Moore--Penrose) pseudoinverse} is denoted $A^{\dagger} \in \R^{n \times m}$ and is the unique matrix satisfying the following conditions: 
(i) $A A^{\dagger} A = A$,
(ii) $A^{\dagger}  A A^{\dagger} = A^{\dagger}$,
(iii) $(A A^{\dagger})^\top = A A^{\dagger}$, and
(iv) $(A^{\dagger} A)^\top = A^{\dagger} A$.

If $A$ is square and invertible, then $A^{\dagger} = A^{-1}$, the usual matrix inverse.
\end{definition}

\paragraph{Kronecker algebra.}
For any two matrices $A \in \R^{m_1 \times n_1}$ and $B \in \R^{m_2 \times n_2}$, let $A \otimes B \in \R^{m_1m_2 \times n_1n_2}$ denote the {\em Kronecker product} given by the block matrix
\[
A \otimes B = \begin{bmatrix}
    A_{1,1} B & \ldots & A_{1, n_1} B \\
    \vdots & \ddots & \vdots \\
    A_{m_1,1} B & \ldots & A_{m_1, n_1} B \\
\end{bmatrix}\,,
\]
or more explicitly, $(A \otimes B)_{i, j} = A_{i_1, j_1} \cdot B_{i_2, j_2}$, where $i = m_2 (i_1 - 1) + i_2$ and $j = n_2 (j_1 - 1) + j_2$.

It is often possible to operate on Kronecker product matrices in terms of their individual components without having to materialize the full product matrix.  We next state some identities that we repeatedly use.

\begin{fact}[Kronecker Algebra]\label{fact:kronecker}
\mbox{}
\begin{itemize}
\item $(A \otimes B) (C \otimes D) = AC \otimes BD$ (whenever the matrix dimensions make sense, namely when the number of columns of $A$ equals the number of rows of $C$, and likewise for $B$ and $D$).
\item $(A \otimes B)^{\dagger} = A^{\dagger} \otimes B^{\dagger}$; in particular $A \otimes B$ is invertible if and only if $A$ and $B$ are invertible.
\item $(A \otimes B) \cdot \mathsf{vec}(Z) = \mathsf{vec}(A Z B^\top)$, for $A \in \R^{m_1 \times n_1}$, $B \in \R^{m_2 \times n_2}$ and $Z \in \R^{n_1 \times n_2}$. Here, $\mathsf{vec}(\cdot)$ flattens a matrix into a vector using row-major ordering. This identity is sometimes known as ``Roth's Column Lemma'' or the ``vec trick''~\cite{roth1934direct}.
\end{itemize}
\end{fact}

\paragraph{Discrete Gaussian distribution.}
The discrete Gaussian distribution is the discrete analog of the continuous Gaussian distribution, defined over integers.
\begin{definition}[Discrete Gaussian Distribution; see, e.g.,  \cite{canonne20discrete}]\label{def:discrete-gaussian}
Let $\mu \in \R$ be the mean and $\sigma$ be a parameter controlling the standard deviation. The \emph{discrete Gaussian distribution} over $\Z$, with parameters $\mu, \sigma^2$, is the probability distribution given by 
\[\textstyle
\Pr[X = x] \propto \exp \left(-\frac{(x-\mu)^2}{2 \sigma^2} \right), \mbox{ for } x \in \Z.
\]
\end{definition}

\subsection{Constrained generalized linear regression}\phantom{}

\noindent The main optimization problem we work with (discussed formally in the next section) can be described as a \emph{constrained generalized linear regression} problem with various constraints of equality, inequality, and integrality. Before describing in detail our main problem, we discuss relevant background on the classical topic of generalized linear regression both with and without equality constraints, which serves as the foundation of our work.

\paragraph{\boldmath Generalized linear regression ($\GLR_{W, \Sigma}$).}
The {\em generalized linear regression} problem $\GLR_{W, \Sigma}$, parameterized by a workload matrix $W \in \R^{m \times n}$ and a noise covariance $\Sigma \in \R^{m \times m}$, is the problem of estimating an unknown vector $x^\star \in \R^n$ given observation $y \sim Wx^\star + \eta$, where $\eta$ is drawn from a distribution with zero mean and covariance $\Sigma$; for convenience, we let $\GLR_{W,\Sigma}(x^\star)$ denote the distribution of $y \in \R^m$. Note that while the most common setting considers Gaussian noise, in our setting we will instead be using noise drawn from a discrete Gaussian distribution.
The special case of $\Sigma = \sigma^2 I$ is often referred to as (ordinary) linear regression.

For any $\GLR_{W,\Sigma}$ problem, we will always assume that $W$ has full column rank (i.e.\ rank $n$) and that $\Sigma$ is positive definite.
We study {\em estimators} $\hx : \R^m \to \R^n$ of $x^\star$ under $\GLR_{W,\Sigma}$; while $\hx$ is technically a function, we will interchangeably use it to denote the random variable $\hx(y)$ for $y \sim \GLR_{W,\Sigma}(x^\star)$. For any estimator $\hx$, we can associate a covariance matrix $\hSigma := \Ex_y[(\hx - \Ex \hx) (\hx - \Ex \hx)^\top] \in \R^{n \times n}$.

\begin{definition}[BLUE]\label{def:blue}
The \emph{Best Linear Unbiased Estimator (BLUE)} for $x^\star$ under $\GLR_{W, \Sigma}$, denoted $\hx_{\mathrm{BLUE}}$, is an estimator that satisfies the following conditions (where $\hSigma_{\mathrm{BLUE}}$ denotes the covariance of $\hx_{\mathrm{BLUE}}$):
\begin{enumerate}
    \item {\em Linearity}: $\hx_{\mathrm{BLUE}}(y)$ is a linear function of $y$, i.e., $\hx_{\mathrm{BLUE}}(y) = Ay$ for some $A \in \R^{n \times m}$.
    \item {\em Unbiasedness}: Its expected value $\Ex_{y\sim \GLR_{W,\Sigma}(x^\star)}[\hx_{\mathrm{BLUE}}(y)]$ equals the true parameter $x^\star$.
    \item {\em Minimum Variance}: For any linear unbiased estimator $\hx$ with covariance $\hSigma$, the difference of covariance matrices $\hSigma - \hSigma_{\mathrm{BLUE}}$
    is positive semi-definite.
\end{enumerate}
\end{definition}

\noindent Given an estimator $\hx$ for $x^\star$, we can derive an estimator for any linear projection of $x^\star$. For linear map $L : \R^n \to \R^q$, $L\hx$ is an estimator for $Lx^\star$ with covariance $L\hSigma L^\top$. If $\hx_{\mathrm{BLUE}}$ is the BLUE for $\GLR_{W, \Sigma}$, then for any linear unbiased estimator $\hx$ with covariance $\hSigma$, it follows that $L \hSigma L^\top - L \hSigma_{\mathrm{BLUE}} L^\top$ is also positive semi-definite. We will therefore refer to $\hz : \R^m \to \R^q$ given by $\hz(y) = L\hx_{\mathrm{BLUE}}(y)$ as the best linear unbiased estimator for $Lx^\star$ under $\GLR_{W,\Sigma}$.

The {\em Gauss--Markov} theorem precisely characterizes the BLUE for $\GLR_{W, \Sigma}$ as the minimizer of the generalized least squares ($\GLS$) objective.

\begin{lemma}[Gauss--Markov Theorem]\label{lem:gls}
The BLUE for $x^\star$ under $\GLR_{W, \Sigma}$, and its covariance satisfy:
\begin{align}
\hx_{\GLS} & ~=~ \argmin_{x} (y - Wx)^\top \Sigma^{-1} (y - Wx) ~=~ (W^\top \Sigma^{-1} W)^{-1} W^\top \Sigma^{-1} y, \label{eq:beta-gls} \\
\hSigma_{\GLS} &~=~ (W^\top \Sigma^{-1} W)^{-1}. \label{eq:sigma-gls}
\end{align}
\end{lemma}

\noindent In the special case of $\Sigma = \sigma^2 \mathbf{I}$, $\hx_{\GLS}$ is known as the \emph{Ordinary Least Squares (OLS)} estimator.

\paragraph{\boldmath Equality constrained generalized linear regression ($\ECGLR_{W,\Sigma}^{R,r}$).}
The {\em equality constrained generalized least squares regression} problem, denoted $\ECGLR_{W,\Sigma}^{\EQ,\eq}$ and parameterized by a workload matrix $W \in \R^{m \times n}$, a noise covariance $\Sigma \in \R^{m \times m}$, and equality constraints encoded by $\EQ \in \R^{\numcons \times n}$ and $\eq \in \R^{\numcons}$, is an extension of $\GLR_{W,\Sigma}$ where we are provided additional information that $\EQ x^\star = \eq$, which allows us to have an estimator $\hx$ with reduced variance. We assume that $\EQ$ has full row rank (this is without loss of generality by removing redundant equalities) and, as with $\GLR_{W,\Sigma}$, that $W$ has full column rank and $\Sigma$ is positive definite.

The notion of a BLUE for $x^\star$ under $\ECGLR_{W,\Sigma}^{\EQ,\eq}$ is modified as follows:
\begin{itemize}
\item we allow the estimator to be {\em affine}, namely $\hx = Ay + b$ for some $A \in \R^{n \times m}$ and $b \in \R^n$, and
\item the condition $\Ex[\hx_{\mathrm{BLUE}}] = x^\star$ is required to hold only for $x^\star$ satisfying $\EQ x^\star = \eq$.
\end{itemize}
The BLUE estimator for $x^\star$ under $\ECGLR_{W, \Sigma}^{\EQ, \eq}$ is then the one with minimum variance among all such estimators. Even though we are considering {\em affine} estimators, we still refer to it as the best linear unbiased estimator (BLUE). As in the case of $\GLR_{W,\Sigma}$, we can extend the notion of BLUE under $\ECGLR_{W,\Sigma}^{\EQ,\eq}$ to $Lx^\star$ for any linear map $L : \R^n \to \R^q$.

\begin{lemma}\label{lem:ecgls}
The BLUE for $x^\star$ under $\ECGLR_{W,\Sigma}^{\EQ,\eq}$ and its covariance (as returned by Algorithm~\ref{alg:ecgls}) satisfy:
\begin{align}
    \hx_{\ECGLS} &~=~ \hx_{\GLS} - \hSigma_{\GLS} \EQ^\top (\EQ \hSigma_{\GLS} \EQ^\top)^{-1} (\EQ \hx_{\GLS} - \eq) &&\hspace{-18mm}~=~ \hx_{\GLS} - P (\EQ\hx_{\GLS} - \eq),
    \label{eq:beta-eclgs}\\
    \hSigma_{\ECGLS} &~=~ \hSigma_{\GLS} - \hSigma_{\GLS} \EQ^\top (\EQ \hSigma_{\GLS} \EQ^\top)^{-1} \EQ \hSigma_{\GLS} &&\hspace{-18mm}~=~ \hSigma_{\GLS} - P \EQ \hSigma_{\GLS},
    \label{eq:simga-eclgs}
\end{align}
where $\hx_{\GLS}$ and $\hSigma_{\GLS}$ are as defined in (\ref{eq:beta-gls}, \ref{eq:sigma-gls}) and $P := \hSigma_{\GLS} \EQ^\top (\EQ \hSigma_{\GLS} \EQ^\top)^{-1}$.
In particular, $\hx_{\ECGLS}$ is the minimizer of $(y - Wx)^\top \Sigma^{-1} (y - Wx)$ subject to $\EQ x = \eq$.
\end{lemma}

Although both equality constraints and general covariance matrices are well studied in the literature, we could not find a reference for this exact theorem.   
Nevertheless, it is straightforward to derive from the BLUE for equality constrained {\em ordinary} linear regression ($\Sigma = \sigma^2 I$), which is more standard.

\begin{lemma}[BLUE for $\ECGLR_{W,\Sigma}^{\EQ,\eq}$ with $\Sigma=I$ (e.g.,\ \cite{amemiya1985advanced})]
\label{lem:ecols}
The BLUE for $x^\star$ under $\ECGLR_{W,I}^{\EQ,\eq}$ and its corresponding covariance satisfy:
\begin{align}
\hx_{\ECOLS}
&~=~ \hx_{\OLS} - \hSigma_{\OLS} \EQ^
\top (\EQ \hSigma_{\OLS} \EQ^\top)^{-1} (\EQ \hx_{\OLS} - \eq)
&&\hspace{-18mm}~=~ \hx_{\OLS} - P (\EQ\hx_{\OLS} - \eq),
\label{eq:beta-ecols}\\
\hSigma_{\ECOLS}
&~=~ \hSigma_{\OLS} - \hSigma_{\OLS} \EQ^\top (\EQ \hSigma_{\OLS} \EQ^\top)^{-1} \EQ \hSigma_{\OLS}
&&\hspace{-18mm}~=~ \hSigma_{\OLS} - P \EQ \hSigma_{\OLS}, \label{eq:sigma-ecols}
\end{align}
where $\hx_{\OLS}$ and $\hSigma_{\OLS}$ are as defined in (\ref{eq:beta-gls}, \ref{eq:sigma-gls}) for the case $\Sigma = I$,
and $P := \hSigma_{\OLS} \EQ^\top (\EQ \hSigma_{\OLS} \EQ^\top)^{-1}$. In particular, $\hx_{\ECOLS}$ is the minimizer of $\|y - Wx\|_2^2$ subject to $\EQ x = \eq$.
\end{lemma}
\begin{proof}
The expression for $\hx_{\ECOLS}$ is standard (e.g.\ see \cite{amemiya1985advanced}). To derive the expression for $\hSigma_{\ECOLS}$, we observe that the variance of $\hx_{\ECOLS}$ is:
\begin{align*}
\hSigma_{\ECOLS} &~=~ (I - P \EQ) \hSigma_{\OLS} (I - \EQ^\top P^\top)
~=~ \hSigma_{\OLS} - P \EQ \hSigma_{\OLS} - \hSigma_{\OLS} \EQ^\top P^\top + P \EQ \hSigma_{\OLS} \EQ^\top P^\top
\end{align*}
Now, observe that the third term $\hSigma_{\OLS} \EQ^\top P^\top = \hSigma_{\OLS} \EQ^\top (\EQ \hSigma_{\OLS} \EQ^\top)^{-1} \EQ \hSigma_{\OLS} = P \EQ \hSigma_{\OLS}$, and the last term is $P \EQ \hSigma_{\OLS} \EQ^\top P^\top = P \EQ \hSigma_{\OLS} \EQ^\top (\EQ \hSigma_{\OLS} \EQ^\top)^{-1} \EQ \hSigma_{\OLS} = P \EQ \hSigma_{\OLS}$.

Thus, we obtain $\hSigma_{\ECOLS} = \hSigma_{\OLS} - P \EQ \hSigma_{\OLS}$ as desired.
\end{proof}

\begin{proof}[Proof of Lemma~\ref{lem:ecgls}]
In $\ECGLR_{W,\Sigma}^{\EQ,\eq}$, our goal is to estimate $x^\star$ given observations $y \sim Wx^\star + \eta$ for $\eta$ which has zero mean and covariance $\Sigma$ and for which $\EQ x^\star = \eq$ holds.

Since $\Sigma$ is full rank, we can reparameterize the observations as $y' = \Sigma^{-1/2} y$. This substitution transforms the problem into an equality constrained ordinary linear regression problem in which we observe $y' \sim \Sigma^{-1/2} W x^\star + \eta'$, where $\eta'$ has zero mean and variance $I$.

The lemma is immediate by substituting $W \gets \Sigma^{-1/2} W$ and $y \gets \Sigma^{-1/2} y$ in Lemma~\ref{lem:ecols}, paired with the observation that there is no loss of information when multiplying observations by $\Sigma^{-1/2}$. Namely any affine estimator $\hx \gets Ay + b$ can be written as $\hx \gets A\Sigma^{1/2} y' + b$, where $y' = \Sigma^{-1/2} y$, i.e., $\hx$ is also an estimator for $\ECGLR_{W',I}^{\EQ,\eq}$ with observations $y'$. From Lemma~\ref{lem:ecols}, we have that $\hx_{\ECGLS}$, which we have defined as $\hx_{\ECOLS}$ for the problem $\ECGLR_{W',I}^{\EQ,\eq}$, has minimum covariance, and hence $\hSigma - \hSigma_{\ECGLS}$ is positive semi-definite.
\end{proof}

\paragraph{Derivation via Lagrange multipliers.}
An alternate way to derive the expression for $\hx_{\ECGLS}$ is using the method of Lagrange multipliers---this view will be helpful in our proofs. Namely, we want to minimize $(y - Wx)^\top \Sigma^{-1} (y - Wx)$ subject to $\EQ x = \eq$.  (Note that here we are starting with the constrained minimization problem and deriving the expression for the estimator, whereas Lemma~\ref{lem:ecgls} also shows that this minimizer is the BLUE, which follows from the Gauss--Markov Theorem. So this derivation alone does not suffice to prove Lemma~\ref{lem:ecgls}.)

Abstractly, consider the constrained minimization problem: $\text{min}_{x \in \R^n} f(x)$ subject to $g(x) = 0$.
Suppose that the objective $f: \R^n \to \R$ is {\em convex} and differentiable and the constraint function $g: \R^n \to \R^p$ is {\em affine}, i.e., $g(x) = \EQ x - \eq = 0$ for some matrix $\EQ$ and vector $\eq$.
The Lagrange function $\cL: \R^n \times \R^p \to \R$ is obtained by introducing the vector of Lagrange multipliers $\lambda \in \R^p$ as
$\cL(x, \lambda) := f(x) + \lambda^\top g(x)$. The following lemma provides a necessary and sufficient condition under which $x$ is the optimal solution for the above problem.

\begin{lemma}[e.g., Section 5.5.3 in \cite{boyd2014convex}]\label{lem:langrangian}
For convex $f : \R^n \to \R$ and affine $g : \R^n \to \R^p$, $x$ is the minimizer of $f(x)$ subject to $g(x) = 0$ if and only if
$$\nabla_x \cL(x, \lambda) = \nabla f(x) + \nabla g(x)^\top  \lambda = 0 \qquad \text{and} \qquad g(x) = 0.$$
\end{lemma}

\noindent For $\ECGLR_{W,\Sigma}^{\EQ,\eq}$, we have from Lemma~\ref{lem:langrangian} that the optimal solution for $f(x) = \frac12 (y - Wx)^\top \Sigma^{-1} (y - Wx)$ and $g(x) = \EQ x - \eq$ is precisely the $\hx$ that satisfies:
\begin{align*}
    W^\top \Sigma^{-1} W x - W^\top \Sigma^{-1} y + R^\top \lambda ~=~ 0
    \qquad \text{and} \qquad
    \EQ x ~=~ \eq,
\end{align*}
which can alternatively be written as,
\begin{align*}
    \begin{bmatrix} W^\top \Sigma^{-1} W & \EQ^\top \\ \EQ & 0\end{bmatrix}
    \begin{bmatrix} x \\ \lambda \end{bmatrix}
    &~=~ \begin{bmatrix} W^\top \Sigma^{-1} y  \\ \eq \end{bmatrix}\\
    \Longrightarrow \qquad
    \begin{bmatrix} x \\ \lambda \end{bmatrix}
    &~=~ 
    \begin{bmatrix} W^\top \Sigma^{-1} W & \EQ^\top \\ \EQ & 0\end{bmatrix}^{-1}
    \begin{bmatrix} W^\top \Sigma^{-1} y  \\ \eq \end{bmatrix}.
\end{align*}
The expression for $\hx_{\ECGLS}$ immediately follows from the following proposition, which can be proved by a simple application of the Schur complement method~(ref. \cite{HornJohnson}).

\begin{proposition}\label{prop:ecgls-langrangian-view}
    For $W$ that has full column rank, $\Sigma$ that is positive definite, it holds that:
    \begin{align*}
        \begin{bmatrix} W^\top \Sigma^{-1} W & \EQ^\top \\ \EQ & 0 \end{bmatrix}^{-1}
        &~=~ \begin{bmatrix} \hSigma_{\GLS} + \hSigma_{\GLS} \EQ^\top S^{-1} \EQ \hSigma_{\GLS} & - \hSigma_{\GLS} \EQ^\top S^{-1} \\ - S^{-1} \EQ \hSigma_{\GLS} & S^{-1} \end{bmatrix},
    \end{align*}
    where $\hSigma_{\GLS} = (W^\top \Sigma^{-1} W)^{-1}$ and $S = - \EQ \hSigma_{\GLS} \EQ^\top$, where $S$ is invertible because we assume that $R$ has full row rank. Note that the top left entry is precisely $\hSigma_{\ECGLS}$ (as defined in Lemma~\ref{lem:ecgls}).
\end{proposition}

\section{Formal problem specification for abstract problem}
\label{sec:formal-problem}
In this and the next section, we build the formal notation we will use for expressing the optimization problem.
At the highest level of abstraction, we have a constrained linear regression problem.  Recall that the ground truth Census data is represented as a non-negative, integer-valued vector $x^\star \in \Z_{\ge 0}^N$. We observe $y \sim W x^\star + \eta$ where,
\begin{itemize}
\item $W \in \R^{\um \times \un}$ is the {\em workload matrix},
\item $y \in \R^\um$ is the {\em observation vector},
\item $\eta$ is the {\em noise vector} with zero mean and covariance $\Sigma \in \R^{\um \times \um}$.
\end{itemize}

\noindent Additionally, some prior information regarding $x^\star$ is available, in the form of $D_{\mathrm{eq}}$ equality and $D_{\mathrm{ie}}$ inequality constraints.

\begin{itemize}
\item $\EQ x^\star = \eq$ encodes the equality constraints, where $\EQ \in \R^{D_{\mathrm{eq}} \times N}$ and $\eq \in \R^{D_{\mathrm{eq}}}$, and
\item $\IEQ x^\star \le \ieq$ encodes the inequality constraints, where $\IEQ \in \R^{D_{\mathrm{ie}} \times N}$ and $\ieq \in \R^{D_{\mathrm{ie}}}$.
\end{itemize}

Motivated by the notion of BLUE for equality constrained generalized linear regression, we consider the following mixed integer linear program optimizing over $x \in \Z_{\ge 0}^N$ as our desired objective,
\begin{align*}
\argmin_{x \in \Z_{\ge 0}^N} \|W x - y\|_{\Sigma^{-1}}^2 \qquad
\text{subject to }\ \ 
\begin{matrix}
    \EQ x = \eq & \text{(equality constraints),}\\
    \IEQ x \le \ieq & \text{(inequality constraints),}
\end{matrix}
\end{align*}
where $\|v\|_{\Sigma^{-1}}^2 := v^\top \Sigma^{-1} v$ denotes the squared Mahalanobis norm for any $v \in \R^\um$.\\

\noindent While integer programming is NP-hard in general~\cite{karp72reducibility}, we approach it in two phases as follows:
\begin{itemize}
    \item We first describe an efficient procedure that solves $\argmin_{x \in \R^\un} \|W x - y\|_{\Sigma^{-1}}^2$ subject to $\EQ x = \eq$, thereby obtaining the {\em best linear unbiased estimator (BLUE)} for the $\ECGLR_{W,\Sigma}^{\EQ,\eq}$ problem~(Lemma~\ref{lem:ecgls}), optimizing over $x \in \R^{\un}$ instead of $x\in \Z_{\ge 0}^{\un}$ and ignoring the inequality constraints $\IEQ x \le \ieq$.
    \item We modify our procedure in various places in order to incorporate the inequality constraints and to constrain the solution to be non-negative and integral. While this phase is heuristic in nature, we leverage structural properties of the optimization problem in its design.
\end{itemize}

A critical challenge is that the dimensions $\um$ and $\un$ are very large (e.g., 
$N \geq 11,879,679,168$), making it computationally infeasible to apply the expressions in Lemma~\ref{lem:ecgls} directly. Thus, it becomes necessary to leverage certain natural structural properties of our specific problem to design an efficient algorithm.

Note that our full algorithm does not directly invoke the algorithm described in the first phase that solves the $\ECGLR_{W,\Sigma}^{\EQ,\eq}$ problem. Instead, we use the $\ECGLR_{W,\Sigma}^{\EQ,\eq}$ algorithm as a motivational stepping stone to the final algorithm. Nevertheless, we believe that our $\ECGLR_{W,\Sigma}^{\EQ,\eq}$ algorithm is of independent interest for solving equality constrained linear regression problems with ``block-hierarchical'' structure. We next describe these structural properties.

\section{Block hierarchical optimization problem}
\label{sec:abstract-tree-structure}

We consider workload matrices $W$, covariance matrices $\Sigma$ and constraint matrices $\EQ$ and $\IEQ$ that collectively respect what we refer to as {\em block hierarchical structure}. This structure is represented in terms of a tree $\cT$ with root $\rt$. Let $\cL$ denote the set of leaves in $\cT$. Let $\cB$ denote a set of \emph{buckets} corresponding to combinations of other feature values. Each datapoint (i.e.\ person) in the dataset takes the form $(v, b)$ for $v \in \cL$ and $b \in \cB$. The input $x^\star \in \Z_{\ge 0}^{\un}$ is a vector indexed by $(u, b) \in \cT \times \cB$, where $x^\star_{(u, b)}$ equals the total number of datapoints $(v, b)$ where $v$ is any descendant of $u$. For ease of notation, for any node $u \in \cT$, let $x_u \in \R^{|\cB|}$ denote the vector $(x_{(u,b)})_{b \in \cB}$.

\subsection{Structure of $W$.}
The workload matrix is a block diagonal matrix specified by the per-node workload matrices $(W_u)_{u \in \cT}$, where $W_u \in \R^{m_u \times |\cB|}$. Namely,
\begin{align*}
    W ~=~ \begin{bmatrix}
        W_{u_1} & 0 & \cdots & 0 \\
        0 & W_{u_2} & \cdots & 0 \\
        \vdots & \vdots & \ddots & \vdots \\
        0 & \cdots & 0 & W_{u_{|\cT|}}
    \end{bmatrix} \in \R^{\um \times |\cT \times \cB|}, \qquad \text{where } \um = \sum_{u \in \cT} m_u.
\end{align*}

\subsection{Structure of $\Sigma$.}
The covariance matrix is defined in terms of per-node covariances $(\Sigma_u)_{u \in \cT}$ where each $\Sigma_u \in \R^{m_u \times m_u}$, and the covariance matrix $\Sigma$ is again a block-diagonal matrix
\begin{align*}
    \Sigma ~=~ \begin{bmatrix}
        \Sigma_{u_1} & 0 & \cdots & 0 \\
        0 & \Sigma_{u_2} & \cdots & 0 \\
        \vdots & \vdots & \ddots & \vdots \\
        0 & \cdots & 0 & \Sigma_{u_{|\cT|}}
    \end{bmatrix} \in \R^{\um \times \um}.
\end{align*}
Together, we can view the measurement vector $y \sim \GLR_{W,\Sigma}(x^\star)$ as a collection of independent measurements $y_u \sim \GLR_{W_u,\Sigma_u}(x^\star_u)$, for each $u \in \cT$.

\subsection{Structure of $(\EQ, \eq)$.}
There are three families of equality constraints:
\begin{itemize}
    \item {\em Consistency Constraints:} These enforce that $x^\star_{(u, b)} = \sum_{v \in \child(u)} x^{\star}_{(v, b)}$ for all non-leaf nodes $u \in \cT$ and buckets $b \in \cB$, where $\child(u)$ is the set of children of $u$.
    \item {\em Per-node Constraints:} Additionally, we have per-node equality constraints that enforce $\EQ_u x^\star_u = \eq_u$.
    \item {\em Statewide and US total constraints:} For each individual state and for the US as a whole, these enforce that the total number of individuals equals the released enumerated population for that state or for the US total.
\end{itemize}

\noindent Whenever the workload $W$, covariance $\Sigma$ and equality constraints $(\EQ, \eq)$ satisfy the {\em block hierarchical structure} as described above with respect to a tree $\cT$, we refer to the corresponding $\ECGLR_{W,\Sigma}^{\EQ, \eq}$ problem as $\BHECGLR{\cT}_{W, \Sigma}^{\EQ, \eq}$ to emphasize this structure. As alluded to earlier, we provide an {\em efficient algorithm} to obtain the BLUE for $x^\star$ under $\BHECGLR{\cT}_{W,\Sigma}^{\EQ,\eq}$. We then provide heuristics to modify this procedure to handle the additional constraints that our estimate must be integral and non-negative, and must also satisfy certain inequality constraints as described below.

\subsection{Structure of $(\IEQ, \ieq)$.}
The inequality constraints also appear in the form of ``per-node inequality constraints'', namely, as a set of constraints of the form $\IEQ_u x^\star_u \le \ieq_u$.

\subsection{Additional symmetries.}
There is additional structure in the per-node workload matrices $W_u$ and the corresponding covariances $\Sigma_u$, which allows us to operate over a certain ``compressed'' representation of $\Sigma_u$. We defer the discussion of this aspect to the section on ``Exploiting symmetries in per-node workload matrices'' below.

\begin{remark}\label{rem:redundancy}
We note that our formalism has certain redundancies. An equivalent formalization to ours would be to define variables $x^\star_{(v, b)}$ only for {\em leaf} nodes $v$, and replace $x^\star_{(u, b)}$ with $\sum_{v \in \text{descendants}(u)} x^\star_{v, b}$, additionally eliminating the consistency constraints. However, for purposes of our algorithm, it will be more convenient to work with the representation described above.
\end{remark}

\section{Some Structured $\ECGLR$ Problems}
\label{sec:equality-constraints}

In this section we discuss two types of structured $\ECGLR$ problems which we will use en route to the algorithm for computing the BLUE for the block hierarchical $\BHECGLR{\cT}_{W,\Sigma}^{\EQ,\eq}$ problem.

\subsection{Factorized-$\ECGLR$}

We first develop the concept of a Factorized-$\ECGLR$, and establish a crucial lemma. Before defining this concept in full generality, we discuss a simple motivating example. Consider a ground truth vector $x^\star := [s^\star, z_1^\star, \ldots, z_t^\star]^\top \in \R^{t+1}$, subject to the constraint that $s^\star = \sum_i z_i^\star$ (a simple tree of depth 2), where we observe two independent measurements, $y_1 = W_1 s^\star + \eta_1$ and $y_2 = W_2 z^\star + \eta_2$ (where $z^\star := [z_1^\star, \ldots, z_t^\star]^\top$). Using $y_1$, we have a BLUE $\hs$ for $s^\star$, and using $y_2$, we have a BLUE $\hz$ for $z^\star$, which in turn provides an estimator for $s^\star$ via $\sum_i \hz_i$. We would like to combine these two independent estimators into a single estimator for $s^\star$ given all measurements together. Furthermore, we would like to also handle a more general case, where there may be additional equality constraints imposed on $s^\star$ or $z^\star$.

In more generality, consider a $\ECGLR_{W, \Sigma}^{\EQ,\eq}$ problem, with a special factorized structure. In this setting, we assume that the ground truth vector $x^\star$, the workload $W$, noise covariance $\Sigma$ and equality constraints $(\EQ,\eq)$ have the following form:
\begin{align}
x^\star := \begin{bmatrix} x^\star_1 \\ x^\star_2 \end{bmatrix},\ 
W = \begin{bmatrix}W_1&0\\0&W_2\end{bmatrix},\ 
\Sigma = \begin{bmatrix}\Sigma_1&0\\0&\Sigma_2\end{bmatrix},\ 
\EQ = \begin{bmatrix}\EQ_1&0\\0&\EQ_2\\L_1&-L_2\end{bmatrix},\ 
\eq= \begin{bmatrix}\eq_1\\\eq_2\\0\end{bmatrix}\label{eq:fecglr-structure}
\end{align}
Note that $x^\star_1$ and $x^\star_2$ can have different dimensions. When $W$, $\Sigma$, $R$ and $r$ have this structure, we will refer to the $\ECGLR_{W,\Sigma}^{\EQ,\eq}$ problem as $\FECGLR_{W,\Sigma}^{\EQ,\eq}$ to emphasize this structure.

Informally speaking, the observation vector $y \sim Wx^\star + \eta$ can be factorized as $y = \begin{bmatrix}y_1 \\ y_2\end{bmatrix}$, where $y_1$ and $y_2$ are essentially observations from the two independent $\ECGLR_{W_1, \Sigma_1}^{\EQ_1, \eq_1}$ and $\ECGLR_{W_2, \Sigma_2}^{\EQ_2, \eq_2}$ problems, subject to the consistency constraint $L_1 x_1 = L_2 x_2$ on $x_1$ and $x_2$. We visualize this structure in Figure~\ref{fig:combine-problem}.

Unlike for standard $\ECGLR$, we will focus on directly obtaining an estimator $\hz$ that is the BLUE for $L_1 x^\star_1$ (equivalently $L_2 x^\star_2)$ under $\FECGLR_{W, \Sigma}^{\EQ, \eq}$, rather than explicitly computing an estimator for $(x^\star_1, x^\star_2)$.

Given the BLUE $\hx_1$ for $x^\star_1$ under $\ECGLR_{W_1, \Sigma_1}^{\EQ_1, \eq_1}$ with covariance $\hSigma_1$, we know that $\hz_1 := L_1 \hx_1$ is the BLUE for $L_1 x^\star_1$ under $\ECGLR_{W_1, \Sigma_1}^{\EQ_1, \eq_1}$, with covariance $\hOmega_1 := L_1 \hSigma_1 L_1^\top$.
We can define $\hx_2$, $\hOmega_2$ analogously.

The following lemma obtains the best linear unbiased estimator for $L_1 x^\star_1$ (which equals $L_2 x^\star_2$) for the entire $\FECGLR_{W,\Sigma}^{\EQ,\eq}$ problem, by combining the individual best linear unbiased estimators $L_1 \hx_1$ and $L_2 \hx_2$. For any $\Omega \in \R^{d \times d}$, let $\mathsf{range}(\Omega)$ denote the set $\{ y \in \R^d : y = \Omega x \text{ for some } x \in \R^d\}$.

\begin{lemma}\label{lem:combine-equality-optimal}
Consider $\FECGLR_{W,\Sigma}^{\EQ,\eq}$ with structure as defined in \eqref{eq:fecglr-structure}.
Let $\hz_1$ be the BLUE for $L_1 x^\star_1$ under $\ECGLR_{W_1,\Sigma_1}^{\EQ_1,\eq_1}$ with covariance $\hOmega_1$, and $\hz_2$, $\hOmega_2$ be defined analogously, and assume that $\hz_1 - \hz_2 \in \mathsf{range}(\hOmega_1 + \hOmega_2)$.
Then $\hz$ as defined below (and also as returned by Algorithm~\ref{alg:combine-estimates}) is the BLUE for $L_1 x^\star_1$ (equivalently, $L_2 x^\star_2$) under $\FECGLR_{W, \Sigma}^{\EQ, \eq}$, where $A := \hOmega_2 (\hOmega_1 + \hOmega_2)^\pinv$ and $B = I - A$:
\begin{align*}
\hz &~:=~ A \hz_1 + B \hz_2,\\
\text{and the covariance of $\hz$ is } \qquad \hat{\Omega} &~:=~ A \hOmega_1 A^\top + B \hOmega_2 B^\top = B \hOmega_2.
\end{align*}
\end{lemma}

\begin{proof}
The BLUE $\hx'$ for $x^\star$ under $\FECGLR_{W,\Sigma}^{\EQ,\eq}$ can be obtained via the method of Lagrange multipliers (see Lemma~\ref{lem:langrangian}). Namely, we define $\cL(x, \lambda)$ as
\begin{align*}
\cL(x, \lambda) &~:=~\textstyle \frac12 (y - Wx)^\top \Sigma^{-1} (y - Wx) + \lambda^\top (\EQ x - \eq)\\
&~=~ \textstyle
\frac12 (y_1 - W_1 x_1)^\top \Sigma_1^{-1} (y_1 - W_1 x_1)
+ \frac12 (y_2 - W_2 x_2)^\top \Sigma_2^{-1} (y_2 - W_2 x_2)\\
&~\quad ~+ \lambda_1^\top (\EQ_1 x_1 - \eq_1) + \lambda_2^\top (\EQ_2 x_2 - \eq_2) + \mu^\top (L_1 x_1 - L_2 x_2) \qquad
\text{ where } \lambda = \begin{bmatrix}\lambda_1 \\ \lambda_2 \\ \mu\end{bmatrix}
\end{align*}
From Lemma~\ref{lem:langrangian}, we have that the minimizer $\hx'$ can be obtained by setting $\nabla_x \cL(x, \lambda) = 0$, namely,
\begin{align*}
W_1^\top \Sigma_1^{-1} (W_1 \hx_1' - y_1) + \EQ_1^\top \lambda_1 + L_1^\top \mu &= 0,\\
W_2^\top \Sigma_2^{-1} (W_2 \hx_2' - y_2) + \EQ_2^\top \lambda_2 - L_2^\top \mu &= 0,
\end{align*}
along with the constraints $\EQ_1 \hx_1' = \eq_1$, $\EQ_2 \hx_2' = \eq_2$ and $L_1 \hx_1' = L_2 \hx_2'$.
On the other hand, the optimal solutions $\hx_1, \hx_2$ for the two independent problems $\ECGLR_{W_1,\Sigma_1}^{\EQ_1,\eq_1}$ and $\ECGLR_{W_2,\Sigma_2}^{\EQ_2,\eq_2}$ respectively satisfy the equations 
\begin{align*}
W_1^\top \Sigma_1^{-1} (W_1 \hx_1 - y_1) + \EQ_1^\top \nu_1 &= 0\\
W_2^\top \Sigma_2^{-1} (W_2 \hx_2 - y_2) + \EQ_2^\top \nu_2 &= 0,
\end{align*}
along with the first two sets of constraints $\EQ_1 \hx_1 = \eq_1$, $\EQ_2 \hx_2 = \eq_2$. 
Note that the values of the Lagrange multipliers $\nu_1, \nu_2$ may be different from $\lambda_1, \lambda_2$.

Let $\delta_1 = \hx_1' - \hx_1$ and $\delta_2 = \hx_2' - \hx_2$ be the difference vectors between the two solutions. Subtracting corresponding equations, we obtain the following identities for $\delta_1, \delta_2$.
\begin{align*}
W_1^\top \Sigma_1^{-1} W_1 \delta_1 + \EQ_1^\top (\lambda_1 - \nu_1) + L_1^\top \mu &= 0\\
W_2^\top \Sigma_2^{-1} W_2 \delta_2 + \EQ_2^\top (\lambda_2 - \nu_2) - L_2^\top \mu &= 0.
\end{align*}
Since $\EQ_1 \hx_1 = \EQ_1 \hx_1' = \eq_1$ and $\EQ_2 \hx_2 = \EQ_2 \hx_2' = \eq_2$, we additionally have that $\EQ_1\delta_1 = 0$ and $\EQ_2 \delta_2 = 0$.
The equations for $\delta_1$ can be written as
\begin{align*}
\begin{bmatrix}
W_1^\top \Sigma_1^{-1} W_1 & \EQ_1^\top\\
\EQ_1 & 0
\end{bmatrix}
\begin{bmatrix}
\delta_1\\ \lambda_1 - \nu_1
\end{bmatrix}
&~=~ 
\begin{bmatrix}
- L_1^\top \mu\\0
\end{bmatrix}\\
\Longrightarrow\qquad
\begin{bmatrix}
\delta_1\\ \lambda_1 - \nu_1
\end{bmatrix}
&~=~ 
\begin{bmatrix}
W_1^\top \Sigma_1^{-1} W_1 & \EQ_1^\top\\
\EQ_1 & 0
\end{bmatrix}^{-1}
\begin{bmatrix}
- L_1^\top \mu\\0
\end{bmatrix}.
\end{align*}
From Proposition~\ref{prop:ecgls-langrangian-view}, we get that $\delta_1 = - \hSigma_1 L_1^\top \mu$, where we use that the top left sub-matrix of the matrix inverse is precisely $\hSigma_1$.
By an identical argument, we have that $\delta_2 = \hSigma_2 L_2^\top \mu$.
Multiplying the expressions by $L_1$ and $L_2$ and replacing $\delta_1, \delta_2$ by their definitions, we have
\begin{align*}
L_1 (\hx_1' - \hx_1) &= -L_1 \hSigma_1 L_1^\top \mu = -\hOmega_1 \mu\\
L_2 (\hx_2' - \hx_2) &= L_2 \hSigma_2 L_2^\top \mu =\hOmega_2 \mu.
\end{align*}
Let $\hz_1 = L_1 \hx_1$, $\hz_2 = L_2 \hx_2$, and $\hz = L_1 \hx_1' = L_2 \hx_2'$. Then the equations become
\begin{align*}
\hz - \hz_1 &= -\hOmega_1 \mu\\
\hz - \hz_2 &= \hOmega_2 \mu.
\end{align*}
Thus, we get $\hz_1 - \hz_2 = (\hOmega_1 + \hOmega_2) \mu$ (here we would run into a contradiction if $\hz_1 - \hz_2$ were not in $\mathsf{range}(\hOmega_1 + \hOmega_2)$).
Thus, we have that $\mu = (\hOmega_1 + \hOmega_2)^\pinv (\hz_1 - \hz_2) + w$ for some $w$ in the nullspace of $\hOmega_1 + \hOmega_2$.
Hence, we get for $A = \hOmega_2 (\hOmega_1 + \hOmega_2)^\pinv$ and $B = I - A$ that
\begin{align*}
    \hz &~=~ \hz_2 + \hOmega_2 (\hOmega_1 + \hOmega_2)^\pinv (\hz_1 - \hz_2) + \hOmega_2 w\\
    &~=~ A \hz_1 + B \hz_2
\end{align*}
where we use that $\hOmega_2 w = 0$ because $w$ is in the nullspace of $\hOmega_1 + \hOmega_2$, and hence also in the nullspace of $\hOmega_2$.
Finally, to derive the expression for $\hOmega$, we invoke the independence of $\hz_1$ and $\hz_2$ and observe that
\begin{align*}
\hOmega
&~=~ A \hOmega_1 A^\top + (I - A) \hOmega_2 (I - A)^\top \\
&~=~ \hOmega_2 - A \hOmega_2 - \hOmega_2 A^\top + A (\hOmega_1 + \hOmega_2) A^\top\\
&~=~ \hOmega_2 - \hOmega_2 (\hOmega_1 + \hOmega_2)^\pinv \hOmega_2 \\
&~=~ (I - \hOmega_2 (\hOmega_1 + \hOmega_2)^\pinv) \hOmega_2 \\
&~=~ B \hOmega_2
\end{align*}
\end{proof}

\subsection{Product-$\ECGLR$}

In this section we develop the simple concept of a Product-$\ECGLR$ problem, which is an $\ECGLR_{W, \Sigma}^{\EQ, \eq}$ problem with a special product structure described by a collection of $k$ {\em independent} $\ECGLR$ problems. Namely, we assume that the ground truth vector $x^\star$, the workload $W$, noise covariance $\Sigma$ and equality constraints $(\EQ,\eq)$ have the following form:
\begin{align}
x^\star := \begin{bmatrix} x^\star_1 \\ \vdots \\ x^\star_k\end{bmatrix},\ 
W = \begin{bmatrix}
        W_1 & \ldots & 0 \\
        \vdots & \ddots & \vdots \\
        0 & \ldots & W_k
    \end{bmatrix},\ 
\Sigma = \begin{bmatrix}
        \Sigma_1 & \ldots & 0 \\
        \vdots & \ddots & \vdots \\
        0 & \ldots & \Sigma_k
    \end{bmatrix},\ 
\EQ = \begin{bmatrix}
        \EQ_1 & \ldots & 0 \\
        \vdots & \ddots & \vdots \\
        0 & \ldots & \EQ_k
    \end{bmatrix},\ 
\eq= \begin{bmatrix}\eq_1\\\vdots\\\eq_k\end{bmatrix}\label{eq:pecglr-structure}
\end{align}

Note that vector components $x^\star_i$ can have different dimensions.
When $W$, $\Sigma$, $R$ and $r$ have this structure, we will refer to the $\ECGLR_{W,\Sigma}^{\EQ,\eq}$ problem as $\PECGLR_{W,\Sigma}^{\EQ,\eq}$ to emphasize this structure.

\begin{lemma}\label{lem:pecglr}
Consider $\PECGLR_{W,\Sigma}^{\EQ,\eq}$ with structure as defined in \eqref{eq:pecglr-structure}.
Let $\hx_i$ be the BLUE for $x^\star_i$ under $\ECGLR_{W_i,\Sigma_i}^{\EQ_i,\eq_i}$ with covariance $\hSigma_i$ for each $i$. Then, $\hx$ is the BLUE for $x^\star$ under $\PECGLR_{W,\Sigma}^{\EQ,\eq}$ with covariance $\hSigma$ as defined below:
\begin{align*}
    \hx = \begin{bmatrix} \hx_1 \\ \vdots \\ \hx_k\end{bmatrix} \qquad \text{and} \qquad
    \hSigma = \begin{bmatrix}
        \hSigma_1 & \ldots & 0 \\
        \vdots & \ddots & \vdots \\
        0 & \ldots & \hSigma_k
    \end{bmatrix}
\end{align*}
As a corollary, for any linear map $L : x \mapsto L_1 x_1 + \ldots + L_k x_k$, the BLUE $\hz$ for $Lx^\star$ under $\PECGLR_{W,\Sigma}^{\EQ,\eq}$ is given as $\sum_{i=1}^k L_i \hx_i$ and has covariance $\hSigma_L := \sum_{i=1}^k L_i \hSigma_i L_i^\top$.
\end{lemma}
\begin{proof}
This proof is immediate by observing that $\hx$ is given as:
\[
\argmin_{x} (y - Wx)^\top \Sigma^{-1} (y - Wx) \qquad \text{subject to }\ \EQ x = \eq
\]
which essentially breaks down as independent optimization problems.
\[
\argmin_{x} \sum_{i=1}^k (y_i - W_ix_i)^\top \Sigma_i^{-1} (y_i - W_ix_i) \qquad \text{subject to } \EQ_i x_i = \eq_i \text{ for each } i.
\]
\end{proof}

\section{Optimal block hierarchical post-processing algorithm}
\label{sec:optimal-linear}

In this section, we present an algorithm and show that it computes the best linear unbiased estimator $\hx$ for $x^\star$ under $\BHECGLR{\cT}_{W,\Sigma}^{R,r}$, or equivalently that the output $\hx_v$ is the BLUE for $x^\star_v$ under $\BHECGLR{\cT}_{W,\Sigma}^{R,r}$ for all $v \in \cT$. At a high level, this algorithm computes the BLUE for $x^\star_v$ under sub-problems of $\BHECGLR{\cT}_{W,\Sigma}^{R,r}$ that involve measurements and constraints on a subset of all variables in $x^\star$. These solutions are then recursively combined using Lemma~\ref{lem:combine-equality-optimal} to obtain the BLUE for $x^\star_v$ under increasingly larger sub-problems of $\BHECGLR{\cT}_{W,\Sigma}^{R,r}$, until we cover all the variables in $x^\star$. To simplify the presentation of the algorithm, consider the following subsets of $\cT$ defined for any $v \in \cT$ and visualized in Figure~\ref{fig:post-processing-subgraphs}:
\begin{itemize}
    \item $\Tup_v$ is the set of all nodes in the tree that are descendants of $v$, {\em not including} $v$.
    \item $\TUp_v$ is the set of all nodes in the tree that are descendants of $v$, {\em including} $v$.
    \item $\Tdown_v = \cT \smallsetminus \TUp_v$ is the set of all nodes that are neither $v$ itself nor a descendant of $v$.
    \item $\TDown_v = \cT \smallsetminus \Tup_v = \Tdown_v \cup \{v\}$ is the set of all nodes that are not strict descendants of $v$.
\end{itemize}

\noindent Corresponding to any subset $S \subseteq \cT$, let $\ECGLR_S$ be a sub-problem defined over the subset of all variables $x^\star_S := (x^\star_{(u, b)})_{(u, b) \in S \times \cB}$ that involve
\begin{itemize}
    \item measurements $y_u \sim \GLR_{W_u, \Sigma_u}(x^\star_u)$ for $u \in S$, subject to
    \item all per-node equality constraints $\EQ_u x_u = \eq_u$ for $u \in S$, and
    \item all consistency equality constraints $x^\star_{(u,b)} = \sum_{v \in \child(u)} x^\star_{(v, b)}$ such that $\{u\} \cup \child(u) \subseteq S$.
\end{itemize}
For further convenience, let $\ECGLR_v$ denote $\ECGLR_{\{v\}}$, 
$\ECGLRup_v$ denote $\ECGLR_{\Tup_v}$, and analogously define $\ECGLRUp_v$, $\ECGLRdown_v$ and $\ECGLRDown_v$, and finally, let $\ECGLR_{\cT}$ denote the full $\BHECGLR{\cT}_{W, \Sigma}^{\EQ, \eq}$ problem.

\begin{theorem}\label{thm:tree-optimality}
The value $\hx_v$ returned by Algorithm~\ref{alg:tree-post-processing} is the BLUE for 
$x^\star_v$ under $\BHECGLR{\cT}_{W,\Sigma}^{\EQ,\eq}$ for all $v \in \cT$, and its covariance is $\hvar_v$.

\end{theorem}
Note that we explicitly compute only the covariance of the measurements corresponding to each individual node, and not the covariances between different nodes.
\begin{proof}
This theorem follows by a repeated applications of Lemmas~\ref{lem:combine-equality-optimal} and \ref{lem:pecglr}. In particular, we prove the following claims:
\begin{enumerate}
\item[\textbf{(C1)}]  $\hz_v$ is the BLUE for $x^\star_v$ under $\ECGLR_v$, with covariance $\var_v$.
\item[\textbf{(C2)}] $\hzUp_v$ is the BLUE for $x^\star_v$ under $\ECGLRUp_v$, with covariance $\varUp_v$.
\item[\textbf{(C3)}] $\hzDown_v$ is the BLUE for $x^\star_v$ under $\ECGLRDown_v$, with covariance $\varDown_v$.
\item[\textbf{(C4)}] $\hx_v$ is the BLUE for $x^\star_v$ under $\ECGLR_{\cT}$, with covariance $\hvar_v$.
\end{enumerate}
The last step is precisely the final statement we desire. We now prove each of these claims.

\subpara{(C1) $\hz_v$ is the BLUE for $x^\star_v$ under $\ECGLR_v$ with covariance $\var_v$}
Immediate from Lemma~\ref{lem:ecgls}.

\subpara{(C2) $\hzUp_v$ is the BLUE for $x^\star_v$ under $\ECGLRUp_v$ with covariance $\varUp_v$} We prove this by induction from bottom to top.\\[-2mm]

\noindent {\em Base Case:} For a leaf $v$, (C2) is equivalent to (C1) above, since $\hzUp_v = \hz_v$ and $\ECGLRUp_v$ is same as $\ECGLR_v$.\\[-2mm]

\noindent {\em Inductive Step:} For any non-leaf node $v$, assume that (C2) holds for all $u \in \child(v)$.

In this case, $\ECGLRUp_v$ precisely has the form of a factorized-$\ECGLR$ problem (see \eqref{eq:fecglr-structure}), corresponding to the sub-problems of $\ECGLR_v$ and $\ECGLRup_v$ with additional consistency constraints of $x^\star_v = \sum_{u \in \child(v)} x^\star_u$. Furthermore, $\ECGLRup_v$ precisely has the form of a product-$\ECGLR$ (see \eqref{eq:pecglr-structure}) with components of $\ECGLRUp_u$ for all $u \in \child(v)$.

From Lemma~\ref{lem:pecglr}, we have that the BLUE for $\sum_{u \in \child(v)} x^\star_u$ under $\ECGLRup_v$ is given as $\sum_{u \in \child(v)} \zUp_u =: \hzup_v$ with covariance $\sum_{u \in \child(v)} \varUp_u =: \varup_v$
(where we use the inductive hypothesis that $\zUp_u$ is the BLUE for $x^\star_u$ under $\ECGLRUp_u$).

We have from (C1) above that $\hz_v$ is the BLUE for $x^\star_v$ under $\ECGLR_v$ with covariance $\var_v$.

Putting this together using Lemma~\ref{lem:combine-equality-optimal},
we conclude that the BLUE for $x^\star_v$ (equivalently $\sum_{u \in \child(v)} x^\star_u$) under $\ECGLRUp_v$ is precisely $\zUp_v$ given as $\CombineEstimators((\hz_v, \var_v), (\hzup_v, \varup_v))$ with covariance $\varUp_v$.

\subpara{(C3) $\hzDown_v$ is the BLUE for $x^\star_v$ under $\ECGLRDown_v$ with covariance $\varDown_v$} We again apply induction, this time from top to bottom.\\[-2mm]

\noindent {\em Base Case:} For root $\rt$, (C3) is equivalent to (C1) above, since $\hzDown_{\rt} = \hz_{\rt}$ and $\ECGLRDown_{\rt}$ is same as $\ECGLR_{\rt}$.

\noindent {\em Inductive Step:} For any non-root node $v$, assume that (C3) holds for its parent $p$. We have already established that (C2) holds for all nodes.

In this case, $\ECGLRDown_v$ precisely has the form of a factorized-$\ECGLR$ problem (see \eqref{eq:fecglr-structure}), corresponding to the sub-problems of $\ECGLR_v$ and $\ECGLRdown_v$ with additional consistency constraints of $x^\star_v = x^\star_p - \sum_{u \in \child(p) \smallsetminus \{v\}} x^\star_u$. Furthermore, $\ECGLRdown_v$ precisely has the form of a product-$\ECGLR$ (see \eqref{eq:pecglr-structure}) with components of $\ECGLRDown_p$ and $\ECGLRUp_u$ for all $u \in \child(p) \smallsetminus \{v\}$.

From Lemma~\ref{lem:pecglr}, we have that the BLUE for $x^\star_p - \sum_{u \in \child(p) \smallsetminus \{v\}} x^\star_u$ under $\ECGLRdown_v$ is given as $\zDown_p - \sum_{u \in \child(p) \smallsetminus\{v\}} \zUp_u =: \hzdown_v$ with covariance $\varDown_p + \sum_{u \in \child(p) \smallsetminus\{v\}} \varUp_u =: \vardown_v$,
where we use the inductive hypothesis that $\zDown_p$ is the BLUE for $x^\star_p$ for $\ECGLRDown_p$ and that $\zUp_u$ is the BLUE for $x^\star_u$ under $\ECGLRUp_u$ from (C2).

We have from (C1) above that $\hz_v$ is the BLUE for $x^\star_v$ under $\ECGLR_v$ with covariance $\var_v$ with $\nullspace_v = \nullspace(\var_v)$.

Putting this together using Lemma~\ref{lem:combine-equality-optimal},
we conclude that the BLUE $\zDown_v$ for $x^\star_v$ (equivalently $x^\star_p - \sum_{u \in \child(p) \smallsetminus \{v\}} x^\star_u$) under $\ECGLRDown_v$ is precisely $\zDown_v$ given as $\CombineEstimators((\hz_v, \var_v), (\hzdown_v, \vardown_v))$ with covariance $\varDown_v$.

\subpara{(C4) $\hx_v$ is the BLUE for $x^\star_v$ under $\ECGLR_{\cT}$ with covariance $\hvar_v$} We build on the previous claims.\\[-2mm]

\noindent The case of root node $\rt$ follows from (C2), since $\hx_{\rt} = \zUp_\rt$, $\hvar_\rt = \varUp_\rt$ and $\ECGLR_\cT$ is the same as $\ECGLRUp_\rt$.\\[-2mm]

\noindent The case of leaf nodes $v$ follows from (C3), since $\hx_v = \zDown_v$, $\hvar_v = \varDown_v$ and $\ECGLR_\cT$ is same as $\ECGLRDown_v$.\\[-2mm]

\noindent Finally, for any intermediate node $v$, we observe that $\ECGLR_\cT$ precisely has the form of a factorized-$\ECGLR$ problem (see \eqref{eq:fecglr-structure}), corresponding to the sub-problems of $\ECGLRUp_v$ and $\ECGLRdown_v$ with additional consistency constraints of $x^\star_v = x^\star_p - \sum_{u \in \child(p) \smallsetminus \{v\}} x^\star_u$.

As already argued in the proof of (C3), the BLUE for $x^\star_p - \sum_{u \in \child(p) \smallsetminus \{v\}} x^\star_u$ under $\ECGLRdown_v$ is given as $\zDown_p - \sum_{u \in \child(p) \smallsetminus\{v\}} \zUp_u =: \hzdown_v$ with covariance $\varDown_p + \sum_{u \in \child(p) \smallsetminus\{v\}} \varUp_u =: \vardown_v$.

From (C2), the BLUE for $x^\star_v$ under $\ECGLRUp_v$ is given as $\zUp_v$ with covariance $\varUp_v$.

Putting this together using Lemma~\ref{lem:combine-equality-optimal},
we conclude that the BLUE $\hx_v$ for $x^\star_v$ (equivalently $x^\star_p - \sum_{u \in \child(p) \smallsetminus \{v\}} x^\star_u$) under $\ECGLR_{\cT}$ is precisely $\hx_v$ given as $\CombineEstimators((\hzUp_v, \varUp_v), (\hzdown_v, \vardown_v))$ with covariance $\hvar_v$.
\end{proof}

\begin{remark}\label{rem:covariance-not-block-diagonal}
The covariance $\hvar$ of the BLUE $\hx$ of $\BHECGLR{\cT}_{W,\Sigma}^{\EQ,\eq}$ can be dense, whereas as mentioned above, $(\hvar_v)_{v \in \cT}$ returned by Algorithm~\ref{alg:tree-post-processing} is only the ``block-diagonal'' portion of $\hvar$. While we do not compute the full $\hvar$, computing only the block diagonal portion is crucial for efficiency and suffices for our final \BlueDown{} algorithm, as we describe below.
\end{remark}

\subsection{Computational complexity of Algorithm~\ref{alg:tree-post-processing}}
The number of computational operations performed in Algorithm~\ref{alg:tree-post-processing} scales as $O(|\cT| \cdot \poly(|\cB|))$. (For the sake of simplicity of notation, we assume that the number of measurements ``per node'' is $O(|\cB|)$.) The $\poly$ accounts for operations such as matrix multiplications and matrix inversion. The main thing to note is that the number of computational operations scales {\em linearly} in the size of the tree. By contrast, a naive implementation of the expressions in Lemma~\ref{lem:ecgls} would have required $O(\poly(|\cT| \cdot |\cB|))$ operations, which is intractable for large trees $|\cT|$.

Furthermore, Algorithm~\ref{alg:tree-post-processing} is {\em local} and {\em massively parallelizable}, since the computations are essentially performed on a per-node basis, each requiring only $O(k \cdot \poly(|\cB|))$ memory where $k$ is the maximum degree of a node.

While the linear dependence in $|\cT|$ is highly efficient, the $\poly(|\cB|)$ dependence can still be prohibitive. For instance, in the US Census setting,
each per-node workload matrix has size $2603 \times 2016$. Maintaining and performing operations on covariance matrices of dimension $2016 \times 2016$, while feasible in terms of memory limits on typical machines, is 
still a substantial overhead since we must do so for every node in $\cT$. However, as we have mentioned,
there are substantial symmetries in the structure of the workload, noise covariance matrices and constraints. Below, we show how to leverage this to significantly reduce the number of computational operations performed. As an illustration, instead of maintaining dense per-node covariances of dimensions $2016 \times 2016$, we show that it suffices to maintain a pair of matrices of dimension $32 \times 32$ per node, resulting in a nearly $2000$-fold reduction in the representation size and a significantly larger reduction in the number of operations performed.

\section{Exploiting symmetries in per-node workload matrices}
\label{sec:symmetries}

Going back to the $\BHECGLR{\cT}_{W,\Sigma}^{\EQ,\eq}$ problem, 
recall that $W_{u} \in \R^{m_u \times |\cB|}$ and $\Sigma_u \in \R^{m_u \times m_u}$. We observe that the linear queries performed in Tables~\ref{tab:per-node-workload-single-feature} and~\ref{tab:per-node-workload-multi-feature} satisfy certain symmetries that make it possible to perform the operations in Algorithm~\ref{alg:tree-post-processing} much more efficiently.

We write the set of buckets $\cB$ as a product set $\Ba \times \Bs$, where each bucket $b \in \cB$ is written as a pair $(b_a, b_s)$, and we will refer to $b_a$ as the {\em asymmetric feature} and $b_s$ as the {\em symmetric feature}.
In the setting of Tables~\ref{tab:per-node-workload-single-feature} and~\ref{tab:per-node-workload-multi-feature}, $\Ba = [8] \times [2] \times [2]$ is the combination of Household or Group Quarters type, Hispanic/Latino origin and Voting age, and $\Bs = [63]$ is the Census Race.
Now, we assume that each type of query in $W_u$ has one of the following forms:
\begin{itemize}
\item {\bf [individual]} a linear query $\sum_{b_a \in \Ba} w_{b_a} x_{u, (b_a, b_s)}$ for each value of the symmetric feature $b_s \in \Bs$, or
\item {\bf [aggregate]} $\sum_{b_a \in \Ba} w_{b_a} \sum_{b_s \in \Bs} x_{u, (b_a, b_s)}$.
\end{itemize}
and all queries of the same type are measured with the same noise scale. We assume a similar structure for the equality constraints.
Formally, the per-node workload matrix, covariance matrix and equality constraints have the following structure:
\begin{align}
    W = \begin{bmatrix}
        W_1 \otimes I_{\Bs}\\
        W_2 \otimes \ones_{\Bs}^\top
    \end{bmatrix},\ 
    \Sigma = \begin{bmatrix}
        \Sigma_1 \otimes I_{\Bs} & 0 \\
        0 & \Sigma_2
    \end{bmatrix},\ 
    \EQ = \begin{bmatrix}
        \EQ_1 \otimes I_{\Bs} \\
        \EQ_2 \otimes \ones_{\Bs}^\top
    \end{bmatrix},\ 
    \eq = \begin{bmatrix}
        \eq_1 \\
        \eq_2
    \end{bmatrix}
    \label{eq:symmetry-structure}
\end{align}
where $W_1 \in \R^{m_1 \times |\Ba|}$, $W_2 \in \R^{m_2 \times |\Ba|}$, $\Sigma_1 \in \R^{m_1 \times m_1}$, $\Sigma_2 \in \R^{m_2 \times m_2}$, $\EQ_1 \in \R^{e_1 \times |\Ba|}$, $\EQ_2 \in \R^{e_2 \times |\Ba|}$, $\eq_1 \in \R^{e_1 |\Bs|}$, $\eq_2 \in \R^{e_2}$. Submatrix $W_1$ corresponds to queries such as CENRACE, HISPANIC $\times$ CENRACE and DETAILED in Tables~\ref{tab:per-node-workload-single-feature} and~\ref{tab:per-node-workload-multi-feature} that slice the symmetric feature $b_s$ since tensoring by $I_{\Bs}$ corresponds to the operation of considering each possible value of the symmetric feature separately, and $W_2$ corresponds to queries such as TOTAL and VOTINGAGE that aggregate over the symmetric feature since tensoring by $\ones_{\Bs}$ corresponds to summing the estimates for all $b_s\in\Bs$. The measurements $y$ can be written as
\[
y = \begin{bmatrix}
    y_1 \\ y_2
\end{bmatrix}
\quad \text{ where } y_1 = (W_1 \otimes I_{\Bs}) x^\star + \eta_1 \text{ and } y_2 = (W_2 \otimes \ones_{\Bs}^\top) x^\star + \eta_2
\]
corresponding to the {\em individual} and {\em aggregate} measurements respectively as described above,
for independent noises $\eta_1$ and $\eta_2$ with covariances $\Sigma_1 \otimes I_{\Bs}$ and $\Sigma_2$ respectively.

These symmetries allow us to perform Algorithm~\ref{alg:tree-post-processing} without having to explicitly materialize, or perform matrix operations over, covariance matrices that have size $|\cB| \times |\cB|$, working instead with covariances represented in terms of a pair of matrices of size $|\Ba| \times |\Ba|$. In particular, in this section we describe how to perform all the operations in Algorithm~\ref{alg:tree-post-processing} in terms of these succinctly represented matrices. 
As discussed in \ref{sec:efficiency}, this improvement reduces runtime of Algorithm~\ref{alg:tree-post-processing} by a factor of ${\sim}10^4$ in the Census redistricting setting, where $|\cB| = 2016$ and $|\Ba| = 32$.

Let $d := |\Bs|$, and let $I$ denote the $d \times d$ identity matrix and $J$ denote the $d \times d$ all-ones matrix. Let $P_0 = I - \frac1d J$ and $P_1 = \frac1d J$; $P_1$ is a linear operator that projects a vector in $\R^d$ onto the all-ones vector, and $P_0$ is a linear operator also known as the Centering Matrix that projects a vector in $\R^d$ onto the orthogonal complement of the all-ones vector. A key property is that $P_0^2 = P_0$, $P_1^2 = P_1$ and $P_0 P_1 = P_1 P_0 = 0$.

\begin{definition}[Succinct Matrices]\label{def:succinct}
A matrix $M \in \R^{md \times kd}$ is said to be a {\em succinct matrix} if $M = A \otimes P_0 + B \otimes P_1$, for $A, B \in \R^{m \times k}$.
\end{definition}

We show how Algorithm~\ref{alg:tree-post-processing} may be implemented efficiently by showing the same for operations $\ECGLS_{W,\Sigma}^{\EQ,\eq}$ (Algorithm~\ref{alg:ecgls}) and $\CombineEstimators$ (Algorithm~\ref{alg:combine-estimates}), when $W$, $\Sigma$, $\EQ$ and $\eq$ satisfy the structure in Equation~\ref{eq:symmetry-structure}. In particular, we will show that all the covariances $\var_v$, $\varup_v$, $\varUp_v$, etc.\ are succinct matrices.

\subsection{Algebra of succinct matrices}
We observe that various operations that we need to perform on succinct matrices can be performed efficiently without having to ever materialize the full matrix.

If $M = A \otimes P_0 + B \otimes P_1 \in \R^{kd \times md}$ is a succinct matrix, then for any vector $y \in \R^{md}$, we can compute $My$ without materializing $M$ because we can compute each of $(A \otimes P_0) y$ and $(B\otimes P_1) y$ without materializing the Kronecker products via the ``vec trick'' (Fact~\ref{fact:kronecker}).

If $M_1, M_2 \in \R^{kd \times kd}$ are succinct matrices, then $M_1 M_2$ is a succinct matrix which can be computed efficiently; namely if $M_1 = A_1 \otimes P_0 + B_1 \otimes P_1$ and $M_2 = A_2 \otimes P_0 + B_2 \otimes P_1$, then $M_1 M_2 = (A_1 A_2) \otimes P_0 + (B_1 B_2) \otimes P_1$.

If $M \in \R^{kd \times kd}$ is a succinct matrix, then $M^{\dagger}$ is also a succinct matrix that can be computed efficiently, as shown by the following lemma.

\begin{lemma}\label{lem:uncompressed-inverse-v2}
If $M = A \otimes P_0 + B \otimes P_1 \in \R^{kd \times kd}$, then $M^\dagger = A^\dagger \otimes P_0 + B^\dagger \otimes P_1$. In particular, we have that $M$ is invertible if and only if both $A$ and $B$ are invertible.
\end{lemma}
\begin{proof}
Let $M' = A^\dagger \otimes P_0 + B^\dagger \otimes P_1$.
We verify all the properties of pseudoinverse as defined in Definition~\ref{def:pseudo-inverse}.
\begin{align*}
MM' &~=~
    (A \otimes P_0 + B \otimes P_1) (A^\dagger \otimes P_0 + B^\dagger \otimes P_1)\\
    &~=~ A A^\dagger \otimes P_0^2 + BB^\dagger \otimes P_1^2\\
    &~=~  A A^\dagger \otimes P_0 + BB^\dagger \otimes P_1
\end{align*}
so 
\begin{align*}
MM'M &= A A^\dagger A \otimes P_0 + BB^\dagger B \otimes P_1 =  A \otimes P_0 + B \otimes P_1 = M \\
\intertext{and}
M'MM' &= A^\dagger A A^\dagger \otimes P_0 + B^\dagger B B^\dagger \otimes P_1 =  A^\dagger \otimes P_0 + B^\dagger \otimes P_1 = M'.
\end{align*}
Moreover, 
\[(MM')^* = (A A^\dagger \otimes P_0 + BB^\dagger \otimes P_1)^*
= (A A^\dagger)^* \otimes P_0^* + (BB^\dagger)^* \otimes P_1^* = MM',
\]
so $MM'$ is Hermitian. By an identical calculation $M'M$ is also Hermitian, hence we get that $M^\dagger = M'$.
\end{proof}

\subsection{Efficient implementation of $\ECGLS$}
To simulate Algorithm~\ref{alg:ecgls}, we will show how to compute $\hSigma_{\GLS}$, $\hx_{\GLS}$, $\hSigma_{\ECGLS}$ and $\hx_{\ECGLS}$ efficiently, and in particular that $\hSigma_{\GLS}$ and $\hSigma_{\ECGLS}$ are succinct matrices, which will be crucial for later operations.

\subpara{Computing $\hSigma_{\GLS}$}
We have $\hSigma_{\GLS} = (W^\top \Sigma^{-1} W)^{-1}$. Under the structure in Equation~\ref{eq:symmetry-structure}, we have
\[
\Sigma^{-1} = \begin{bmatrix}
    \Sigma_1^{-1} \otimes I & 0 \\
    0 & \Sigma_2^{-1}
\end{bmatrix}
\]
Thus, we have
\begin{align*}
\hSigma_{\GLS} &~=~ (W^\top \Sigma^{-1} W)^{-1} \\
&~=~ \prn{(W_1^\top \Sigma_1^{-1} W_1) \otimes I + (W_2^\top \Sigma_2^{-1} W_2) \otimes J}^{-1}\\
&~=~ \prn{(W_1^\top \Sigma_1^{-1} W_1) \otimes (P_0 + P_1) + (W_2^\top \Sigma_2^{-1} W_2) \otimes (d P_1)}^{-1}\\
&~=~ A \otimes P_0 + B \otimes P_1
\end{align*}
where $A := (W_1^\top \Sigma_1^{-1} W_1)^{-1}$ and $B := ((W_1^\top \Sigma_1^{-1} W_1) + d (W_2^\top \Sigma_2^{-1} W_2))^{-1}$.
Thus, we have established that $\hSigma_{\GLS}$ is a succinct matrix, and its components can be computed efficiently.

\subpara{Computing $\hx_{\GLS}$}
Recall $\hx_{\GLS} = (W^\top \Sigma^{-1} W)^{-1} W^\top \Sigma^{-1} y = \hSigma_{\GLS} W^\top \Sigma^{-1} y$. We have already shown that $\hSigma_{\GLS}$ is a succinct matrix. We can efficiently compute $W^\top \Sigma^{-1} y$ under Equation~\ref{eq:symmetry-structure} as follows: 
\begin{align*}
W^\top \Sigma^{-1} y = (W_1^\top \Sigma_1^{-1} \otimes I) y_1 + (W_2^\top \Sigma_2^{-1} \otimes \ones_{d \times 1}) y_2
\end{align*}

Note that each term on the RHS can be computed efficiently without materializing the Kronecker products via the vec trick~(Fact~\ref{fact:kronecker}). Having computed $x' := W^\top \Sigma^{-1} y$, we can again efficiently compute $\hx_{\GLS} = \hSigma_{\GLS} x'$ via the vec trick.

\subpara{Computing $\hSigma_{\ECGLS}$ and $\hx_{\ECGLS}$}
Recall that $\hSigma_{\ECGLS} = (I_{kd\times kd} - L \EQ) \hSigma_{\GLS}$ and $\hx_{\ECGLS} = (I_{kd \times kd} - LR) \hx_{\GLS} + Lr$, where $L = \hSigma_{\GLS} \EQ^\top (\EQ \hSigma_{\GLS} \EQ^\top)^{-1}$  (see Lemma~\ref{lem:ecgls}). We have already shown that $\hSigma_{\GLS}$ is a succinct matrix, so let $\hSigma_{\GLS} = A \otimes P_0 + B \otimes P_1$. We show the following lemma.

\begin{lemma}\label{lem:projection-structure}
Under the structure of \eqref{eq:symmetry-structure}, let $\hSigma_{\GLS} = A \otimes P_0 + B \otimes P_1$. For $\tilde{\EQ} = \begin{bmatrix} \EQ_1 \\ \EQ_2 \end{bmatrix}$ let
\begin{align}
    L_A &~:=~ A \EQ_1^\top (\EQ_1 A \EQ_1^\top)^{-1} \\
    L_B &~:=~ B \tilde{R}^\top (\tilde{R} B \tilde{R}^\top)^{-1} =: \begin{bmatrix}
        L_{B,1} & L_{B,2}
    \end{bmatrix}\\
    \tilde{A} &~:=~ A \EQ_1^\top (\EQ_1 A \EQ_1^\top)^{-1} \EQ_1 = L_A R_1\label{eq:tilde-A}\\
    \tilde{B} &~:=~ B \tilde{\EQ}^\top (\tilde{\EQ} B \tilde{\EQ}^\top)^{-1} \tilde{\EQ} = L_B \tilde{R}\label{eq:tilde-B}
\end{align}
where, $L_B$ has shape $|\Ba| \times (e_1 + e_2)$, and we define the blocks $L_{B,1}$ and $L_{B,2}$ to be of shape $|\Ba| \times e_1$ and $|\Ba| \times e_2$ respectively. Then, it holds that,
\begin{align}
    L ~=~ \hSigma_{\GLS} \EQ^\top (\EQ \hSigma_{\GLS} \EQ^\top)^{-1} &~=~ \begin{bmatrix}
    L_A \otimes P_0 + L_{B,1} \otimes P_1 & \frac1d L_{B,2} \otimes \ones_{d \times 1} \end{bmatrix}\\
    LR ~=~ \hSigma_{\GLS} \EQ^\top (\EQ \hSigma_{\GLS} \EQ^\top)^{-1} \EQ &~=~ \tilde{A} \otimes P_0 + \tilde{B} \otimes P_1
\end{align}
\end{lemma}

The proof of this lemma is technical, so we defer it to a later section. Using this lemma, it is easy to see that $\hSigma_{\ECGLS}$ and $\hx_{\ECGLS}$ can be computed efficiently and that $\hSigma_{\ECGLS}$ is a succinct matrix. Since $(I - LR)$ is a succinct matrix, it is efficient to compute $(I - LR)\hSigma_{\GLS}$ and $(I - LR)\hx_{\GLS}$. Furthermore, $Lr = (L_A \otimes P_0 + L_{B,1} \otimes P_1) r_1 + \frac1d (L_{B,2} \otimes \ones_{d \times 1}) r_2$ can be computed efficiently using the vec trick~(Fact~\ref{fact:kronecker}).

\subsection{Efficient implementation of $\CombineEstimators$}

Let $\hOmega_1 := A_1 \otimes P_0 + B_1 \otimes P_1$ and $\hOmega_2 := A_2 \otimes P_0 + B_2 \otimes P_1$.
$\CombineEstimators$ involves computing the following quantities, which we show how to compute efficiently:
\begin{itemize}
\item $A \gets \hOmega_2 (\hOmega_1 + \hOmega_2)^\pinv$. We use multiplication and inverse algebra of succinct matrices (Lemma~\ref{lem:uncompressed-inverse-v2}) to observe that
\begin{align*}
    A &~=~ (A_2 (A_1 + A_2)^\pinv) \otimes P_0 + (B_2 (B_1 + B_2)^\pinv) \otimes P_1
\end{align*}
\item $B = I - A = (I - A_2 (A_1 + A_2)^\pinv) \otimes P_0 + (I - B_2 (B_1 + B_2)^\pinv) \otimes P_1$.
\item $\hz = A \hz_1 + B \hz_2$ can be computed efficiently using the vec trick (Fact~\ref{fact:kronecker}).
\item $\hOmega = B \hOmega_2$ can be computed efficiently as the multiplication of two succinct matrices.
\end{itemize}

\subsection{Proof of Lemma~\ref{lem:projection-structure}}
The following lemma serves as a crucial step towards Lemma~\ref{lem:projection-structure}.

\begin{lemma}\label{lem:structure-block-inverse}
Let $F, G_{00} \in \R^{m \times m}, G_{01} \in \R^{m \times n}, G_{10} \in \R^{n \times m}, G_{11} \in \R^{n \times n}$ with $F$ and $G$ (as defined below) being non-singular. Let
\[
G := \begin{bmatrix} G_{00} & G_{01}\\ G_{10} & G_{11} \end{bmatrix}
\text{ and }
H := \begin{bmatrix} H_{00} & H_{01}\\ H_{10} & H_{11}\end{bmatrix} := G^{-1}.
\]
Then, for
\[
M = \begin{bmatrix}
    F \otimes P_0 + G_{00} \otimes P_1 & G_{01} \otimes \ones_{d \times 1} \\
    G_{10} \otimes \ones_{1 \times d} & d \cdot G_{11}
\end{bmatrix}\]
it holds that
\[
M^{-1} = \begin{bmatrix}
    F^{-1} \otimes P_0 + H_{00} \otimes P_1 & \frac1d H_{01} \otimes \ones_{d \times 1} \\
    \frac1d H_{10} \otimes \ones_{1 \times d} & \frac1d H_{11}
\end{bmatrix}.
\]

\end{lemma}
\begin{proof}
We verify that $M M^{-1} = I$ for the claimed $M^{-1}$. First note that $G H = I$, which is equivalent to
\begin{align*}
    \begin{matrix}
    G_{00} H_{00} + G_{01} H_{10} = I &
    G_{00} H_{01} + G_{01} H_{11} = 0 \\
    G_{10} H_{00} + G_{11} H_{10} = 0 &
    G_{10} H_{01} + G_{11} H_{11} = I    
    \end{matrix}
\end{align*}
Using the above and properties of $P_0$ and $P_1$ (namely, $P_0 \ones_{d \times 1} = 0$, $P_1 \ones_{d \times 1} = \ones_{d \times 1}$ and $\frac1d \ones_{d \times 1} \ones_{1 \times d} = P_1$), we get the following
\begin{align*}
& M M^{-1} \\
&~=~ 
\prn{\begin{bmatrix}
        F \otimes P_0 & 0 \\
        0 & 0
    \end{bmatrix}
    +
    \begin{bmatrix}
        G_{00} \otimes P_1 & G_{01} \otimes \ones_{d \times 1} \\
        G_{10} \otimes \ones_{1 \times d} & d \cdot G_{11}
    \end{bmatrix}
}
\prn{
    \begin{bmatrix}
        F^{-1} \otimes P_0 & 0 \\
        0 & 0
    \end{bmatrix}
    +
    \begin{bmatrix}
        H_{00} \otimes P_1 & \frac1d H_{01} \otimes \ones_{d \times 1} \\
        \frac1d H_{10} \otimes \ones_{1 \times d} & \frac1d H_{11}
    \end{bmatrix}
}\\
&~=~
\begin{bmatrix}
    I \otimes P_0 & 0 \\
    0 & 0
\end{bmatrix}
+
\begin{bmatrix}
    G_{00} \otimes P_1 & G_{01} \otimes \ones_{d \times 1} \\
    G_{10} \otimes \ones_{1 \times d} & d \cdot G_{11}
\end{bmatrix}
\begin{bmatrix}
    H_{00} \otimes P_1 & \frac1d H_{01} \otimes \ones_{d \times 1} \\
    \frac1d H_{10} \otimes \ones_{1 \times d} & \frac1d H_{11}
\end{bmatrix}\\
&~=~
\begin{bmatrix}
    I \otimes P_0 & 0 \\
    0 & 0
\end{bmatrix}
+
\begin{bmatrix}
    (G_{00} H_{00} + G_{01} H_{10}) \otimes P_1 &
    \frac1d (G_{00} H_{01} + G_{01} H_{11}) \otimes \ones_{d \times 1} \\
    (G_{10} H_{00} + G_{11} H_{10}) \otimes \ones_{1 \times d} &
    G_{10} H_{01} + G_{11} H_{11}
\end{bmatrix}\\
&~=~
\begin{bmatrix}
    I \otimes P_0 + I \otimes P_1 & 0 \\
    0 & I
\end{bmatrix} ~=~ I\,. 
\end{align*}
\end{proof}

\begin{proof}[Proof of Lemma~\ref{lem:projection-structure}]
Recall that, we have the following structure:
\[
R = \begin{bmatrix}
    R_1 \otimes I \\ R_2 \otimes \ones_{1 \times d}
\end{bmatrix}
\qquad \text{ and } \qquad 
\hSigma_{\GLS} = A \otimes P_0 + B \otimes P_1
\]
We expand $R \hSigma_{\GLS} R^\top$ as follows.
\begin{align*}
R \hSigma_{\GLS} R^\top &~=~ 
\begin{bmatrix}
    R_1 \otimes I \\ R_2 \otimes \ones_{1 \times d}
\end{bmatrix}
(A \otimes P_0 + B \otimes P_1)
\begin{bmatrix}
    R_1^\top \otimes I & R_2^\top \otimes \ones_{d \times 1}
\end{bmatrix}\\
&~=~ \begin{bmatrix}
    R_1 A R_1^\top \otimes P_0 + R_1 B R_1^\top \otimes P_1
    & R_1 B R_2^\top \otimes \ones_{d \times 1}\\
    R_2 B R_1^\top \otimes \ones_{1 \times d} &
    d \cdot R_2 B R_2^\top
\end{bmatrix}
\end{align*}
Thus, $R \hSigma_{\GLS} R^\top$ has precisely the structure required in Lemma~\ref{lem:structure-block-inverse} for
\[
F = R_1 A R_1^\top
\quad \text{and} \quad
G = \begin{bmatrix}
    R_1 B R_1^\top & R_1 B R_2^\top \\
    R_2 B R_1^\top & R_2 B R_2^\top
\end{bmatrix}
= \tilde{R} B \tilde{R}^\top
\ \ \text{ where }
\tilde{R} := \begin{bmatrix} R_1 \\ R_2 \end{bmatrix}
\]
Note that $F$ and $G$ are invertible assuming $R_1$ and $\begin{bmatrix}R_1 \\ R_2\end{bmatrix}$ are both full rank, which can be assumed without loss of generality.
Thus, applying Lemma~\ref{lem:structure-block-inverse}, we get
\begin{align*}
    (R \hSigma_{\GLS} R^\top)^{-1} &~=~
    \begin{bmatrix}
        (R_1 A R_1^\top)^{-1} \otimes P_0 + H_{00} \otimes P_1 & \frac1d H_{01} \otimes \ones_{d \times 1} \\
        \frac 1d H_{10} \otimes \ones_{1 \times d} & \frac1d H_{11}
    \end{bmatrix}\\
    \text{where}\quad
    \begin{bmatrix}
        H_{00} & H_{01}\\
        H_{10} & H_{11}
    \end{bmatrix} &~=~ (\tilde{R} B \tilde{R}^\top)^{-1}
\end{align*}
Next, we get that
\begin{align*}
&R^\top (R \hSigma_{\GLS} R^\top)^{-1}\\
&~=~ 
\begin{bmatrix}
    R_1^\top \otimes I &
    R_2^\top \otimes \ones_{d \times 1}
\end{bmatrix}
\begin{bmatrix}
    (R_1 A R_1^\top)^{-1} \otimes P_0 + H_{00} \otimes P_1 & \frac1d H_{01} \otimes \ones_{d \times 1} \\
    \frac 1d H_{10} \otimes \ones_{1 \times d} & \frac1d H_{11}
\end{bmatrix}\\
&~=~ \begin{bmatrix}
    R_1^\top (R_1 A R_1^\top)^{-1} \otimes P_0 & 0
\end{bmatrix}
+ \begin{bmatrix}
    (R_1^\top H_{00} + R_2^\top H_{10}) \otimes P_1 &
    \frac1d (R_1^\top H_{01} + R_2^\top H_{11}) \otimes \ones_{d \times 1}
\end{bmatrix}
\end{align*}
Thus, we get
\begin{align*}
&L := \hSigma_{\GLS} R^\top (R \hSigma_{\GLS} R^\top)^{-1}\\
&~=~ (A \otimes P_0 + B \otimes P_1) \begin{bmatrix}
    R_1^\top (R_1 A R_1^\top)^{-1} \otimes P_0 & 0
\end{bmatrix}\\
&~~~~+~ (A \otimes P_0 + B \otimes P_1) \begin{bmatrix}
    (R_1^\top H_{00} + R_2^\top H_{10}) \otimes P_1 &
    \frac1d (R_1^\top H_{01} + R_2^\top H_{11}) \otimes \ones_{d \times 1}
\end{bmatrix}\\
&~=~ \begin{bmatrix}
    A R_1^\top (R_1 A R_1^\top)^{-1} \otimes P_0 & 0
\end{bmatrix}
+ \begin{bmatrix}
    B (R_1^\top H_{00} + R_2^\top H_{10}) \otimes P_1 &
    \frac1d B (R_1^\top H_{01} + R_2^\top H_{11}) \otimes \ones_{d \times 1}
\end{bmatrix}\\
&~=~ \begin{bmatrix}
    L_A \otimes P_0 + L_{B,1} \otimes P_1 & \frac1d L_{B,2} \otimes \ones_{d \times 1}
\end{bmatrix}
\end{align*}
where $L_A := A R_1^\top (R_1 A R_1^\top)^{-1}$ and $L_B = \begin{bmatrix}
    L_{B,1} & L_{B,2}
\end{bmatrix} := B \tilde{R}^\top (\tilde{R} B \tilde{R}^\top)^{-1}$.
Finally, it is now immediate to see that
\begin{align*}
LR &~:=~ \hSigma_{\GLS} R^\top (R \hSigma_{\GLS} R^\top)^{-1} R\\
&~=~ \begin{bmatrix}
    L_A \otimes P_0 + L_{B,1} \otimes P_1 & \frac1d L_{B,2} \otimes \ones_{d \times 1}
\end{bmatrix} \begin{bmatrix}
    R_1 \otimes I \\
    R_2 \otimes \ones_{1 \times d}
\end{bmatrix}\\
&~=~ \tilde{A} \otimes P_0 + \tilde{B} \otimes P_1
\end{align*}
where $\tilde{A} = L_A R_1 = A R_1^\top (R_1 A R_1^\top)^{-1} R_1$ and $\tilde{B} = L_B \tilde{R} = B \tilde{R}^\top (\tilde{R} B \tilde{R}^\top)^{-1} \tilde{R}$.
\end{proof}

\section{Handling inequality and integrality constraints}
\label{sec:top-down-pass}
As discussed earlier Algorithm~\ref{alg:tree-post-processing} provides the BLUE for $x^\star$. However, that algorithm only takes into account the NMF estimates and equality constraints. In order to also account for the inequality and integrality constraints, we step outside the realm of {\em best linear unbiased estimators}, and provide a heuristic modification to Algorithm~\ref{alg:tree-post-processing} using mixed integer programming (using Gurobi \cite{gurobi} or another solver) on the final top-down pass. We refer to our algorithm as \BlueDown{} (Algorithm~\ref{alg:bluedown}); our naming follows the Census Bureau \TopDown{} algorithm, which also used Gurobi for mixed integer programming~\cite{AbowdAC+22}.

In particular, \BlueDown{} uses the same per-node processing and bottom-up pass phases as the linear estimator~(Algorithm~\ref{alg:tree-post-processing}). These phases provide us with estimates $\hz_v, \zup_v, \zUp_v$ for each node $v$, as well as their corresponding covariances $\var_v, \varup_v, \varUp_v$. We replace the top-down pass phase with a sequence of mixed-integer programming problems that can be solved using Gurobi. If there are no inequality or integrality constraints, the resulting quadratic programming problems are equivalent to Algorithm~\ref{alg:tree-post-processing}, but this is no longer the case when these additional constraints are included.

In the top-down pass in Algorithm~\ref{alg:tree-post-processing}, the estimator $\hx_v$ is constructed as a function of the estimator $\hzDown_p$ for the parent of $v$ and $\hzUp_u$ for all children of $p$.
While in Algorithm~\ref{alg:tree-post-processing} we compute $\hx_v$ directly for each individual child $v$ of $p$ using invocations of $\CombineEstimators$, this is equivalent to performing the following joint optimization using the final postprocessed parent estimate $\hx_p$:

\begin{align} \label{eq:joint-opt} 
(\hx_{v})_{v \in \child(p)} \gets
& \argmin_{x_v \in \R^{|\cB|} \, : \, v \in \child(p)} \ \ 
\sum_{v \in \child(p)}
(x_v - \zUp_v)^\top (\varUp_v)^{\dagger}
(x_v - \zUp_v)\\
\text{subject to:}\quad 
& \EQ_v x_v = \eq_v \quad \text{for all } v \in \child(p), \nonumber\\
& \textstyle \sum_{v \in \child(p)} x_v = \hx_p \nonumber
\end{align}

Since the pseudoinverse is often numerically unstable to compute, we replace $(\varUp_v)^\dagger$ by $(\varUp_v + \alpha \EQ_v^\top \EQ_v)^{-1}$ for any constant $\alpha > 0$, where we can assume without loss of generality that $\varUp_v + \alpha \EQ_v^\top \EQ_v$ is invertible. Invertibility holds because any equality constraint on $x_v$ that is collectively implied by the equality constraints on $x_u$ for all descendants $u$ of $v$ can be assumed to be included in $\EQ_v x_v = \eq_v$, and so all the eigenvectors of $\varUp_v$ with zero eigenvalues are in the rowspan of $\EQ_v$.
This perturbation does not change the solution, since we are constraining $x_v$ in the directions in which the objective does not impose any penalty.

\subsection{Multi-phase quadratic program and rounder.}
Our approach is to modify the joint optimization objective (Equation~\ref{eq:joint-opt}) to incorporate the inequality and integrality constraints (and use a mixed integer program solver). While in principle it is possible to incorporate all the constraints directly into the above program, the Census \TopDown{} Algorithm instead takes a {\em multi-pass} approach, where they consider a sequence of optimization problems. Each successive optimization problem computes a partial solution taking into account only part of the objective or only some of the constraints, which is refined in subsequent problems to obtain the full estimator $(\hx_u)_{u \in \child(p)}$. We follow essentially the same approach, with the key difference being that we use our bottom-up estimates $\hzUp_v$ as the starting point, instead of the raw estimates $\hz_v$. Since the bottom-up estimates $\hzUp_v$ are not independent (whereas the raw estimates $\hz_v$ are independent), this requires extending the approach of \cite{AbowdAC+22} to work with non-diagonal covariance matrices. This extension raises additional computational challenges that we address below. The modified least-squares estimator and rounder are described in Algorithms~\ref{alg:l2-multipass} and~\ref{alg:rounder-multipass}, and the multi-pass routine that may invoke each multiple times is described in Algorithm~\ref{alg:multipass}.

While the algorithm can be flexibly configured with any set of increasingly granular refined solutions, our experiments use the same configuration as the 2020 US Census \cite{AbowdAC+22} and use either one or two passes for both the least-squares quadratic program and the rounder routines at each level of the tree.
(To avoid introducing additional notation, Algorithm~\ref{alg:multipass} implicitly assumes that the least squares quadratic program and rounder phase use the same passes and projection operators, which is the case in the configuration we use. One can easily modify the algorithm to allow different sets of projection operators in the two phases.)
The \total{} pass uses the all-ones vector projection operator $\ones_\un^\top$ and minimizes the error of the total count for each child geocode. The \full{} pass uses the identity projection operator $I_\un$ and minimizes the overall error of the full estimate and covariance matrix. This pass is analogous to the second pass used in the 2020 Census \cite{AbowdAC+22}, which considered all input queries, but allows generalization to non-diagonal covariance matrices. Again matching the production setting used for the 2020 Census \cite{AbowdAC+22}, we use a single \full{} pass $\mathscr{Q}=(I_\un)$ for the US (root) and Block (leaf) levels, and two (\total{} and \full{}) passes $\mathscr{Q} = (\ones_\un^\top, I_\un)$ for all other levels of the tree (State, County, Tract, and Block Group). 

\subsection{Efficient operations leveraging succinct matrices}
Two optimizations substantially improve the time required to solve the least-squares quadratic optimization problems: using variable substitutions to reduce the number of terms in the objective function, and leveraging succinct matrices to compute matrix products more efficiently.
The time it takes to run a single iteration of an iterative solver such as Gurobi is generally proportional to the complexity of the objective function.
A direct encoding of the optimization problem in Algorithm~\ref{alg:l2-multipass} even in the special case $Q=I_\un$ (i.e.\ the \full{} pass) would have $4\un^2$ monomial objective terms for each child node, which would incur a large slowdown compared to the simpler optimization problem used in the 2020 Census, whose objective function size is a small multiple of $\un$ rather than $\un^2$. However, using auxiliary variables we can encode the objective function of Algorithm~\ref{alg:l2-multipass} in an alternate form so that it has at most $\un$ objective terms, yielding an end-to-end algorithm with very similar runtime to the 2020 Census \TopDown{} algorithm.

We first illustrate with the \full{} pass of Algorithm~\ref{alg:l2-multipass}, where $Q=I_\un$. In this case the least-squares objective for a single child node $v$ is $(z_v - \hz_v)^\top P_v (z_v - \hz_v)$, where $P_v = (Q(\varUp_v + \alpha \EQ_v^\top \EQ_v)Q^\top)^{-1} = (\varUp_v + \alpha \EQ_v^\top \EQ_v)^{-1}$.
Let $P_v = L_v L_v^\top$ be a factorization of the symmetric positive-definite matrix $P_v$ (e.g.\ the Cholesky decomposition), where $L_v$ is lower triangular. Define auxiliary variables $y_v = L_v^\top (z_v - \hz_v)$. Using these additional $\un$ variables and equality constraints, the objective for child node $v$ can be written as the sum of $\un$ quadratic objective terms $y_v^\top y_v = \sum_{i=1}^{\un} y_{v,i}^2$, using only $\un$ quadratic monomials.
Moreover, when some coordinates of $z_v$ are constrained to zero, we can omit the corresponding variables and restrict $P_v$ to the remaining unconstrained coordinates, further reducing the number of variables and objective terms.
With general $Q \in \R^{q\times \un}$ the objective for a single child node $v$ is 
$(Q(z_v - \hz_v))^\top P_v (Q(z_v - \hz_v))$; again factoring $P_v = L_v L_v^\top$ and defining $q$ auxiliary variables $y_v = L_v^\top Q(z_v - \hz_v)$,
we can similarly rewrite the objective as $y_v^\top y_v$ using only $q$ quadratic objective terms with $q$ additional variables and linear equality constraints.

Further efficiency improvements can be obtained leveraging the succinct matrix structure of $\varUp$, $Q$ and $\EQ$. Recall 
that $\varUp_v$ is of the form $\varUp_v = A_0 \otimes P_0 + A_1 \otimes P_1$ and $\EQ_v$ is of the form 
\[\EQ_v = \begin{bmatrix}
        \EQ_1 \otimes I_{\Bs} \\
        \EQ_2 \otimes \ones_{\Bs}^\top
\end{bmatrix}\]
for some $A_0, A_1 \in \R^{|\Ba| \times |\Ba|}$ and $\EQ_1 \in \R^{e_1 \times |\Ba|}, \EQ_2\in \R^{e_2 \times |\Ba|}$, where $P_0 = I_{\Bs} - \frac{1}{d} J_{\Bs}$ and $P_1 = \frac{1}{d} J_{\Bs}$, for $d = |\Bs|$.
Then $\EQ_v^\top \EQ_v = \EQ_1^\top \EQ_1 \otimes I_{\Bs} + \EQ_2^\top \EQ_2 \otimes J_{\Bs}$,
so $\varUp_v + \alpha \EQ_v^\top \EQ_v$ is also a succinct
matrix of the form $C_0 \otimes P_0 + C_1 \otimes P_1$ for $C_0, C_1 \in \R^{|\Ba| \times |\Ba|}$ given by
\begin{align*}
C_0 &= A_0 + \alpha \EQ_1^\top \EQ_1 \\
C_1 &= A_1 + \alpha\EQ_1^\top \EQ_1 + d \alpha\EQ_2^\top \EQ_2.
\end{align*}

For $Q = \ones_\un^\top = \ones_{\Ba}^\top \otimes \ones_{\Bs}^\top$ as in the \total{} pass,
$Q(\varUp_v + \alpha \EQ_v^\top \EQ_v)Q^\top$ is simply the
sum of the entries of the matrix $C_0 \otimes P_0 + C_1 \otimes P_1$, so its inverse is the scalar 
\begin{align*}
P_v = (Q(\varUp_v + \alpha \EQ_v^\top \EQ_v)Q^\top)^{-1}
&= (\ones_{\Ba}^\top C_0 \ones_{\Ba} \otimes \ones_{\Bs}^\top P_0 \ones_{\Bs} + \ones_{\Ba}^\top C_1 \ones_{\Ba} \otimes \ones_{\Bs}^\top P_1 \ones_{\Bs})^{-1}\\
&= (|\Bs| \cdot \ones_{\Ba}^\top C_1 \ones_{\Ba})^{-1}.
\end{align*}
Similarly, if $Q = I_\un = I_{\Ba} \otimes I_{\Bs}$ as in the \full{} pass, then $P_v = (Q(\varUp_v + \alpha \EQ_v^\top \EQ_v)Q^\top)^{-1}$
is simply $(C_0 \otimes P_0 + C_1 \otimes P_1)^{-1} = C_0^{-1} \otimes P_0 + C_1^{-1} \otimes P_1$, as shown above.
Taking the Cholesky decompositions $L_0 L_0^\top = C_0^{-1}$ and $L_1 L_1^\top = C_1^{-1}$ of $C_0^{-1}$ and $C_1^{-1}$, where $L_0, L_1 \in \R^{|\Ba| \times |\Ba|}$ are lower triangular, the succinct matrix $L_v = L_0 \otimes P_0 + L_1 \otimes P_1$ satisfies $L_v L_v^\top = P_v$ by the multiplication rule for succinct matrices.
Consequently, for $Q\in \{\ones_\un^\top, I_\un\}$ as in the \total{} and \full{} passes, we can compute $P_v$, its factors $L_v, L_v^\top$, and $y_v = L_v^\top Q (z_v - \hz_v)$ efficiently in the succinct representation without expanding the Kronecker products.
Using auxiliary variables for the averages of $z_v - \hz_v$ across $\Bs$ and exploiting the lower triangular structure of $L_0$ and $L_1$, we can additionally make the encoding of $y_v$ sparse, further improving solver runtime.
More generally, if $Q$ is of the same form as $\EQ$, i.e.\ if 
\[
Q=\begin{bmatrix}
        Q_1 \otimes I_{\Bs} \\
        Q_2 \otimes \ones_{\Bs}^\top
\end{bmatrix}
\]
for some $Q_1\in \R^{q_1 \times |\Ba|}, Q_2\in \R^{q_2 \times |\Ba|}$, then we can use Lemma~\ref{lem:structure-block-inverse} to invert $Q(\varUp_v + \alpha \EQ_v^\top \EQ_v)Q^\top$ efficiently.

We formalize our final end-to-end algorithm in Algorithm~\ref{alg:bluedown}.

\section{The Census \TopDown{} algorithm}
\label{sec:topdown-comparison}
For completeness, we now give a short description of the full Census Bureau \TopDown{} algorithm~\cite{AbowdAC+22} that we use as our main baseline,
rephrased to match the notation of \BlueDown.
The \TopDown{} algorithm follows a multi-phase optimization approach that we have emulated in the top-down optimization pass of \BlueDown. The principal distinction is that while the \BlueDown{} optimization pass operates on already-postprocessed bottom-up estimates $\hzUp$, \TopDown{} operates directly on the raw input estimates from the NMF. Since the NMF estimates are independent, this means that corresponding covariance matrix is diagonal, which simplifies computations and limits the number of terms in the optimization objective without requiring the succinct representation employed in \BlueDown{}. Consequently, \TopDown{} can be viewed as a special case of one component of \BlueDown, with both algorithms employing the same least squares optimization and rounder phases.

Due to the close relationship between the algorithms, we will describe the modifications needed to Algorithms~\ref{alg:l2-multipass} and \ref{alg:multipass} in order to recover \TopDown{}. The rounder pass, Algorithm~\ref{alg:rounder-multipass}, is entirely unchanged.

For the least-squares estimator (Algorithm~\ref{alg:l2-multipass}), we now invoke this subroutine with NMF measurements that aggregate over variables. Let these measurements and their covariances for node $v$ be $(M_v, \hOmega_v)$, which we will input to the algorithm in place of $(\hz_v, \hOmega_v)$.
We rewrite the terms of the optimized quadratic form as $(Qz_v - M_v)^\top P_v (Qz_v - M_v)$ instead of 
$(Q(z_v - \hz_v))^\top P_v (Q(z_v - \hz_v))$. This reformulation allows us to input measurements $M_v$ that correspond to aggregates of variables, which was unnecessary for \BlueDown{} because in place of measurements there we used partially postprocessed estimates $\hz_v$ of the variables themselves.
The corresponding covariance $\hOmega_v$ is diagonal and full-rank, so we can let $P_v = \hOmega_v^{-1}$, inverting each diagonal element, without adding $\alpha R_v^\top R_v$ because $\hOmega_v$ is invertible and without multiplying by $Q$ and $Q^\top$ because the covariance $\hOmega_v$ applies directly to the measurements $M_v$ without projecting.

To recover the \TopDown{} algorithm, we rewrite the overall multi-pass algorithm (Algorithm~\ref{alg:multipass}) as follows. Here as well, instead of partially processed estimates and their covariances $(\hzUp_v, \varUp_v)$, we now have input measurements $(M_v^{(j)}, \hOmega_v^{(j)})$, where the superscript indicates that different measurements may be provided to each phase of the least squares estimator. With these changes to Algorithms~\ref{alg:l2-multipass} and \ref{alg:multipass}, leaving the rounder pass unchanged, we recover \TopDown.

It remains to specify the number of passes and which measurements and projection operators are used in each pass at each level of the tree. At the US and block levels a single pass with all queries is used for both least-squares and rounder phases. At the remaining levels (state, county, tract, and block group), both the least-squares and rounder phase consist of two passes, \total{} and \full{}. In the \total{} pass, only the Total query is used (i.e. $M_v^{(1)}$ consists of only the Total query and $Q_1 = \ones_{2016}^\top$). In the \full{} pass, all remaining queries from Tables~\ref{tab:per-node-workload-single-feature}--\ref{tab:per-node-workload-multi-feature} are used for both the measurements $M_v^{(2)}$ and projection operator $Q_2$. These configurations are as chosen in \cite{AbowdAC+22}.

We briefly note one additional distinction between our presentation and the original description of the \TopDown{} algorithm. The \cite{AbowdAC+22} formulation included a slack term $\tau$ in the multi-pass least-squares estimation step to ensure the feasibility of future solutions. We find that this slack term is unnecessary in the production version of the algorithm, and so we omit it.

\section{Efficiency}
\label{sec:efficiency}
We briefly discuss the concrete computational efficiency of our algorithm, focusing on the significance of the succinct encoding scheme.
Our full algorithm (Algorithm~\ref{alg:bluedown}) has comparable total runtime to the Census Bureau TopDown algorithm \cite{AbowdAC+22}, since the rate-limiting step is a sequence of quadratic and integer optimizations over objective functions of similar complexity.
Moreover, both algorithms are natively highly parallel: apart from a very fast root optimization step, computations can be performed entirely independently for each of the 88 state-level regions, and one could also parallelize over counties or lower levels of the hierarchy if needed.  
While we observed some variation in total wall clock time across trials, likely due in part to variations in scheduling on shared machines and latency in RPC calls used to invoke the Gurobi optimizer,
both TopDown and the top-down phase of BlueDown took between 2 and 3 days of CPU time on the largest state in typical runs. The bottom-up phase of BlueDown took an additional 1 to 4 hours on the largest states, including scheduling.

Our succinct representation is critical for achieving this near parity in running time with TopDown, ensuring that our accuracy improvements are practically achievable.
To highlight this, we compared the timing of the succinct and uncompressed versions of the $\CombineEstimators$ subroutine (Algorithm~\ref{alg:combine-estimates}), which is the core step of both our standalone BLUE algorithm (Algorithm~\ref{alg:tree-post-processing}) and the bottom-up pass of BlueDown (Algorithm~\ref{alg:bluedown}). We apply each version to a sequence of pairs of randomly sampled estimate vectors and random positive-definite covariance matrices with Kronecker structure compatible with the succinct representation.
In the succinct form, 400{,}000 invocations of $\CombineEstimators$ were completed in $411.4$ seconds, or $0.001$ seconds per invocation. In the uncompressed form, 40 invocations were completed in $432.4$ seconds, or $10.8$ seconds per invocation, which is over $10^4$ times slower.
The succinct version of Algorithm~\ref{alg:tree-post-processing} should also achieve a similar speedup compared to 
the BLUE algorithm described by Cumings-Menon and Zhuravlev \cite{cumingsmenon2025squaresestimationhierarchicaldata}, as that algorithm performs essentially the same computations as the uncompressed version of our algorithm.

The succinct representation is similarly crucial for the top-down pass of \BlueDown. The uncompressed version is so slow that it is impractical to benchmark explicitly, but we expect that it would also be thousands of times slower than the succinct version (and \TopDown) based on the number of terms in the quadratic objective. 
Overall, the succinct representation is critical for practical application of the method at US Census scale.

\section{Additional related work}
\label{sec:additional-related-work}
The block hierarchical estimation problem characterizing the structure of the census measurements can be viewed as a generalization of tree-structured queries studied in \cite{hay2009boosting,cormode2012differentially,honaker2015efficient,DawsonGK0KLMMNS23}. Algorithms developed in these works share a similar structure with our Algorithm~\ref{alg:tree-post-processing}, a component of \BlueDown. However, additional challenges are raised in our setting since the census measurements are only partially or {\em block} hierarchical. Our algorithm additionally supports ``per-node'' equality constraints, which arise only in the block-hierarchical setting.

More closely related is the work of Cumings-Menon and Zhuravlev \cite{cumingsmenon2025squaresestimationhierarchicaldata}, who developed an algorithm very similar to our Algorithm \ref{alg:tree-post-processing} subroutine for computing the least-squares optimal linear unbiased estimator of block hierarchical estimates.
The algorithm of Cumings-Menon and Zhuravlev is equivalent to the special case of Algorithm \ref{alg:tree-post-processing} applied in the absence of equality constraints, in the sense that both compute the unique best linear unbiased estimator.
However, Algorithm~\ref{alg:tree-post-processing} also optimally incorporates linear equality constraints, requiring the use of more robust formulas that remain valid with singular covariance matrices. Furthermore, as discussed in \ref{sec:efficiency}, the succinct representation introduced in \ref{sec:symmetries} allows us to implement Algorithm~\ref{alg:tree-post-processing} much more efficiently by leveraging symmetries of the measurement structure.

The algorithms of
\cite{hay2009boosting,cormode2012differentially,honaker2015efficient,DawsonGK0KLMMNS23,cumingsmenon2025squaresestimationhierarchicaldata}, 
as well as our Algorithm \ref{alg:tree-post-processing}, belong to the class of DP \emph{matrix mechanisms} or \emph{factorization mechanisms} \cite{LiMHMR15,NikolovTZ16}, which apply a linear transformation to data, add independent noise to the transformed data, and then apply a second linear transformation to the noisy transformed data to compute query answers. Matrix mechanisms enjoy strong guarantees: when the transformations are selected carefully, they can achieve asymptotically minimal error among all DP mechanisms when the size of the dataset is sufficiently large~\cite{NikolovTZ16,EdmondsNU20}.
Choosing an optimal factorization of a query workload into such a pair of linear transformations can be formulated as a semi-definite program (e.g.~\cite{MNT20}), but solving such programs is impractical for large workloads.
While attempts have been made to make the general matrix mechanism more efficient~\cite{McKennaMHM23}, it remains infeasible at the scale of the US Census. Our work can be seen as implicitly applying the optimal second linear transformation for the US Census workload, without explicitly computing and materializing the transformation matrix.

Our focus is on improved postprocessing of noisy estimates. A complementary line of work explores how to choose the privacy budget allocation used to generate the initial noisy estimates \cite{CohenDMS21Census,cohenDMS22private,DawsonGK0KLMMNS23,ko2025optimalallocationprivacybudget}. Optimizing the privacy budget allocation specifically for \BlueDown{} should amplify the accuracy improvement.

In a recent work, \cite{su2025usdecennial} demonstrated that the privacy guarantees provided by the Census DAS are better than originally reported, using improved privacy accounting analysis. Because our work consumes the same noisy measurements as the Census DAS, the privacy guarantees are identical. Thus, this improved privacy analysis applies to our \BlueDown{} algorithm as well.

\newpage

\begin{figure}
\centering
\includegraphics[width=\textwidth]{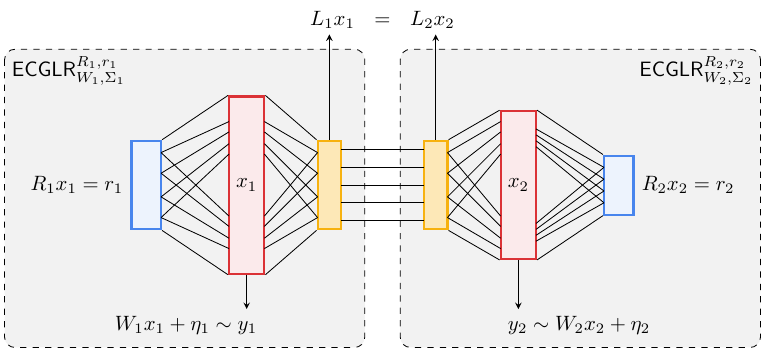}
\caption{\textbf{Visualization of the Factorized-$\ECGLR$ problem.}
This abstraction captures two overlapping equality-constrained generalized linear regression problems, where the variables $x_1, x_2$ are related by the linear constraint 
$L_1 x_1 = L_2 x_2$. Lemma~\ref{lem:combine-equality-optimal} shows that the optimal estimator of the linear expression $L_1 x_1 $ (or equivalently $L_2 x_2$) in the combined regression problem can be computed using Algorithm~\ref{alg:combine-estimates} from the transformed optimal solutions $L_1 x_1^\star$, $L_2 x_2^\star$ to the two individual problems. This step is one of the key building blocks of our optimal block hierarchical generalized linear regression algorithm (Algorithm~\ref{alg:tree-post-processing}).
}
\label{fig:combine-problem}
\end{figure}

\begin{figure}
\centering
\includegraphics[width=0.4\textwidth]{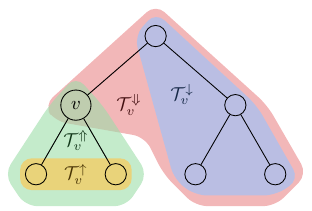}
\caption{\textbf{Subgraphs used in Theorem~\ref{thm:tree-optimality} to prove optimality of Algorithm~\ref{alg:tree-post-processing}.} These subgraphs illustrate subproblems that are solved recursively in Algorithm~\ref{alg:tree-post-processing} to obtain the full optimal block hierarchical regression. Figure adapted from \cite{DawsonGK0KLMMNS23}.}
\label{fig:post-processing-subgraphs}
\end{figure}

\begin{figure}
\begin{center}
\includegraphics[width=\textwidth]{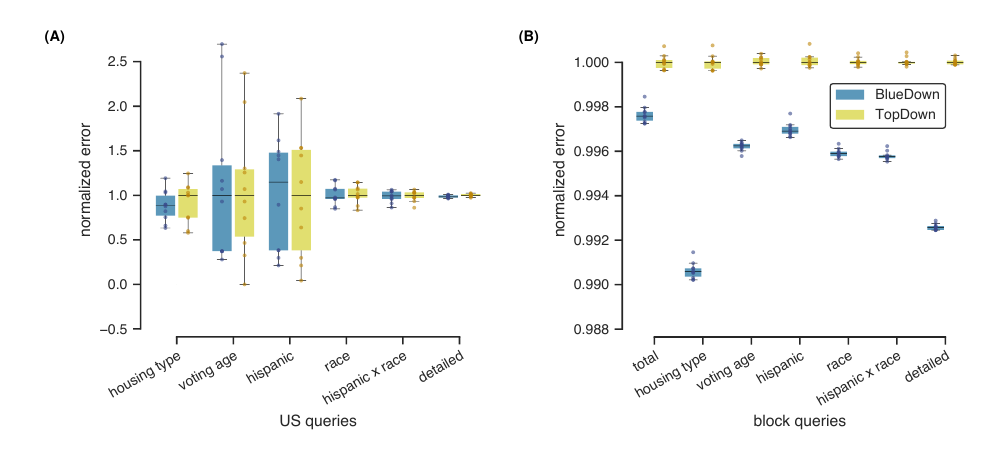}
\end{center}  
\caption{\textbf{Accuracy improvement at US and block levels.} We compare the accuracy of estimates produced by \BlueDown{} and \TopDown{} for marginal queries aggregating over \textbf{(A)} the entire country and \textbf{(B)} by block. 
As in Figure~\ref{fig:general-query-relative-errors}, for each of the $10$ replicates we report the error of each algorithm for each marginal aggregations, normalized by the median error of \TopDown{} for that query and level. For the block level we average over all nodes of that level. Distributions are visualized as box plots showing the median and interquartile range (IQR), where error bars (whiskers) extend up to 1.5$\times$ IQR and individual NMF replicates are overlaid as points.
The total US population is constrained to be exact under both algorithms, and is omitted from the plot.}
\label{fig:more-general-query-relative-errors}
\end{figure}

\begin{figure}
\begin{center}
\includegraphics[width=\textwidth]{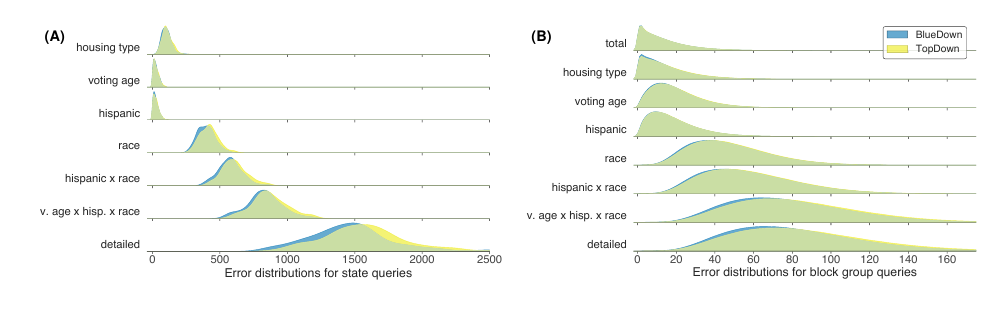}
\end{center}  
\caption{
\textbf{Error distribution of marginal queries at the state and block group levels.} 
The methodology is the same as in Figure~\ref{fig:county-and-tract-bias}(A)--(B).}
\label{fig:more-levels-ridgeline}
\end{figure}

\begin{figure}
\begin{center}
\includegraphics[width=\textwidth]{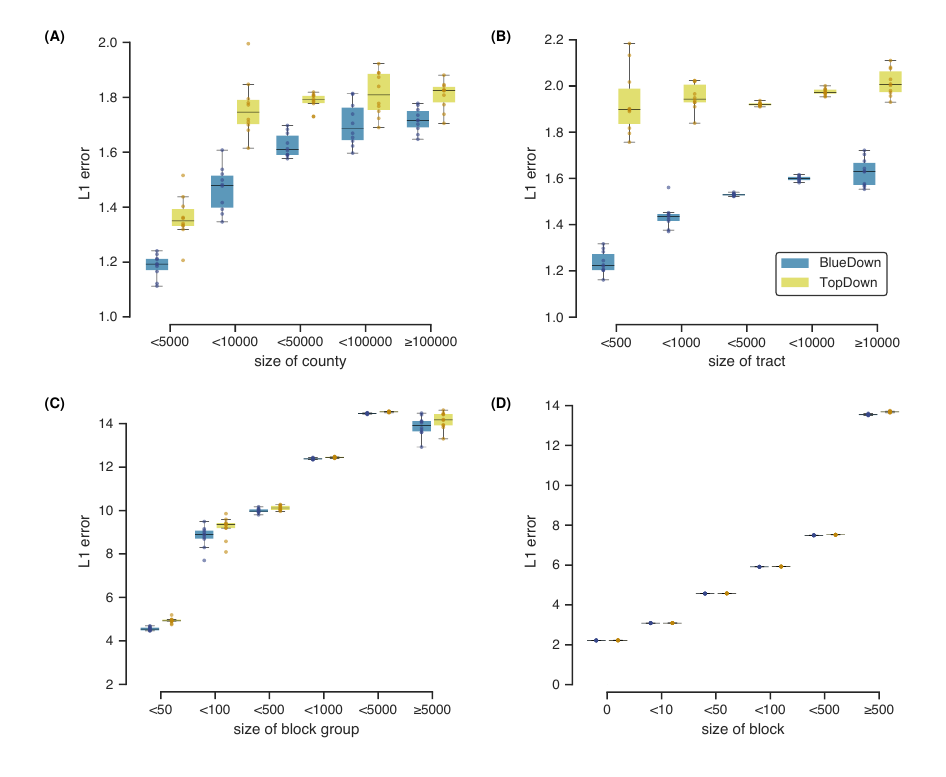}
\end{center}  
\caption{
\textbf{Error distribution by size of geographical regions.} We report the mean absolute error in estimating the TOTAL query for population bins at 
the county, tract, block group, and block levels, presented as box plots (defined as in Figures~\ref{fig:general-query-relative-errors} and \ref{fig:more-general-query-relative-errors}). Population bins are mutually exclusive.
}
\label{fig:errors-by-population-bin}
\end{figure}

\begin{figure}
\begin{center}
\includegraphics[width=\textwidth]{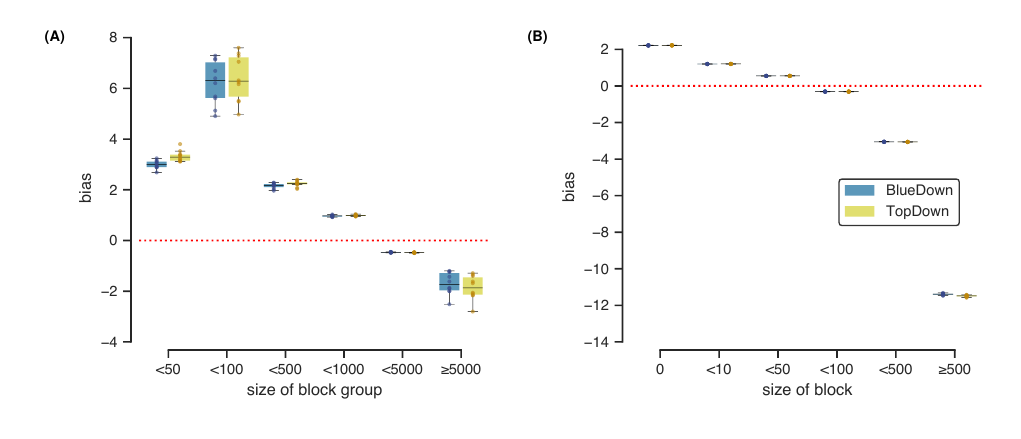}
\end{center}  
\caption{\textbf{Bias at lower geographic levels.} Mean error (bias) is plotted for total population of each population bin for each of the ten independent noisy measurement files at (\textbf{A}) block group and (\textbf{B}) block levels, presented as box plots defined as in Figures~\ref{fig:general-query-relative-errors} and \ref{fig:more-general-query-relative-errors}. The methodology is the same as in Figure~\ref{fig:county-and-tract-bias}(\textbf{C})--(\textbf{D}). Population bins are mutually exclusive, and the dotted red line corresponds to zero bias.}
\label{fig:more-levels-bias}
\end{figure}

\begin{figure}
    \centering
    \includegraphics[width=\algorithmwidth]{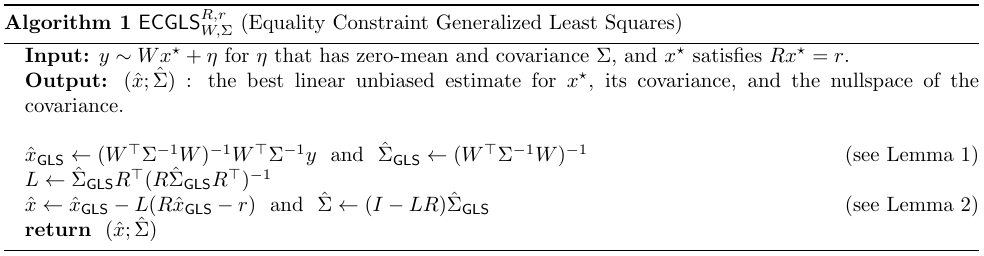}
    
    \vspace{1em}
    \refstepcounter{algorithm}
    
    \textbf{Algorithm \thealgorithm:} Equality Constraint Generalized Least Squares.
    \label{alg:ecgls}
\end{figure}

\begin{figure}
    \centering
    \includegraphics[width=\algorithmwidth]{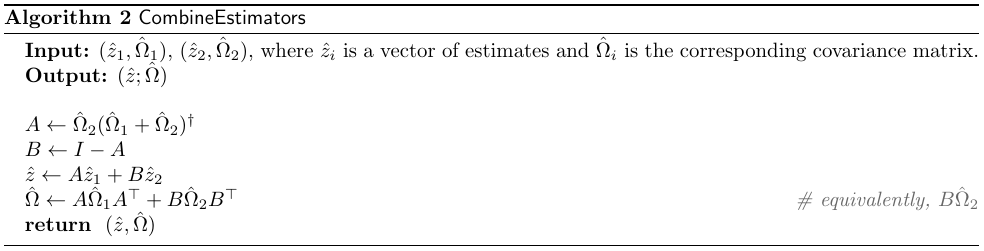}
    
    \vspace{1em}
    \refstepcounter{algorithm}
    
    \textbf{Algorithm \thealgorithm:} CombineEstimators subroutine.
    \label{alg:combine-estimates}
\end{figure}

\begin{figure}
    \centering
    \includegraphics[width=\algorithmwidth]{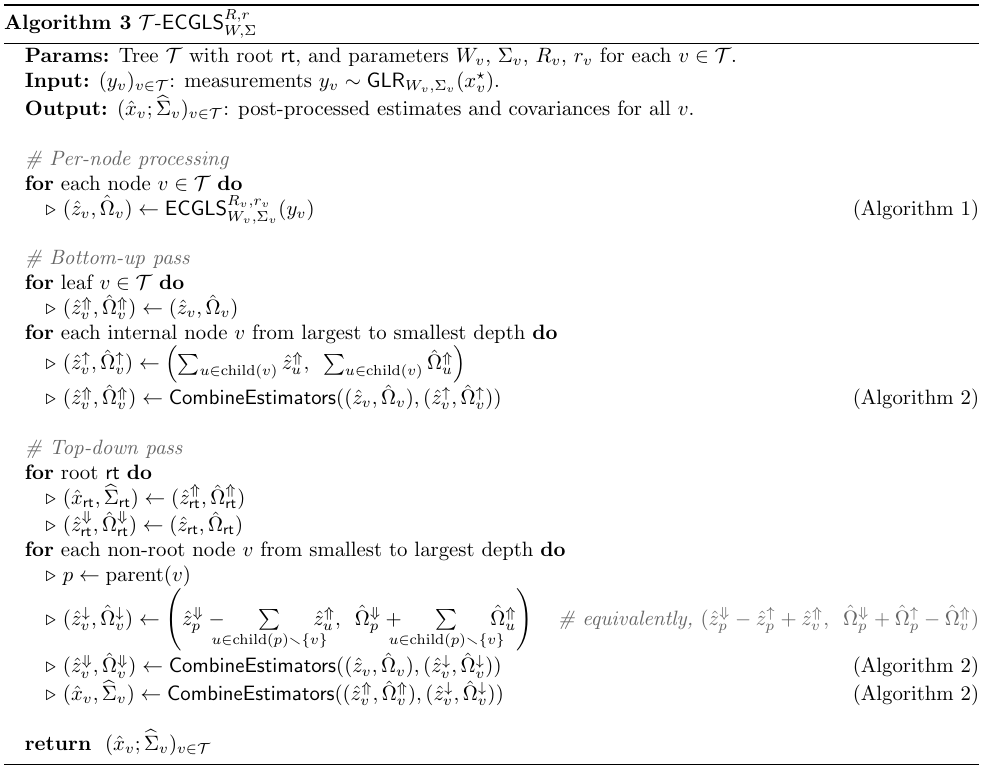}
    
    \vspace{1em}
    \refstepcounter{algorithm}
    
    \textbf{Algorithm \thealgorithm:} Optimal block hierarchical postprocessing of generalized least-squares regression with equality constraints.
    \label{alg:tree-post-processing}
\end{figure}

\begin{figure}
    \centering
    \includegraphics[width=\algorithmwidth]{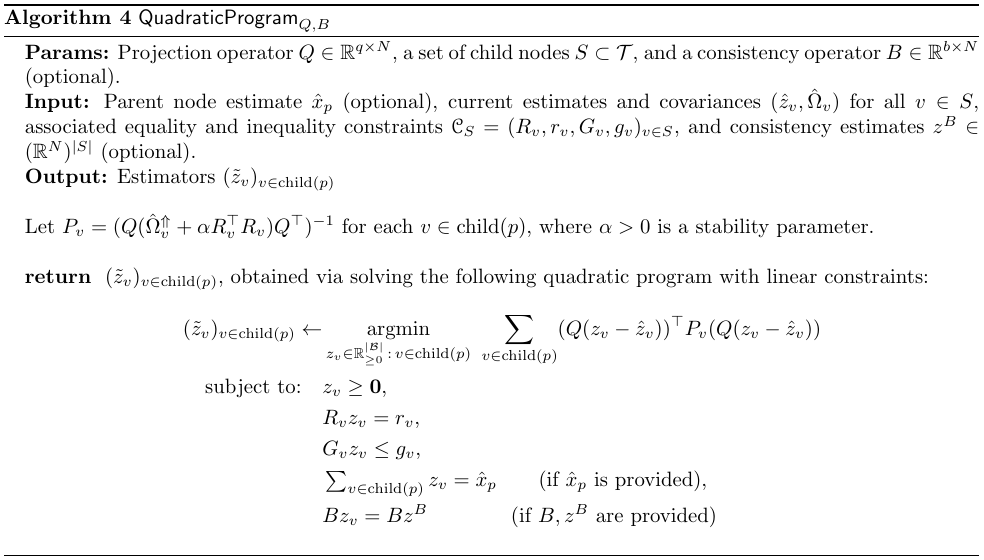}
    
    \vspace{1em}
    \refstepcounter{algorithm}
    
    \textbf{Algorithm \thealgorithm:} Least-squares estimator with projected objective and consistency constraints.
    \label{alg:l2-multipass}
\end{figure}

\begin{figure}
    \centering
    \includegraphics[width=\algorithmwidth]{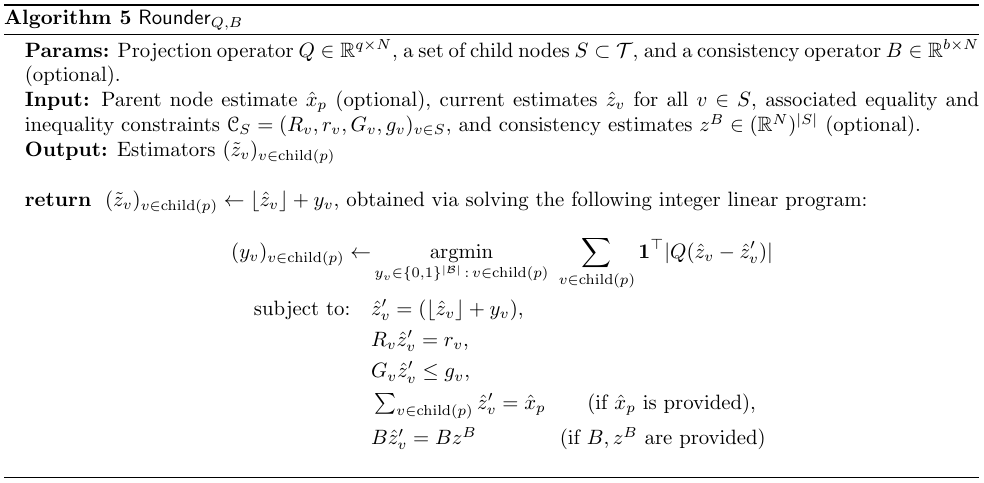}
    
    \vspace{1em}
    \refstepcounter{algorithm}
    
    \textbf{Algorithm \thealgorithm:} Rounder with projected objective and consistency constraints.
    \label{alg:rounder-multipass}
\end{figure}

\begin{figure}
    \centering
    \includegraphics[width=\algorithmwidth]{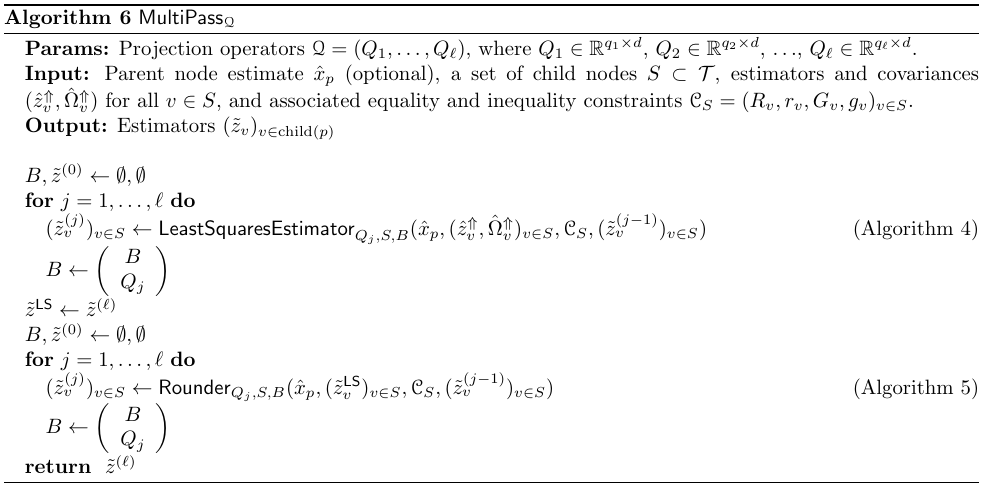}
    
    \vspace{1em}
    \refstepcounter{algorithm}
    
    \textbf{Algorithm \thealgorithm:} Multipass algorithm for least squares estimation and rounding at a single parent node.
    \label{alg:multipass}
\end{figure}

\begin{figure}
    \centering
    \includegraphics[width=\algorithmwidth]{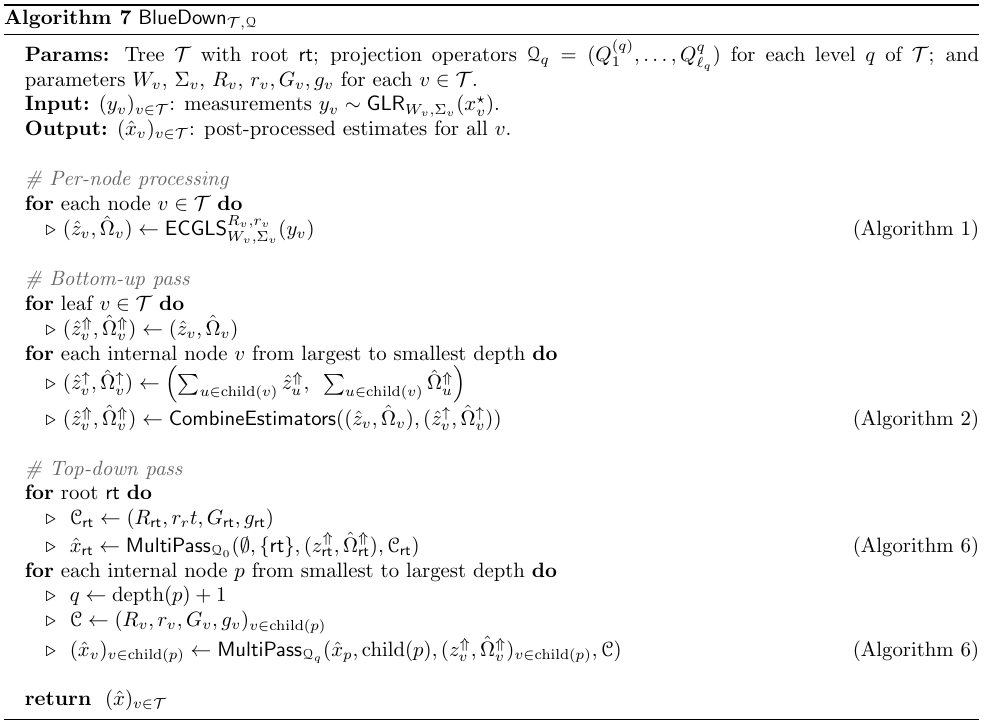}
    
    \vspace{1em}
    \refstepcounter{algorithm}
    
    \textbf{Algorithm \thealgorithm:} Our final end-to-end algorithm.
    \label{alg:bluedown}
\end{figure}

\clearpage

\begin{table}
\centering
\caption{\textbf{Single-feature queries in each per-node workload $W_u$}. This table lists queries that slice by zero or one features, or in the case of HHINSTLEVELS, group slices of a single feature. Such queries account for $79$ of the $2603$ queries in each per-node workload.}
\label{tab:per-node-workload-single-feature}
\vspace{1em}

\renewcommand{\arraystretch}{1.5} 
\begin{tabular}{p{3.3cm} p{1cm} p{4.8cm} p{5.7cm}}
\hline
\textbf{Query Type} & \textbf{Count} & \textbf{Family of Queries} & \textbf{Description of Query} \\
\hline

\textbf{TOTAL} & 
1 & 
$\sum_{b_1, b_2, b_3, b_4} x^\star_{u, b}$ & 
Total population. \\

\textbf{CENRACE} & 
63 & 
$\sum_{b_1, b_2, b_3} x^\star_{u, b}$ for $b_4 \in [63]$ & 
Population counts for each OMB-designated race category. \\

\textbf{HISPANIC} & 
2 & 
$\sum_{b_1, b_3, b_4} x^\star_{u, b}$ for $b_2 \in [2]$ & 
Population counts for each value of Hispanic/Latino origin. \\

\textbf{VOTINGAGE} & 
2 & 
$\sum_{b_1, b_2, b_4} x^\star_{u, b}$ for $b_3 \in [2]$ & 
Population counts for two age groups (17 or younger, 18 or older). \\

\textbf{HHINSTLEVELS} & 
3 & 
$\sum_{b_1 \in C} \sum_{b_2, b_3, b_4} x^\star_{u, b}$ \par \vspace{1.5mm}for $C \in \{ \{1\}, \{2..5\}, \{6..8\}\}$ & 
Population counts per category of household. $C$ sets correspond to individual quarters $\{1\}$, institutional group quarters $\{2, 3, 4, 5\}$, and non-institutional group quarters $\{6, 7, 8\}$. \\

\textbf{HHGQ} & 
8 & 
$\sum_{b_2, b_3, b_4} x^\star_{u, b}$ for $b_1 \in [8]$& 
Population counts by each type of household / group quarters. \\
\hline
\end{tabular}
\end{table}

\begin{table}
\centering
\caption{\textbf{Multi-feature queries in each per-node workload $W_u$}. This table lists queries that slice by two or more features. Such queries account for $2524$ of the $2603$ queries in each per-node workload.}
\label{tab:per-node-workload-multi-feature}
\vspace{1em}

\renewcommand{\arraystretch}{1.5} 
\begin{tabular}{p{3.3cm} p{1cm} p{4.8cm} p{5.7cm}}
\hline
\textbf{Query Type} & \textbf{Count} & \textbf{Family of Queries} & \textbf{Description of Query} \\
\hline
{\bf\boldmath HISPANIC \par $\times$ CENRACE} & 
126 & 
$\sum_{b_1, b_3} x^\star_{u, b}$ \par for $(b_2, b_4) \in [2] \times [63]$& 
Population counts for each value of Hispanic/Latino origin and race. \\

{\bf\boldmath VOTINGAGE \par $\times$ CENRACE} & 
126 & 
$\sum_{b_1, b_2} x^\star_{u, b}$ \par for $(b_3, b_4) \in [2] \times [63]$& 
Population counts for each value of voting age and race. \\

{\bf\boldmath VOTINGAGE \par $\times$ HISPANIC} & 
4 & 
$\sum_{b_1, b_4} x^\star_{u, b}$ \par for $(b_2, b_3) \in [2] \times [2]$& 
Population counts for each value of Hispanic/Latino origin and voting age. \\

{\bf\boldmath VOTINGAGE \par $\times$ HISPANIC \par $\times$ CENRACE} & 
252 & 
$\sum_{b_1} x^\star_{u, b}$ \par for $(b_2, b_3, b_4) \in [2] \times [2] \times [63]$& 
Population counts for each value of voting age, Hispanic/Latino origin and race. \\

\textbf{DETAILED} & 
2016 & 
$x^\star_{u, b}$\par for $b \in [8] \times [2] \times [2] \times [63]$ & 
Population counts for each individual bucket.\\

\hline
\end{tabular}
\end{table}

\begin{table}
\centering
\caption{\textbf{Mean absolute error in total, housing type, voting age, and Hispanic queries.} Each error is averaged across
geographic units by geographic level, and the standard deviation and two-tailed paired t-test p-values of these average error values are reported across the ten replicate runs. The US and state total populations are invariant under both algorithms and have zero error.}
\label{tab:first-four-queries}
\vspace{1em}
\begin{tabular}{|ll|r@{\hspace{1em}$\pm$\,}r|r@{\hspace{1em}$\pm$\,}r|r|}
\hline
Geographic level & Query & \multicolumn{2}{c|}{BlueDown error} & \multicolumn{2}{c|}{TopDown error} & p-value \\
\hline
US & Total & 0 & 0 & 0 & 0 &  \\
State & Total & 0 & 0 & 0 & 0 &  \\
County & Total & 1.59 & 0.02 & 1.75 & 0.02 & $<0.0001$ \\
Tract & Total & 1.53 & 0.00 & 1.92 & 0.00 & $<0.0001$ \\
Block group & Total & 13.68 & 0.02 & 13.76 & 0.02 & $<0.0001$ \\
Block & Total & 3.55 & 0.00 & 3.56 & 0.00 & $<0.0001$ \\
\hline
US & Housing type & 261.40 & 49.33 & 271.20 & 63.91 & 0.3524 \\
State & Housing type & 103.44 & 4.07 & 109.26 & 2.93 & 0.0010 \\
County & Housing type & 25.47 & 0.15 & 36.35 & 0.21 & $<0.0001$ \\
Tract & Housing type & 4.81 & 0.01 & 8.49 & 0.01 & $<0.0001$ \\
Block group & Housing type & 14.78 & 0.02 & 15.68 & 0.02 & $<0.0001$ \\
Block & Housing type & 3.60 & 0.00 & 3.63 & 0.00 & $<0.0001$ \\
\hline
US & Voting age & 48.20 & 36.13 & 45.20 & 30.18 & 0.6308 \\
State & Voting age & 26.13 & 3.85 & 28.02 & 3.82 & 0.0073 \\
County & Voting age & 18.94 & 0.21 & 20.89 & 0.27 & $<0.0001$ \\
Tract & Voting age & 12.42 & 0.03 & 16.17 & 0.06 & $<0.0001$ \\
Block group & Voting age & 19.93 & 0.02 & 20.54 & 0.01 & $<0.0001$ \\
Block & Voting age & 4.85 & 0.00 & 4.86 & 0.00 & $<0.0001$ \\
\hline
US & Hispanic & 47.20 & 28.61 & 46.00 & 30.25 & 0.6970 \\
State & Hispanic & 25.98 & 2.15 & 27.65 & 2.43 & 0.0048 \\
County & Hispanic & 17.21 & 0.30 & 21.44 & 0.38 & $<0.0001$ \\
Tract & Hispanic & 7.32 & 0.02 & 8.50 & 0.03 & $<0.0001$ \\
Block group & Hispanic & 18.71 & 0.02 & 18.93 & 0.02 & $<0.0001$ \\
Block & Hispanic & 4.57 & 0.00 & 4.58 & 0.00 & $<0.0001$ \\
\hline
\end{tabular}
\end{table}

\begin{table}
\centering
\caption{\textbf{Mean absolute error in race, Hispanic $\times$ race, voting age $\times$ Hispanic $\times$ race, and detailed queries}. Each error is averaged across geographic units by geographic level, and the standard deviation and two-tailed paired t-test p-values of these average error values is reported across the ten replicate runs.
The detailed query is housing type $\times$ voting age $\times$ Hispanic $\times$ race.
}
\label{tab:last-four-queries}
\vspace{1em}
\begin{tabular}{|ll|r@{\hspace{1em}$\pm$\,}r|r@{\hspace{1em}$\pm$\,}r|r|}
\hline
Geographic level & Query & \multicolumn{2}{c|}{BlueDown error} & \multicolumn{2}{c|}{TopDown error} & p-value \\
\hline
US & Race & 644.60 & 67.96 & 647.40 & 62.58 & 0.6297 \\
State & Race & 397.65 & 5.21 & 420.31 & 5.37 & $<0.0001$ \\
County & Race & 123.19 & 0.30 & 152.81 & 0.34 & $<0.0001$ \\
Tract & Race & 35.66 & 0.02 & 40.90 & 0.02 & $<0.0001$ \\
Block group & Race & 47.19 & 0.02 & 48.35 & 0.02 & $<0.0001$ \\
Block & Race & 7.83 & 0.00 & 7.86 & 0.00 & $<0.0001$ \\
\hline
US & Hispanic x race & 920.20 & 57.59 & 922.60 & 53.21 & 0.6593 \\
State & Hispanic x race & 575.84 & 7.41 & 609.42 & 9.06 & $<0.0001$ \\
County & Hispanic x race & 159.96 & 0.24 & 199.99 & 0.35 & $<0.0001$ \\
Tract & Hispanic x race & 44.35 & 0.02 & 51.08 & 0.02 & $<0.0001$ \\
Block group & Hispanic x race & 57.53 & 0.02 & 59.02 & 0.02 & $<0.0001$ \\
Block & Hispanic x race & 8.73 & 0.00 & 8.77 & 0.00 & $<0.0001$ \\
\hline
US & V x H x R & 1347.60 & 66.97 & 1349.80 & 64.54 & 0.6907 \\
State & V x H x R & 829.54 & 6.42 & 871.41 & 8.94 & $<0.0001$ \\
County & V x H x R & 222.61 & 0.24 & 260.94 & 0.29 & $<0.0001$ \\
Tract & V x H x R & 78.76 & 0.03 & 100.73 & 0.03 & $<0.0001$ \\
Block group & V x H x R & 78.12 & 0.03 & 82.31 & 0.03 & $<0.0001$ \\
Block & V x H x R & 10.41 & 0.00 & 10.47 & 0.00 & $<0.0001$ \\
\hline
US & Detailed & 4100.20 & 60.41 & 4161.40 & 60.54 & $<0.0001$ \\
State & Detailed & 1465.03 & 8.48 & 1583.16 & 9.28 & $<0.0001$ \\
County & Detailed & 275.86 & 0.23 & 335.15 & 0.20 & $<0.0001$ \\
Tract & Detailed & 83.89 & 0.03 & 110.75 & 0.04 & $<0.0001$ \\
Block group & Detailed & 80.14 & 0.03 & 85.59 & 0.03 & $<0.0001$ \\
Block & Detailed & 10.45 & 0.00 & 10.53 & 0.00 & $<0.0001$ \\
\hline
\end{tabular}
\end{table}

\end{document}